\documentclass[11pt]{article}

\usepackage[margin=1in]{geometry}
\usepackage{times}
\usepackage{amsmath,amssymb,amsthm}
\usepackage{booktabs}
\usepackage{array}
\usepackage{float}
\usepackage{algorithm}
\usepackage{algpseudocode}
\usepackage{enumitem}
\usepackage{url}
\usepackage{cite}
\usepackage{hyperref}
\usepackage{xcolor}
\usepackage{microtype}
\usepackage{tikz}
\usetikzlibrary{arrows.meta,positioning}
\usepackage{pgfplots}
\pgfplotsset{compat=1.18}

\pgfplotsset{
  resultaxis/.style={ width=0.47\linewidth, height=4.7cm, tick label style={font=\scriptsize}, label
  style={font=\scriptsize}, title style={font=\small}, legend style={font=\scriptsize, draw=none,
  fill=none}, grid=both, grid style={black!12}, axis lines=left, } }

\hypersetup{
  colorlinks=true,
  linkcolor=black,
  citecolor=blue!50!black,
  urlcolor=blue!50!black,
  pdftitle={LazyAgent: Demand-Driven Materialization and Physical Optimization of Agentic Programs},
  pdfauthor={Xin Heng},
  pdfsubject={Demand-driven permission for agentic programs},
  pdfkeywords={agentic programs, lazy execution, demand-driven execution, materialization, physical optimization}
}

\theoremstyle{plain} \newtheorem{proposition}{Proposition} \newtheorem{corollary}[proposition]{Corollary}
\newtheorem{lemma}[proposition]{Lemma} \theoremstyle{definition} \newtheorem{definition}{Definition}

\newcolumntype{C}[1]{>{\centering\arraybackslash}p{#1}}
\newcolumntype{L}[1]{>{\raggedright\arraybackslash}p{#1}}

\newcommand{\system}{LazyAgent}
\newcommand{\code}[1]{\texttt{#1}}
\newcommand{\demanded}{\mathrm{dem}}

\setlist{leftmargin=*}

\let\oldthebibliography\thebibliography
\renewcommand{\thebibliography}[1]{%
  \small\oldthebibliography{#1}%
  \setlength{\itemsep}{1pt}}

\title{\system: Demand-Driven Materialization and Physical\\Optimization of Agentic Programs}
\author{\textbf{Xin Heng}\\ Tote AI\\ \texttt{\small xin@tote.ai}}
\date{\today}

\begin{document}
\maketitle

\begin{abstract}
Agent runtimes that plan before acting run every step as soon as it is ready. Neither these runtimes nor the
benchmarks that evaluate them distinguish a step that is ready from a step the request needs. The difference is
large: across 303 requests on 154 production workflows, even an eager runtime that stops once the goal is met
spends 62.2\% of its CPU on work the request never needed.
\system{} is, to our knowledge, the first complete formulation of demand-driven execution for agent runtimes.
It separates readiness from materialization with one rule: run a ready step only if the current request needs
it, and keep that answer exact as the plan changes. One live demanded set decides: a global view of what the
request needs, derived from the program rather than written as per-request rules. It costs one traversal with
no model calls and recomputes as results arrive. The set refuses ready work off every path to the request,
which no stopping rule can, so \system{} spends nothing outside the request, the most any runtime can avoid
without risking required work. A plan becomes a set of possibilities rather than a list of commands: in
composed programs, describing one more unrelated product raises the eager bill by $22.5\%$ and \system{}'s by
$0.0\%$. Because demand is defined exactly, its behavior is proved, not just measured. The proofs tell those
who build agent runtimes and harnesses, before any run, where demand must save and where it must tie, so no
effort goes into gains that are mathematically impossible. The bill is known before a plan runs, so \system{}
can act as a query optimizer to choose among candidate workflows.

Measured live against an eager executor that stops when the goal is met, \system{} saves 51.7\% and 52.1\% of
container time on release gates where a frontier LLM writes the patches and real containers run the tests,
12.8\% of provider-reported tokens on DataSciBench plans the model writes, at matched quality, and 24.0\% of
dollars inside a live LangGraph application, with no losses. When the LLM re-plans after real code failures, it
solves every solvable goal with 11.7\% fewer tokens per goal. On the production workflows it saves 42.0\% of
CPU, and where a request needs the whole graph, it ties that executor exactly as proved. The rule needs no
runtime of its own: a check before each node runs it unchanged inside stock LangGraph, and the whole engine
costs at most 0.43\% of the work it schedules on Montage workflows.

The same set steers every physical operation that follows: projection saves 80--89\% of a shared step on two
third-party test suites, where the identical eager control saves $0.0\%$.

Finally, we expose a blind spot in agent evaluation: public benchmarks are designed the same way as eager
runtimes. Each task is one request whose deliverables together need nearly every planned step (98.6\% on
DataSciBench), so they cannot show this difference. We therefore compose our test programs from their tasks,
and call on industry and research to build benchmarks that measure how efficiently agent workflows run, not
only whether they finish.
\end{abstract}

\section{Introduction}
\label{sec:intro}

Two execution disciplines dominate current agent systems. In the first, an agent invents its next action
from a running scratchpad, as in ReAct~\cite{yao2023react}: reason over the trajectory, take one action,
observe, and reason again. The canonical loop has no predeclared graph to slice and selects one action at a
time. Our ReAct baseline makes one cost of that choice explicit: it re-reads the complete accumulated
trajectory at each step, so without context truncation or prefix reuse, trajectory tokens grow linearly per
turn and sum quadratically over the run. This is the accounting of our explicit baseline, not a theorem
about every ReAct implementation. In the second discipline, the agent writes a plan before all observations
exist and a runtime executes it, as in ReWOO~\cite{xu2023rewoo} and
LLMCompiler~\cite{kim2024llmcompiler}, reprising the declare-then-execute discipline of dataflow
systems~\cite{abadi2016tensorflow}. The planned family is inspectable before execution, exposes concurrency,
and is auditable afterwards.

Eager execution also carries a cost that is rarely stated. A conventional dataflow scheduler implements
\[
  \mathrm{ready}(n) \;\Longrightarrow\; \mathrm{execute}(n),
\]
so every node the planner writes down becomes a bill, whether or not the request at hand needs that node.

\system{} makes that implication a policy decision, and asks a question agent runtimes have not asked: how
much of a program should a given request make real? Its answer is one rule: run a ready step only if the current
request needs it, and keep that answer exact as the plan changes. We separate \emph{readiness}, a property of the
graph, from \emph{materialization}, a separate permission decision:
\[
  \textsc{ready} \;\neq\; \textsc{materialized}.
\]
A node may be well-typed, dependency-satisfied, and free of failing guards, and still remain abstract
because the active goal does not demand it. Because the demanded set holds every step the request may still
need and nothing else, \system{} spends nothing outside the request, the most a runtime can avoid without risking
work the request needs. A plan thus becomes a set of possibilities rather than a list of commands. To our
knowledge, \system{} is the first complete formulation of this separation for agent runtimes, with proved
properties, a runtime that implements it, and a study of the physical operations under it. Systems research often advances this way, by recognizing that a classical idea
answers a new system's problem: vLLM brought virtual-memory paging to the key-value cache of LLM
serving~\cite{kwon2023pagedattention}, and speculative decoding brought speculative execution to token
generation~\cite{leviathan2023speculative}. Neither invented its idea; each made it work in a system the idea
was never built for, and changed what that system could do. Build systems and data orchestrators already serve standing programs by running only the
targets a request selects~\cite{feldman1979make,dbt2026selection}. Agent runs add what builds do not face:
plans that change while they run. \system{} recomputes the set at every change, whether a planner re-plans
after a failure or a candidate settles an issue (Section~\ref{sec:dynamic}). Once permission admits a
node, those operations decide which admitted node runs first, how much of its output becomes concrete,
whether a still-legal option is declined, whether an earlier artifact can satisfy it, and whether an
irreversible effect is authorized. The same state model carries an independent commit constraint.

\paragraph{Relation to the closest systems.}
The nearest agent systems each solve a different problem well.
LLMCompiler~\cite{kim2024llmcompiler} sees the dependency graph and runs ready tasks in parallel, deciding
\emph{when} each task runs. Helium~\cite{wadlom2026helium} sees one compiled plan and makes it cheaper
to serve: it prunes backward from that plan's outputs, shares duplicate subgraphs, and substitutes cached
results. LLM-as-Scheduler~\cite{xiang2026las} sees the latest artifact and decides whether to exit, verify,
repair, or reroute. \system{} sees the standing program and the live request together, and decides which
part of the program this request owns. The views are complementary, and each meets the permission layer at a
natural point: Helium's backward prune corresponds to the demanded set of a request that asks for every output
of the plan, and its subgraph sharing and cache substitution to fusion and reuse; LLM-as-Scheduler's early exit
corresponds to goal stopping, and its rerouting composes with permission; LLMCompiler's parallel dispatch
becomes ordering inside the permitted set. What the permission layer adds is the request itself: it declines a
step that is ready, valid, and outside today's request while that request is still open, and that is where the
measured savings in this paper come from.

\paragraph{Proved guarantees.}
Defining what a request needs is what makes guarantees possible. We prove five properties
(Sections~\ref{sec:model} and~\ref{sec:theory}), and the evaluation tests every one: the set only shrinks while
the plan is fixed, so a refusal is final; a step outside it never runs, in any visit order; goal stopping's
savings depend on visit order and \system{}'s do not; and the proofs say exactly where the two tie. To our
knowledge, these guarantees are new: earlier agent systems do not define what a request needs. They turn trial
and error into engineering. A team building an agent runtime or harness knows before writing code where demand
must save and where it must tie, so it spends no effort on gains that are mathematically impossible. The bill is
also known before a plan runs, so \system{} can do for agent workflows what a query optimizer does for queries:
when a loosely stated request admits several workflows, it chooses the cheapest before running any
(Appendix~\ref{app:proofs}).

\paragraph{Why the decision cannot be made node by node.}
This invites an obvious objection: keep the eager loop, and simply ask at each ready node whether that node
ought to run. Three obstacles stand in the way: \textbf{relevance is a global graph property; checking every
node has its own cost; and useful work may still be unrequested.} Together they motivate the mechanism.
\textbf{Membership is not local.} Whether a node $n$ matters is whether some path leads from $n$ to a node
the goal named --- a reachability property of the graph, not a property of $n$ and its neighborhood. A
local test must therefore either guess, or first propagate the goal backward through the graph; and the
second is the closure itself, so a correct local test is this algorithm wearing a different interface.
\textbf{Checking every node consumes part of the saving.} Consider first a guard-valid graph $G=(V,E)$,
where $V$ is the set of program nodes, $E$ its data, guard, and ordering edges, and $D \subseteq V$ the
demanded set for the current goal. Let $c(n)$ be the cost of materializing node $n$ and $q(n)$ the cost of
asking a perfect local oracle whether $n$ should run. Let $S$ denote net cost saved relative to an eager
executor that runs every node in $V$. Then
\[
  S_{\mathrm{closure}} = \sum_{n \in V \setminus D} c(n), \qquad
  S_{\mathrm{check}} = \sum_{n \in V \setminus D} c(n) - \sum_{n \in V} q(n).
\]
The first quantity is LazyAgent's saving: the execution cost of every node outside the goal's closure. The
second is the best a per-node checker could save, even if it never makes a mistake: it avoids the same nodes
but pays $q(n)$ for every node, including those it keeps. Thus
$S_{\mathrm{check}}=S_{\mathrm{closure}}-\sum_{n\in V}q(n)$; if checking is itself a model call and
$q(n)\sim c(n)$, most or all of the saving disappears.

Computing $D$ instead is one backward graph traversal. With adjacency lists, each node and edge is visited
at most once, so each recomputation costs $O(|V|+|E|)$ time and $O(|V|)$ space, with no model calls. The
traversal stores one demand bit per node; reading that bit when the node becomes ready is therefore
$O(1)$ --- by construction, not by an empirical assumption. \textbf{Useful is not requested.} A local judge
still sees ready work that is cheap and genuinely valuable: rendering a sky band someone will want tomorrow
is useful, but not required for today's image (Section~\ref{sec:motivation}). Only backward reachability
represents that distinction.

The algorithm is therefore backward reachability, with the same direction as backward induction rather than
greedy local choice. Cheap global analysis followed by constant-time decisions removes execution cost,
tokens, and waiting time for unrequested nodes. More importantly, it lets a planner describe a much wider
graph without forcing the runtime to execute that width.

\paragraph{Why this is not early stopping, or routing by hand.}
A well-engineered eager executor already checks whether the goal is met and halts, and a developer can route each
request type to the branches it needs. Both are predicates written in advance; a router covers only the request
types and the plan its author anticipated. The demanded set is derived from the program for every request, so it
also follows the plan when a planner extends it, retires candidates as they settle, and tells every physical
operation how much of an admitted step the request needs. This is a real mechanism,
implemented in recent dynamic schedulers~\cite{xiang2026las}, and it is the baseline we hold ourselves to
throughout: \code{stop\_eager}. The difference is what the executor consults. Early stopping evaluates a
\emph{predicate} --- \emph{are we done yet?} Demand-driven materialization evaluates a \emph{set} ---
\emph{is this ready node on any path required by what was asked for?} A predicate can only stop the future;
a set can also excuse the present. When a graph contains ready work that no goal-relevant path passes
through, early stopping avoids it only if the scheduler happens to reach the goal before offering it, which
is a property of node ordering rather than of the program. The gap is widest where workflows tend to put
their cost: when the expensive work is ready first and the requested output finishes last, the predicate
turns true only after the off-path bill has been paid. On a recorded Montage workflow
(Section~\ref{sec:motivation}), \code{stop\_eager} ran 56 of 58 jobs to deliver one band image; \code{lazy}
ran 19 and spent 74.8\% less CPU.

\paragraph{A blind spot in agent evaluation.}
Agent evaluation has not noticed the difference between ready and needed either, and to our knowledge this paper
is the first to separate the two and to measure it. Separating readiness from permission also supplies a measuring instrument for the evidence base itself, and
the first thing it measures is how little of a plan lies outside its request. Each benchmark task is one
request, and a node-level census finds almost nothing outside it: across the 28 DataSciBench plans we run, 137
of 139 planned nodes serve some deliverable of the task (\textbf{98.6\%}), \textbf{27 of 28} plans contain no unrequested work,
and the pool arenas equal their own backward closure (Section~\ref{sec:arenas}).

That follows from the question these corpora were built to ask. Each poses a single request and scores
whether an agent can satisfy it, so the plan is written for that request and everything in it is needed by
construction. None was built to ask the complementary question --- whether an agent can decline work nobody
asked for --- which is what demand-driven execution exists to answer. A review of 445 language-model
benchmarks documents the general concern: 42.6\% reuse existing benchmark data, and the authors warn that
reuse can constrain construct validity~\cite{bean2025construct}. Here the specific missing variable is the
share of a program outside the current request. Existing corpora pin it near zero, so eager and
demand-driven execution must agree there; Proposition~\ref{prop:equiv} proves the tie, which we report as a
prediction rather than a failure. A further economic explanation --- that readiness billing
\emph{additionally} suppresses breadth, so a planner told breadth is free would describe more of it --- is
testable, and did not survive the test (Section~\ref{sec:planning-incentive}).

\paragraph{Real programs already have this shape.}
Teams author agent graphs once and run them many times, on schedules, on events, or toward goals that span
days. Such a \emph{standing program} describes every product the team might want, and each run asks for a few.
The programs we measure already have this shape: 86 of the 154 production workflows we replay produce several
products, and 24 of the 28 DataSciBench tasks we run ask for several deliverables (95 in all). Their plans also
change while they run: in our live runs, real code failures led a frontier LLM to re-plan in 10 of 28 tasks.
No public trace records which outputs real invocations ask for, so our requests are drawn by stated rules from
each program's own outputs, and the results do not rest on that frequency. The saving grows with the share of
a program a request leaves out, so a deployment can estimate its saving from its own request mix; where a
request needs everything, \system{} ties exactly and its engine adds at most 0.43\% on Montage workflows; and
recomputation pays with no narrower request at all (Section~\ref{sec:dynamic}).

\paragraph{Why the operations need the same discipline.}
Once permission decides which nodes may run, is any further optimization of \emph{how} they run attributable
to demand-awareness, or a scheduling gain any eager executor could take? Fitting a policy that skips or
reorders agent work does not settle it: LLM-as-Scheduler adapts workflow routing and EnumGRPO searches
quality--cost query plans~\cite{xiang2026las,nie2026agenticquery}, while
Helium~\cite{wadlom2026helium} schedules semantically equivalent operators for serving efficiency.
\system{} contributes a different attribution test. We report each operation's
\emph{demand-specific remainder} --- its worth to \code{lazy} minus the worth of the identical mechanism on
a matched eager arm. A lazy-only column with no matched eager row shows that an operation works, not that
demand made it work.

\paragraph{What this paper contributes.}
\system{} makes demand the execution principle of agent runtimes. What that takes, and what it reveals, are
this paper's contributions:
\begin{enumerate}[topsep=2pt,itemsep=1pt]
  \item The permission rule \(\textsc{ready} \neq \textsc{materialized}\) for agent-authored programs
  (Section~\ref{sec:model}). Readiness stays a property of the graph. Materialization is membership in a
  demanded set derived from the live request and recomputed after observations. The same set is the
  reference for later physical operations, the policy factors into permission,
  order, and projection, and four invariants hold across policies. Five proved properties --- monotonicity,
  equivalence, separation, order contingency, and order invariance --- tell builders before any run where demand
  must save, where it must tie with goal stopping, and when visit order cannot change the bill, which lets
  \system{} choose the cheapest of several candidate workflows before running any. The state model carries an
  independent commit constraint.
  \item A vocabulary of eight operations over the demanded set --- reorder,
  project, prune, reuse, fidelity, fusion, speculation, and commit --- each stated with the matched control
  it requires (Section~\ref{sec:operations}).
  \item A \textbf{blind spot in agent evaluation}
  (Sections~\ref{sec:intro} and~\ref{sec:arenas}): public agent benchmarks contain almost no off-path
  work, because each scores whether one request can be satisfied rather than whether unrequested work can be
  declined. They therefore pin the variable demand-driven execution acts on near zero and cannot exercise
  the mechanism under study. A paired planning-only intervention additionally rules out the tempting
  economic explanation, finding no evidence that stating demand-billed rather than readiness-billed
  execution makes one planner describe broader graphs (Section~\ref{sec:planning-incentive}).
  \item Six public arenas and a family of composed programs
  (Section~\ref{sec:arenas}), each labeled by which of three preconditions it meets.
  \item An empirical demarcation of where permission pays and where it
  provably cannot (Section~\ref{sec:results}), against goal-stopping eager execution. Proofs state where demand
  must tie and where it alone avoids the work under every visit order. Where the request reaches the whole
  graph, both semantics are identical to the integer --- predicted before measuring. Where the program is wider than the request, the eager bill carries that width as a factor and
  the demand-driven bill does not, so breadth stops being a budget decision. Goal stopping recovers part of
  that difference under a kind visit order and none of it while the goal stays open. Over sessions with a cache
  on both runtimes, the difference survives except when every product is eventually requested with nothing
  changed (Section~\ref{sec:discussion}).
  \item A discipline for attributing physical optimizations to demand-awareness
  rather than to scheduling (Sections~\ref{sec:physops} and~\ref{sec:contingent}): each operation is scored
  by its \emph{demand-specific remainder} --- worth to \code{lazy} minus the worth of the identical
  mechanism on a matched eager arm, signed and allowed to come out negative. The remainders sort the
  operation space rather than tally wins: those narrowing how much of a demanded node becomes concrete add
  to permission,
  those removing which nodes run overlap with it, and speculation places \code{lazy} and
  \code{eager} at two ends of one dial.
  \item Systems evidence that the rule is practical (Sections~\ref{sec:demand}, \ref{sec:impl},
  and~\ref{sec:sep}). Computed once per state change, the demanded set costs at most 0.33\% of the work a run
  schedules, and the whole engine at most 0.43\% on replayed Montage workflows, with decisions identical on over
  350{,}000 replayed runs. A per-node wrapper runs the rule unchanged inside stock LangGraph, and wrapped around
  a LangGraph application's own graph with live LLM calls it makes the predicted decision on all 122
  runs and costs 24.0\% fewer dollars than \code{stop\_eager}, with quality matched on 39 of 40 requests. In simulation with a bounded worker pool, the requested output arrives 42\%,
  34\%, and 19\% sooner with 1, 4, and 16 workers and never later; in parallel waves the
  saving rises to 44.5\%; and a built-in retry on a missing input recovers every detectable failure an
  undeclared dependency causes while keeping the saving.
  \item To our knowledge, the first measurements, on real outcomes and live re-planning, of what recomputing
  the demanded set buys when an agent run changes its own plan (Section~\ref{sec:dynamic}): every solvable goal
  solved when a live agent re-plans, with 11.7\% fewer tokens per goal than goal-stopping execution, and 22.7\%
  saved even against per-pool stopping when candidates fix issues of release gates mid-run.
\end{enumerate}

\section{Motivating example}
\label{sec:motivation}

\subsection{A standing workflow and a narrow request}
\label{sec:motivation-montage}

Montage is an astronomy pipeline that assembles survey plates into sky mosaics, and its recorded executions
are among the production traces collected by WfCommons~\cite{coleman2022wfcommons,wfinstances}. One recorded
instance, built from the Digitized Sky Survey, is a standing program of 58 jobs that produces four images.
Each of three sky bands has its own 19-job lane: reproject four plates onto a common grid, fit their six
pairwise overlaps, build a background model, correct the four plates, stitch them, and render the band
image. A final job combines the three stitched bands into a color mosaic. The lanes share no jobs, and the
trace records the measured CPU time of every job.

The program is broader than any one request because it records the full survey product. Today an
astronomer, or an AI research assistant running the standing workflow on her behalf, asks for one thing: the
band-2 image. The other two bands are legitimate work that a later request will want, so deleting them from
the program would be wrong. Only today's request says they are not needed now.

Three execution policies carry the comparisons throughout the paper. Two of them turn on \emph{guards}. A
guard is a condition the program records on a step that reads a result from an upstream step. It is
\emph{unresolved} while that result does not yet exist, and it \emph{fails} once the result contradicts it,
which makes the step it protects useless; declining a step whose guard has failed is the one refusal every
runtime already makes. Montage records no guards; Section~\ref{sec:motivation-agent} gives one.
\code{guard\_eager} runs every dependency-ready node whose guards have not failed. \code{stop\_eager} follows
the same rule while the goal remains open, then refuses new work once the goal is met. \code{lazy} uses the
same guards and stopping test but adds permission: a ready node runs only if the backward closure of the
active goal reaches it. For the band-2 image, that closure is the 19 jobs of band 2's lane and nothing else.

\paragraph{Why stopping on success does not help here.}
This workflow puts its cost where readiness arrives first. The twelve reprojection jobs, four per band,
account for 5{,}469 of the workflow's 5{,}586 CPU-seconds (98\%), and none has an upstream dependency, so all
twelve are ready at the start. The requested image is the last job of its lane, five stages later. Until it
exists the goal is open, and every ready job of bands 1 and 3 is work an eager scheduler has no reason to
refuse. How much of it \code{stop\_eager} pays depends on visit order: a scheduler that happened to finish
band 2's lane before starting the others would stop after 19 jobs. One that interleaves ready work, as the
replay's fixed selector does, has started every band's reprojection before any band image exists, and its
stopping test fires after that bill is paid. In the replay, \code{stop\_eager} ran 56 of the 58 jobs for the
band-2 request and spent 5{,}580 CPU-seconds; the two jobs it declined were renders it had not yet reached.
\code{lazy} ran the 19 jobs of band 2 and spent 1{,}405, a 74.8\% saving, and would have spent exactly that
under any visit order. When the expensive work is ready first and the requested output finishes last, a
predicate that asks \emph{are we done yet?} answers after the money is spent.

The permission decision requires no model call and no knowledge of costs. Once the program and the request
are present, the runtime walks backward from the requested image, marks the jobs it reaches, and checks one
bit when each job becomes ready.

\paragraph{The same program under a broad request.}
Ask instead for the color mosaic, and the closure covers 55 of the 58 jobs: every job except the three
single-band renders. The policies then nearly tie, at 5{,}586 and 5{,}581 CPU-seconds, as they must, because
there is almost nothing outside the request to refuse. Figure~\ref{fig:schematic} shows both requests on one
program. The program never changed; only the request did. Both requests are rows of the workflow study in
Section~\ref{sec:sep}, not a separate experiment.

\paragraph{Is the answer not obvious?}
To a reader of Figure~\ref{fig:schematic} it is, and that is the point: the decision is obvious to whoever
knows the request and invisible to a scheduler that knows only readiness. An agent that skips the other
bands must either plan narrowly for each request, giving up the standing program for the single-request
shape of public benchmarks, or test each ready job against the request, which done correctly is the
backward closure computed at the price of a model call (Section~\ref{sec:intro}). The rule ``run only band
2's jobs'' is therefore not a stopping predicate but the demanded set written by hand, for one request on a
fixed graph. It must be rewritten whenever the request changes, an observation removes part of the closure
(Section~\ref{sec:motivation-agent}), or the agent adds steps; \system{} recomputes it in each case, by graph
traversal rather than model judgment, so the set is the same on every run and never omits a step the request
may still need (Section~\ref{sec:demand}).

\begin{figure}[t]
\centering
\begin{tikzpicture}[
  >={Stealth[width=4pt,length=4pt]},
  box/.style={rounded corners=2pt, minimum height=6mm, minimum width=15mm,
              font=\scriptsize, align=center, inner sep=1.5pt},
  dem/.style={box, draw=black, fill=black!8},
  off/.style={box, draw=black!45, dashed, fill=white, text=black!55},
  goal/.style={box, draw=black, very thick, fill=black!22, font=\scriptsize\bfseries},
  e/.style={->, thick, black!75},
  eoff/.style={->, thick, black!30, dashed},
  head/.style={font=\scriptsize, align=center, anchor=south},
  lane/.style={font=\scriptsize, align=right, anchor=east}
]
\foreach \x/\stage in {0/reproject, 1.95/fit overlaps, 3.9/background model,
                       5.85/correct, 7.8/stitch, 9.75/render, 12.05/combine} {
  \node[head] at (\x, 1.6) {\stage};
}
\node[font=\tiny, text=black!70, align=center] at (0, -1.62) {ready at start;\\[-2pt]98\% of all CPU};

\foreach \b/\y/\kind/\cpu in {1/1.05/off/2{,}206, 2/0/dem/1{,}405, 3/-1.05/off/1{,}971} {
  \node[lane] at (-0.95, \y) {band \b\\[-1pt]{\tiny\cpu\ CPU-s}};
  \node[\kind] (p\b) at (0, \y)    {4 jobs};
  \node[\kind] (f\b) at (1.95, \y) {6 jobs};
  \node[\kind] (m\b) at (3.9, \y)  {2 jobs};
  \node[\kind] (c\b) at (5.85, \y) {4 jobs};
  \node[\kind] (s\b) at (7.8, \y)  {2 jobs};
}
\node[off, minimum width=18mm]  (r1) at (9.75, 1.05)  {band-1 image};
\node[goal, minimum width=18mm] (r2) at (9.75, 0)     {band-2 image};
\node[off, minimum width=18mm]  (r3) at (9.75, -1.05) {band-3 image};
\node[off, minimum height=27mm, minimum width=16mm] (col) at (12.05, -0.25) {color\\mosaic};

\foreach \b in {1,3} {
  \draw[eoff] (p\b) -- (f\b); \draw[eoff] (f\b) -- (m\b); \draw[eoff] (m\b) -- (c\b);
  \draw[eoff] (c\b) -- (s\b); \draw[eoff] (s\b) -- (r\b);
}
\draw[e] (p2) -- (f2); \draw[e] (f2) -- (m2); \draw[e] (m2) -- (c2);
\draw[e] (c2) -- (s2); \draw[e] (s2) -- (r2);
\draw[e] (p2.north) to[bend left=10] (c2.north);
\foreach \b/\y in {1/0.6, 2/-0.45, 3/-1.5} {
  \draw[eoff] (s\b.south) -- (7.8, \y) -- (col.west |- 0, \y);
}

\node[font=\scriptsize, text=black!75, anchor=west, align=left] at (-1.75, -2.2)
  {shaded = closure of the request ``band-2 image'' (\emph{demanded});
   dashed = ready but outside the request (\emph{refused by demand}).};
\end{tikzpicture}

\vspace{4pt}
{\small
\begin{tabular}{@{}lrr@{}}
\toprule
& \multicolumn{2}{c}{\textbf{Jobs run, measured CPU-seconds}}\\
\cmidrule(l){2-3}
\textbf{Policy} & \textbf{Request: band-2 image} & \textbf{Request: color mosaic}\\
\midrule
\code{guard\_eager} & 58 jobs, 5{,}586 & 58 jobs, 5{,}586\\
\code{stop\_eager}  & 56 jobs, 5{,}580 & 58 jobs, 5{,}586\\
\code{lazy}         & \textbf{19 jobs, 1{,}405} & 55 jobs, 5{,}581\\
\bottomrule
\end{tabular}}
\caption{One recorded Montage workflow (WfInstances trace \code{montage-chameleon-dss-05d-001}), drawn as
three band lanes and a color combine; each box groups one stage of a lane's jobs, and each correction also
reads its plate's reprojection (curved arrow, drawn for band 2). Under the request for the band-2 image the shaded lane is the
backward closure, and the other lanes are ready but outside the request. The reprojection column carries
98\% of the workflow's CPU and is ready at the start, while the requested image is the last job of its lane,
so \code{stop\_eager}'s stopping test fires after nearly all of the off-path cost is paid. Under the request
for the color mosaic the closure covers 55 of the 58 jobs and the policies nearly tie. Both requests are rows
of the study in Section~\ref{sec:sep}.}
\label{fig:schematic}
\end{figure}

\subsection{The same shape in an agent's program}
\label{sec:motivation-agent}

The structure is the same when an agent writes the program. Suppose a company entrusts routine release
execution to an AI release maintainer, whose engineering team has encoded the write--review--build--deploy
process as a standing graph. Three issues are eligible for the next release. All of them depend on one
environment build, and each has several candidate patches written by the agent, any one of which may pass.
A guard records that a candidate is tested only if it applied cleanly. The operator asks for a narrower
goal: \emph{ship issue one now}.

Two things are new relative to Montage. First, issue one's candidates are alternatives for the same request:
any of them may be the one that passes, so all of them stay inside the closure, nothing there can be refused
in advance, and \code{lazy} and \code{stop\_eager} must tie on that part of the program
(Section~\ref{sec:equiv}). Second, the requested issue may never be fixed. Stopping on success then never
fires, so \code{stop\_eager} pays for the other issues' branches in full, while \code{lazy} refuses them from
the start, fixed or not. In the live gate of Section~\ref{sec:live}, no candidate ever settled astropy issue
13033: both eager arms spent 355.8 container-seconds, and \code{lazy} spent 4.7. Montage shows success
arriving too late to stop the bill; the release gate shows success that never arrives. In both, a frontier
model may write the steps, but no model is asked which branch should run. The same release program also
hosts the physical operations of Section~\ref{sec:operations}: candidate order, projected test output,
explicit pruning, and reuse of an earlier environment.

This shape recurs whenever a standing program has several independently useful products and a human or
upstream service asks for a subset: a scientific workflow with several outputs, as in
Figure~\ref{fig:schematic}; a release gate with several eligible issues; a data-science plan whose cleaned
table, trained model, and report are separately requestable; or a build graph in which one target is
requested from a package. Section~\ref{sec:results} measures all four on real recorded costs.

\section{Execution model}
\label{sec:model}

\subsection{Programs}

A program is a graph $G$ of typed nodes. Each node carries an identifier, a task specification, input
bindings, ordering dependencies, guards, and a lifecycle state. An input binding is either a literal value
or a field reference $(u,p)$: the identifier $u$ of an upstream node and a dot-separated field path $p$ into
that node's output. At runtime the reference resolves to field $p$ of $u$'s materialized value, so a data
edge can demand one field rather than the whole output. A guard likewise
names an upstream node, field path, comparison operator, and value, and resolves to holds, fails, or
unresolved. Each task specification declares a resource vector over wall time, CPU seconds, dollars, and
tokens, plus an effect class distinguishing reversible from irreversible work.

Lifecycle states are
\[
\textsc{planned} \to \textsc{ready} \to \textsc{demanded} \to \textsc{materializing} \to
\textsc{materialized},
\]
with \textsc{skipped}, \textsc{cancelled}, \textsc{failed}, \textsc{pruned}, and \textsc{blocked} as
additional terminals. The two middle states are \emph{transient}: a node enters \textsc{ready} when the
engine offers it and \textsc{demanded} when the policy accepts it on demand, and a node the policy declines
returns to \textsc{planned}, to be offered again when progress occurs. Readiness is thus a fact about $G$
alone, recomputed every wave; demand is relative to the current goal. At the end of a run a node that was
never demanded is therefore \textsc{planned}, and the metrics report it under the terminal \emph{reason}
\emph{never-demanded}, so ``we declined to run it,'' ``a guard proved it irrelevant,'' and ``we burned a
still-legal option on purpose'' stay three different facts. Terminal reasons partition the nodes that were
never materialized --- never-demanded, skipped, cancelled, pruned, blocked --- while \textsc{materialized}
and \textsc{failed} are the two paid outcomes, and a failed attempt is charged in full.

\subsection{State, actions, and demand}
\label{sec:demand}

In plain terms, a materialization state is the complete runtime snapshot used to answer the next execution
question: what exists, what is wanted, what has already happened, what each choice costs, and what is
authorized.

\begin{definition}[Materialization state]
\label{def:state}
A materialization state is
\[
\sigma = \langle\, G,\; R,\; A,\; O,\; C,\; B \,\rangle,
\]
where $G$ is the program, $R$ the root demand (what the run is \emph{for}), $A$ the artifacts produced so
far, $O$ the observations (node states and resolved guard predicates), $C$ the declared resource vectors,
and $B$ the commit constraints over irreversible nodes.
\end{definition}

This tuple is our state abstraction, not a claimed standard, nor the only possible decomposition. It
is operationally sufficient for every decision studied here: $G$ supplies the available work and
dependencies; $R$ distinguishes two requests against the same program; $A$ and $O$ determine readiness,
goal satisfaction, and resolved branches; $C$ supplies prices known before execution; and $B$ supplies
authorization that neither readiness nor demand implies. Omitting any category would merge runtime
situations that can require different legal decisions. We do not claim the tuple is mathematically minimal.
$G$ may grow when a host agent emits work; $C$ is \emph{declared}, not measured, so a policy never sees a
cost it could not have known; and unresolved guards in partial $O$ fail open, never licensing refusal.

The following four actions are \emph{our engine's transition interface}, not a catalog of optimization
operations. They cover the possible outcomes when the engine
considers work: execute the offered node, postpone it without closing the option, close it permanently, or
rewrite the graph. In the table, $n$ is the offered node, $\phi$ the set of output fields to materialize,
$m$ the selected materializer implementation (and hence possibly its fidelity), and $\tau$ a typed,
semantics-preserving graph transformation. The currently offered $n$ is implicit in $\mathrm{Defer}$.

\begin{center}
\small
\begin{tabular}{p{0.24\linewidth}p{0.70\linewidth}}
\toprule
\textbf{Action} & \textbf{Meaning}\\
\midrule
$\mathrm{Materialize}(n,\phi,m)$ & Make $n$ concrete under field projection
$\phi$ using materializer $m$. Legal only if $n$ is ready, and, if $n$ is
irreversible, only under a released barrier in $B$.\\
$\mathrm{Defer}$ & Pay for nothing this round; $n$ returns to
\textsc{planned} and may be offered again when progress occurs.\\
$\mathrm{Prune}(n)$ & Decline a demanded, still-legal option permanently,
preserving $G$ as the accounting denominator and recording a terminal reason
distinct from \textsc{skipped} and \textsc{never-demanded}.\\
$\mathrm{Transform}(\tau)$ & Apply a typed graph transformation (fusion,
common-subexpression elimination, substitution). Interface defined; not
evaluated here.\\
\bottomrule
\end{tabular}
\end{center}

The action interface and the operation catalog answer different questions. Reordering chooses which $n$ is
offered before an action is taken; projection and fidelity supply $\phi$ and $m$ to
$\mathrm{Materialize}$; prune is already a direct action; fusion and substitution can be represented by
$\mathrm{Transform}(\tau)$; reuse may satisfy a demand from an existing artifact instead of invoking a new
materializer; speculation changes which work permission admits; and commit constrains whether an
irreversible $\mathrm{Materialize}$ action is legal. Thus the four actions describe state transitions,
whereas Section~\ref{sec:operations} describes policies and mechanisms that choose or parameterize them.

A \emph{goal form} supplies two things: root nodes and required fields (an empty field set means the whole
value), and a satisfaction predicate over the resulting graph. We implement three forms:
\textbf{AllOutputs}, a conjunction satisfied when every named node is materialized;
\textbf{AnyAcceptable}, a disjunction satisfied by the first payload passing an acceptance predicate and
demanding every candidate because any might work; and \textbf{EachOf}, a conjunction of subgoals that may
themselves be disjunctive --- several products, each satisfiable by several candidates. These are not
claimed to exhaust possible goals. They are the smallest basis required by this paper's arenas:
conjunction, disjunction, and a conjunction of disjunctions; a new goal may implement the same two
interfaces without changing the engine.

\begin{definition}[Demanded set]
\label{def:demanded}
Let $\sigma = \langle G, R, A, O, C, B \rangle$ be a materialization state (Definition~\ref{def:state}), and
write $N(G)$ for the nodes of $G$. The \emph{demanded set}
\[
\demanded(R,\sigma) \subseteq N(G)
\]
is the smallest set $S$ such that
\begin{enumerate}[topsep=2pt,itemsep=1pt]
  \item every root demand returned by $R$ for the current $G$ belongs to $S$, provided its node is present,
  is nonterminal under $O$, and has no guard resolving to \emph{fails};
  \item if $n \in S$ and $e$ is a data edge, guard edge, or ordering edge
  into $n$, then the producer of $e$ belongs to $S$, provided that producer is present in $G$, is not
  already terminal under $O$, and does not have guards resolving to \emph{fails}.
\end{enumerate}
Equivalently, $\demanded(R,\sigma)$ is the backward closure of $R$ in $G$ along those three edge kinds,
truncated before any absent, terminal, or guard-dead node, including a root.
\end{definition}

For the goal forms evaluated here, root identifiers and field requirements are fixed for the duration of a
request, although the roots currently returned by the goal are re-read as $G$ grows. While an enclosing
\textbf{EachOf} remains open, satisfying one \textbf{AnyAcceptable} subgoal changes satisfaction but does
not remove that pool's sibling roots. A terminal or guard-dead sibling is omitted from the current closure.
Once the whole request is satisfied, the engine stops before another permission decision; stopping does not
mutate the roots in $R$. Removing siblings immediately after one subgoal settles would be a stronger,
subgoal-aware permission rule, which is not evaluated here.

\paragraph{What it costs to know the demanded set.}
Definition~\ref{def:demanded} is a reachability query, computed rather than estimated. One reverse traversal
from the roots of $R$ visits each node and edge at most once --- $O(|N(G)| + |E|)$ pointer operations, no
materializer invoked and no model consulted --- leaving a set-membership test at each node. Two consequences
carry. The first is a small, measured cost.
One computation takes 44~ms on a 10{,}000-node workflow graph when the request reaches all of it and
4.7~ms when it reaches an eighth, growing close to linearly in graph size (log-log slope 1.04--1.16 from
$10^2$ to $10^5$ nodes, $R^2\ge0.99$). The engine recomputes the closure whenever the state changes --- a
node settles, a guard resolves, the graph grows, or the request changes --- and reuses it otherwise, which
keeps $\demanded(R,\sigma)$ a function of the current state rather than of the state at launch. On 27
replayed Montage runs of 58 to 9{,}805 nodes this costs a median 0.017\% and at most 0.33\% of the recorded
CPU the run schedules, and the whole \code{lazy} engine, which also keeps its ready list between state changes,
a median 0.032\% and at most 0.43\%. A first implementation that recomputed before every sequential
permission question reached 15.9\% on the largest of them, computing the closure 594{,}028 times, because
deferred nodes are re-offered after each completion. The two give identical bills, executed sets, and event
traces on 354{,}929 replayed runs across every arena, and a shadow check that recomputed the set on each of
54 million reuses found no mismatch. And membership is a property of $G$ and $R$, not of any node's
neighborhood, so no test local to $n$ decides it without reconstructing the same closure. A correct
per-node check computes this set; an incorrect one either keeps work the request does not need, which costs
money, or refuses work it does, which fails the request. The traversal has two further properties that a
model's judgment of the same question lacks. It is \emph{deterministic}: the same program, request, and
observations yield the same set on every run, whereas repeating one planner at temperature zero changed
which of 31 tasks it completed on 6 (Section~\ref{sec:planner}). And it is \emph{sound} relative to the
declared program: the set contains every node the request may still need, so permission never refuses a
required step; what it can over-include are alternatives that later outcomes show were unnecessary
(Section~\ref{sec:future}).

Three properties carry weight later. The closure is \emph{field-granular}, since a field reference
contributes only its own path upstream, which is what makes projection expressible; it includes
\emph{guard producers}, since a node feeding only a guard is still needed to decide the branch; and it is
current at the next sequential decision or concurrent wave, recomputed whenever the state has changed, since
a resolved guard can remove an entire upstream cone.

\paragraph{Demand is a set at a state, and the state moves.}
Definition~\ref{def:demanded} is stated over the whole state $\sigma$ rather than over $G$ and $R$ alone, so
it survives every way the state can change. Four cases govern demand evolution.

\begin{lemma}[Monotonicity under fixed program and ask]
\label{lem:mono}
Fix $G$ and $R$ and consider a run segment with no graph transformation or growth, request change, or
artifact invalidation. Let $\sigma_0, \sigma_1, \ldots$ be its states under any policy and selector. Then
$\demanded(R,\sigma_{t+1}) \subseteq \demanded(R,\sigma_t)$ for every $t$. In particular, a node outside
the demanded set at any state of the segment remains outside it later in that segment.
\end{lemma}

\begin{proof}[Proof sketch]
With $G$ and $R$ fixed, the roots of the closure are fixed, and a step of the run can change $O$ in only two
ways: a producer becomes terminal, or a guard resolves. Both are truncation conditions in
Definition~\ref{def:demanded}, so each can only remove members from the closure, never add them. Readiness
plays no part in the definition, so a node becoming ready changes nothing.
\end{proof}

The second case is \emph{growth of $G$}: a host emits new work, the closure is recomputed over the larger
graph, and a node formerly outside it can enter only through an edge that did not previously exist ---
demanded work from the moment the edge is, billed to every arm alike, with the engine treating an emission
as a plan resynchronization. The third is a \emph{change of ask}: a new ticket arrives against a standing
program, $R$ changes, and the closure is recomputed from new roots. The fourth is \emph{artifact
invalidation}: a previously terminal node and its upstream cone may re-enter the demanded set even when
$G$ and $R$ are unchanged. Lemma~\ref{lem:mono} applies anew only after that invalidation event. Demand is
therefore a property of a \emph{request against a program} rather than of a program: one $G$ under tickets
$R_1, R_2, \ldots$ has a sequence of demanded sets, with Lemma~\ref{lem:mono} holding inside each segment
without invalidation. Every measurement here is a special case --- the pool and workflow arenas fix $G$ and
$R$, DataSciBench and the portfolios fix $G$ and vary $R$ across runs, and the lifecycle gym varies $R$
within one standing $G$.

\subsection{Policies}

\begin{definition}[Policy]
\label{def:policy}
A policy factors as
\[
\pi = \langle\, \pi_{\mathrm{perm}},\; \pi_{\mathrm{order}},\; \pi_{\mathrm{proj}} \,\rangle,
\]
where $W(\sigma)$ is the dependency-ready set available to be offered,
$\pi_{\mathrm{order}}: (W,\sigma) \to \operatorname{seq}(W)$ ranks that set through the selector,
$\pi_{\mathrm{perm}}: (n,\sigma,D) \to \{0,1\}$ decides whether an offered node may materialize, with
$D=\demanded(R,\sigma)$ consulted only by demand-aware permission, and
$\pi_{\mathrm{proj}}: (n,\sigma) \to \phi$ decides how much of an admitted node must become concrete.
\end{definition}

This tuple isolates the three axes whose effects the paper must not confound: \emph{whether} a ready node
may run, \emph{which} ready candidate is offered first, and \emph{how much} output an admitted node
produces. Ordering precedes permission because permission is evaluated against the state current at each
offer: an earlier materialization may satisfy the goal before a later candidate is considered. This is
neither a unique factorization nor an exhaustive catalog of policy choices. In particular, choosing
materializer $m$ is the optional fidelity operation rather than a fourth core axis; prune, reuse, fusion,
speculation, and commit likewise act through or around the transition interface defined above. We keep the
three-axis tuple because every arm has a permission rule, an ordering rule, and an output projection,
including when the last two use defaults. Holding the functions fixed makes each intervention identifiable
even though ordering can change the state at which later permission questions are asked. We further
require $\pi_{\mathrm{perm}}$ to be \emph{cost-blind}: it may read guards, satisfaction, and demand
membership, but never declared cost. We verify this by permuting declared costs and checking that permission
does not change.

The four reference permission policies are:

\begin{center}
\small
\begin{tabular}{p{0.20\linewidth}p{0.74\linewidth}}
\toprule
\textbf{Policy} & \textbf{Permission rule}\\
\midrule
\code{eager} & Materialize every dependency-ready node. Ignores guards.\\
\code{guard\_eager} & Materialize every ready node except those with a
guard resolving to \emph{fails}.\\
\code{stop\_eager} & \code{guard\_eager}, plus refuse everything once
\emph{the goal is satisfied}.\\
\code{lazy} & Refuse on readiness alone. On demand, refuse if a guard
\emph{fails} or the goal is satisfied; otherwise materialize iff the node is
in $\demanded(R,\sigma)$.\\
\bottomrule
\end{tabular}
\end{center}

\code{lazy} is \emph{stopping plus membership}: it applies the same satisfaction test as
\code{stop\_eager} and then additionally consults the demanded set.

Algorithm~\ref{alg:engine} assembles the pieces into the sequential reference loop used by the primary cost
comparisons. The engine orders dependency-ready candidates before asking permission; each permission rule
then applies its own guard, stopping, and demand tests. \textsc{BackwardClosure} is the whole demand
mechanism. It invokes no materializer and consults no model, and each node and edge enters the frontier at
most once, so refreshing demand whenever the state changes is affordable.

\begin{algorithm}[H]
\caption{Sequential materialization under policy
$\pi = \langle \pi_{\mathrm{perm}}, \pi_{\mathrm{order}}, \pi_{\mathrm{proj}} \rangle$.
Permission owns guards, stopping, and demand; the commit barrier is independent of it. Under \code{lazy},
$\pi_{\mathrm{perm}}(n,\sigma,D)=[\,n\in D\,]$ after guard and satisfaction checks. Roots are filtered like every other
node; with retirement, a settled pool contributes no roots.}
\label{alg:engine}
\begin{algorithmic}[1]
\Require program $G$, root demand $R$, policy $\pi$, commit barriers $B$
\State $A \gets \emptyset$ \Comment{nodes offered since the last progress event}
\While{\textbf{true}}
  \State $W \gets \{\,n:n\text{ is dependency-ready in }\sigma\,\}\setminus A$
  \If{$W=\emptyset$} \State \textbf{break} \EndIf
  \State $n \gets \Call{First}{\pi_{\mathrm{order}}(W,\sigma)}$;\quad $A \gets A \cup \{n\}$
  \State $D \gets \Call{BackwardClosure}{R,\sigma}$ if permission is demand-aware and $\sigma$ or $G$ changed since $D$ was computed
  \If{$\pi_{\mathrm{perm}}(n,\sigma,D)=0$}
    \State $\Call{SettleRefusal}{n}$ \Comment{guard failure skips; otherwise defer}
  \ElsIf{$n$ is irreversible and no barrier in $B$ covering $n$ has released}
    \State $\Call{Hold}{n}$ \Comment{invariant I2; blocked if refused or uncovered}
  \Else
    \State $\Call{Materialize}{n,\pi_{\mathrm{proj}}(n,\sigma),m}$
    \State $\Call{SyncPlan}{G}$ \Comment{incorporate observations and newly emitted work}
    \State skip every node whose guard failed, ready or not \Comment{not under \code{eager}, which ignores guards}
  \EndIf
  \If{state or graph made progress} \State $A \gets \emptyset$ \EndIf
\EndWhile
\State $\Call{Block}{}$ every node still held \Comment{independent of permission}
\Statex
\Procedure{BackwardClosure}{$R, \sigma$}
  \State $S \gets$ roots of $R$ that are present, nonterminal, and not guard-dead; \; frontier $Q \gets S$
  \While{$Q \neq \emptyset$}
    \State pop $n$ from $Q$
    \ForAll{data, guard, and ordering edges $u \to n$ in $G$}
      \If{$u \notin S$ and $u$ is present, nonterminal, and has no guard resolving \emph{fails}}
        \State $S \gets S \cup \{u\}$; push $u$ onto $Q$
      \EndIf
    \EndFor
  \EndWhile
  \State \Return $S$ \Comment{$O(|V|+|E|)$; each node and edge entered at most once}
\EndProcedure
\end{algorithmic}
\end{algorithm}

\paragraph{Parallel extension.}
Parallel execution first snapshots the dependency-ready wave, then offers and launches its admitted members
without forming another wave from an in-flight result. After the wave completes, the engine updates state
and refreshes demand before scheduling the next wave. Sequential execution instead incorporates each
completion before selecting the next node. This is why concurrency weakens goal stopping: siblings already
in flight cannot be recalled when one closes the goal. The primary cost comparisons use the sequential
mode; Section~\ref{sec:sep} repeats the workflow comparison wave-parallel, and the speculation study
measures parallel rounds.

\subsection{Permission is constitutive}
\label{sec:constitutive}
A policy is lazy iff $\pi_{\mathrm{perm}}$ is demand membership --- iff a ready node can be refused because
the active goal does not reach it. No other component has this property: ordering ranks candidates before
permission evaluates them, projection narrows an admitted node, $\mathrm{Prune}$ declines one, reuse serves
one from a cache, and each is equally well defined when $\pi_{\mathrm{perm}}$ is eager. Three consequences
follow, all enforced by the engine. Every arm labeled \code{lazy} carries demand permission. An arm applying
an operation on top of an eager rule --- cheapest-first ordering, projection, pruning, reuse, or another
operation --- is a matched \emph{control}, never a \system{} configuration. And since each operation is
definable on both arms, its absolute saving is not evidence about \system{} --- only the demand-specific
remainder is.

\subsection{Cross-policy invariants}
\label{sec:invariants}

Here an \emph{invariant} is a condition held fixed or enforced across every policy arm, so a cost difference
cannot be purchased by changing correctness, safety, or accounting. These are experimental contracts rather
than quantities claimed constant during every state transition. The engine enforces four. \textbf{I1,
goal:} a run is correct only if the goal is satisfied on the executed graph, judged identically for every
policy; cost is never reported without matched quality. \textbf{I2, safety:} an irreversible node
materializes only under a released commit barrier, and a node with no covering barrier is refused; deferral
is not a safety mechanism and is not credited as one. \textbf{I3, denominator:} every policy is compared
over the same program, and a node materialized under a projection is still charged its whole declared cost
in the denominator; materialization ratio is $\mathrm{MR} = \text{materialized cost} / \text{planned cost}$,
and \emph{wasted} materialization ratio is the share of a policy's own materialized cost that lay outside
the goal's closure and so fed no answer --- zero for \code{lazy} by construction, and for an eager arm the
price of the distinction this paper draws.
\textbf{I4, separated reasons:} never-demanded, guard-false skip, upstream-failure cancellation, explicit
prune, and barrier block are five distinct terminal reasons that partition the unmaterialized part of the
denominator; the paid part is \textsc{materialized} or \textsc{failed}, and the two parts together reconcile
to the planned cost.

\subsection{Predictions}
\label{sec:theory}

Let $M_\pi(\sigma)$ be the set of nodes on which policy $\pi$ invokes
$\mathrm{Materialize}$ from $\sigma$ onward, including attempts that fail and are charged.
The body gives proof sketches to keep the model readable; Appendix~\ref{app:proofs} gives full proofs under
the stated assumptions.

\begin{proposition}[Equivalence on flat disjunctive pools]
\label{prop:equiv}
From a common initial state, suppose every executable node offered before the goal is satisfied belongs to
$\demanded(R,\sigma)$ when it is offered. Under a shared selector and materializers,
$M_{\code{lazy}}=M_{\code{stop\_eager}}$. This holds on a flat disjunctive pool: an \textbf{AnyAcceptable}
goal over candidates $c_1,\dots,c_k$, each named as a root demand and ending its own chain, in which every
edge into a node comes from the node before it.
\end{proposition}

\begin{proof}[Proof sketch]
Every executable node offered belongs to $\demanded(R,\sigma)$, so \code{lazy}'s membership test passes
wherever \code{stop\_eager} would run the node. Both policies refuse once the goal is satisfied, and their
decisions coincide pointwise under the shared selector. On a flat pool, the nodes after an offered node on
its chain have not run and their guards are unresolved, so the offered node reaches its candidate root.
\end{proof}

\begin{proposition}[Structural separation]
\label{prop:sep}
Fix $G$ and $R$, let $\sigma$ be any state of a fixed-program run segment, and let
$n \notin \demanded(R,\sigma)$ become ready. Then $n \notin M_{\code{lazy}}(\sigma)$ under every selector.
If a selector offers $n$ while the goal is open, its guards have not failed, and it is reversible or covered
by a released barrier, then $n \in M_{\code{stop\_eager}}(\sigma)$.
\end{proposition}

\begin{proof}[Proof sketch]
By Lemma~\ref{lem:mono}, $n$ stays outside the demanded set at every later state of the run, and \code{lazy}
invokes materialization only under membership, so it never invokes $n$. \code{stop\_eager} consults guards,
satisfaction, and the independent commit barrier; under the stated conditions it invokes $n$.
\end{proof}

If $G$ grows, the proposition applies again on a subsequent interval with the enlarged $G$ and fixed $R$.
A node entering demand through a new edge is then demanded work and cannot produce separation.

\begin{corollary}[Stopping is order-contingent; reachability is not]
\label{cor:order}
On fixed $G$ and $R$, let an executable off-path node become ready and eventually be offered.
Whether \code{stop\_eager} invokes it depends on whether the offer precedes goal closure;
\code{lazy}'s reachability decision does not. The materialized sets can coincide only if the selector
postpones every such node until after the goal closes. When \code{stop\_eager}'s set contains
\code{lazy}'s, as on conjunctive goals under the conditions of Proposition~\ref{prop:orderinv}, the same
condition is necessary for bill equality with fixed, strictly positive additive node costs. If the goal
never closes, \code{stop\_eager} invokes every such node it is offered.
\end{corollary}

We therefore report ties in two kinds. A \emph{forced} tie occurs when the goal's closure covers the whole
graph, so no policy could have saved anything. An \emph{order-contingent} tie occurs when off-path work
existed and \code{stop\_eager} happened not to be offered it. The same number, two different facts.

\begin{proposition}[Order invariance for conjunctive goals]
\label{prop:orderinv}
Fix $G$, $R$, and a common initial state $\sigma_0$, and let $D_0=\demanded(R,\sigma_0)$. Suppose the goal
cannot be closed by any single materialization before $D_0$ has run ---
\textbf{AllOutputs}, or an \textbf{EachOf} in which every subgoal names exactly one candidate, so that no
subgoal is a pool whose first passing member depends on visit order. Suppose materializer payloads are
deterministic, guards read payloads only, and $G$ remains fixed (conditions C1--C3 of
Section~\ref{sec:reorder}); hold projection, materializer, and barrier choices fixed; and require each
selector eventually to offer every ready permitted node. If every invocation in $D_0$ succeeds and every
guard in $D_0$ holds, then $M_{\code{lazy}}$ is invariant across those selectors.
\end{proposition}

\begin{proof}[Proof sketch]
Such a goal cannot become true early, so the satisfaction test refuses no needed node. A member of $D_0$
stays demanded until it runs: the nodes on its path to a root wait for it, and under C1--C3 and the
hypothesis none of them dies. Fairness eventually offers each member, and \code{lazy} rejects every node
outside $D_0$.
\end{proof}

This does \emph{not} extend to disjunctive goals: under \textbf{AnyAcceptable} the first passing candidate
terminates the run, so which is tried first moves \code{lazy}'s bill exactly as it moves
\code{stop\_eager}'s. An \textbf{EachOf} over pools sits between: its \emph{across-product} part --- which
products run at all --- is order-invariant under \code{lazy}, since an unrequested product is refused in
every order, while its \emph{within-pool} part varies for both arms alike. Under the stated execution
assumptions, order invariance depends on goal shape rather than on permission alone.

\paragraph{Opportunity, capture, and remainder.}
Propositions~\ref{prop:equiv} and~\ref{prop:sep} predict a shape, not just a sign. On $N$ equal-cost
products where the goal names one, the fraction outside the goal is $(N-1)/N$, so a raw saving must grow
with $N$ even if the policy has not improved. Separating opportunity from achievement:
\[
  O = C_{\code{guard\_eager},\mathrm{offpath}}, \qquad A_\pi = 1 - \frac{C_{\pi,\mathrm{offpath}}}{O},
  \qquad \Delta_{\mathrm{perm}} = A_{\code{lazy}} - A_{\code{stop\_eager}}.
\]
$O$ is the \emph{executable opportunity}, what the full guard-respecting eager row pays outside the
requested product's closure; $A_\pi$ is the share policy $\pi$ avoids; and
$\Delta_{\mathrm{perm}}$ is permission's additional capture beyond goal stopping. Every paid event is
classified as shared setup, demanded-product work, or off-path work, and the role totals reconcile to each
policy's cost. Off-path is goal-relative rather than hindsight-relative: a failed candidate inside the
requested product remains demanded work. For a physical operation, the same difference is
\[
  \Delta_{\mathrm{op}} =
  A^{\mathrm{op}}_{\code{lazy}} - A^{\mathrm{op}}_{\code{stop\_eager}},
\]
called its \emph{demand-specific remainder}. The superscript ``op'' means the saving produced by adding that
operation to each permission arm; it does not introduce a new capture metric. Thus both quantities use
$\Delta$: the subscript says whether the comparison isolates permission or one physical operation.

\paragraph{What a planner can move, and what it cannot.}
Two results below have their \emph{sign} fixed under the propositions' coupled fixed-program assumptions,
whichever model writes the program. \code{lazy} invokes a subset of the nodes
\code{stop\_eager} invokes, so with nonnegative additive costs its bill cannot be higher. On conjunctive
goals its invoked set is order-invariant; under the fixed per-node charges used in the comparison, a price
heuristic therefore moves only the eager bill. The remainder is that effect negated --- negative where
cheapest-first helps the eager arm, positive where it backfires --- and never a response of \code{lazy} to
order. A planner determines the amount of off-path work and therefore the \emph{size} of both effects, but
cannot make refusing positive-cost unasked work more expensive than performing it.

Operations that rewrite a node's ask carry no such guarantee. Projection does not filter an answer already
computed; it re-runs the node under a narrower request, a different generation whose price is a fact about
how one model responds to being asked for less --- measurable, not derivable. The division is therefore
between results resting on a subset relation and results resting on a model's response to a changed prompt,
and only the second kind can reverse. Section~\ref{sec:planner} tests this against three planners and finds
the division where this argument places it.

\subsection{The operation space}
\label{sec:operations}

Permission is one operation. Stating it alone would invite each later operation to arrive with its own
scheduler and its own favorable baseline, so we name the full space and the control each requires.

\begin{center}
\small
\begin{tabular}{p{0.15\linewidth}p{0.34\linewidth}p{0.42\linewidth}}
\toprule
\textbf{Operation} & \textbf{Question} & \textbf{Control obligation}\\
\midrule
Demand / defer & May this ready node run at all? & Same graph and goal;
compare against goal-stopping eager, not naive eager.\\
Reorder & Which ready candidate is offered first? & Identical ex-ante selector on
both permission arms; no hindsight cost in the selector.\\
Project & Which output fields become concrete? & Whole $\times$ projected
factorial; the materializer must physically do less; denominator stays
whole.\\
Prune & May a live option be declined permanently? & Distinct from guard skip
and goal stop; a control must show the quality loss when the pruned option was
needed.\\
Fuse / CSE & Can shared physical work be combined? & A transformation
denominator plus a semantic-equivalence check.\\
Fidelity & Which materializer or approximation suffices? & Matched task
quality and the complete resource vector.\\
Reuse / cache & Can an existing artifact satisfy demand again? & Provenance,
validity conditions, explicit cache-hit accounting, and a matched stopping-plus-reuse row.\\
Speculate & Is undemanded work worth racing for latency? & Control is
\emph{degenerate} --- an eager arm has already run everything. Report a
cost/latency frontier against the no-speculation arm with fixed authored
priors.\\
Commit & May an irreversible action occur? & Deterministic authorization and
a validation barrier.\\
\bottomrule
\end{tabular}
\end{center}

Section~\ref{sec:results} credits the first row to permission; Section~\ref{sec:physops} measures the rest.
The individual algorithms have prior literatures; the novelty is making one live demanded set govern their
legality and measuring each operation against the eager control that isolates its demand-specific value
(Section~\ref{sec:related}).

\section{Related work}
\label{sec:related}

Table~\ref{tab:related} places the nearest systems on the axes defined in Section~\ref{sec:model}.
Component precedents exist, but no prior agent runtime spans the row \system{} occupies: request-derived
closure over a standing program, recomputation after observations, ready-queue and field-level physical
operations, and an independent commit barrier, all under one live demanded-set semantics.

\begin{table}[H]
\centering
\caption{Positioning. \textbf{Goal-relativity:} \textbf{S} --- execution is
restricted to a backward closure rooted at requested or designated outputs; \textbf{P} --- restricted by a
predicate over the request, model features, or observations; \textbf{---} --- no such restriction is
established. Other columns: \checkmark{}
present, $\circ$ partial, blank absent. Checkmarks mark component precedents; novelty resides in the unified
\system{} row.}
\label{tab:related}
\footnotesize
\begin{tabular}{p{0.29\linewidth}C{0.075\linewidth}C{0.10\linewidth}C{0.085\linewidth}C{0.085\linewidth}C{0.085\linewidth}}
\toprule
\textbf{System or line of work} &
\textbf{Goal-\newline relativity} &
\textbf{Demand re-evaluated on observation} &
\textbf{Ready-queue ordering} &
\textbf{Field projection} &
\textbf{Commit barrier}\\
\midrule
\multicolumn{6}{@{}l}{\emph{Declarative and incremental execution}}\\
Make, Bazel~\cite{feldman1979make,bazel} & S & $\circ$ & $\circ$ & & \\
Build systems \`a la carte~\cite{mokhov2018build} & S & \checkmark & \checkmark & & \\
Shake~\cite{mitchell2012shake} & S & \checkmark & $\circ$ & & \\
Adapton~\cite{hammer2014adapton} & S & \checkmark & $\circ$ & & \\
TensorFlow graph mode~\cite{abadi2016tensorflow} & S & & $\circ$ & & \\
Volcano, Eddies~\cite{graefe1994volcano,avnur2000eddies} & S & \checkmark & \checkmark & \checkmark & \\
\addlinespace
\multicolumn{6}{@{}l}{\emph{Agent runtimes and agentic execution}}\\
ReAct~\cite{yao2023react} & --- & & & & \\
ReWOO~\cite{xu2023rewoo} & --- & & & & \\
LLMCompiler~\cite{kim2024llmcompiler} & --- & $\circ$ & \checkmark & & \\
Helium~\cite{wadlom2026helium} & S & & \checkmark & & \\
LLM-as-Scheduler~\cite{xiang2026las} & P & \checkmark & & & \\
MaAS, AgentPrune~\cite{zhang2025maas,zhang2025agentprune} & P & & & & \\
Tool-call necessity~\cite{tocall2026} & P & & & & \\
LOTUS~\cite{patel2025lotus} & --- & & & & \\
\addlinespace
\textbf{\system{} (this paper)} & \textbf{S} & \checkmark & $\checkmark^{\dagger}$ &
$\checkmark^{\dagger}$ & \checkmark\\
\bottomrule
\end{tabular}

\vspace{2pt}
{\footnotesize $^{\dagger}$ Implemented as actions and measured with matched controls
(Section~\ref{sec:physops}), always with permission underneath the \code{lazy} arm.}
\end{table}

\paragraph{Demand-driven and goal-relative execution.}
Backward closure has established foundations in call-by-need~\cite{wadsworth1971lazy}, its natural
semantics~\cite{launchbury1993lazy}, demand-driven evaluation~\cite{pingali1985demand},
dependency-directed rebuilding~\cite{feldman1979make}, program slicing~\cite{weiser1984slicing}, and
goal-relative build systems~\cite{mokhov2018build}. Build systems and data orchestrators run only the targets a
request selects~\cite{bazel,dbt2026selection}, and systems with dynamic dependencies go further: Shake and
monadic build systems discover new dependencies while they build~\cite{mitchell2012shake,mokhov2018build}.
Agent runtimes face what builds do not: satisfaction is judged rather
than computed, over pools of candidates that settle at their first acceptable one; the graph grows as an LLM
planner reacts to failures; nodes spend money or commit irreversible effects; and the same set decides how each
admitted node runs. \system{} makes closure well defined, cheap to recompute, and enforceable as permission under
all four, with equivalence and separation regions measurable in real workloads.

\paragraph{Agentic execution and cost.}
Several agent systems hold one piece of demand-driven execution, and each lacks a piece the others have.
LLMCompiler is the closest planned runtime~\cite{kim2024llmcompiler}. A planner writes tasks and
their dependencies, and a fetching unit dispatches every task whose inputs exist, several at once. That is
\(\mathrm{ready}\Rightarrow\mathrm{execute}\) with concurrency, which our \code{guard\_eager} policy models in
ready waves. \system{} keeps the planned graph and the concurrency and changes the rule: parallel dispatch becomes
an ordering choice inside the demanded set.

Helium is the closest optimizer of agent workflows~\cite{wadlom2026helium}. Its job is to serve one compiled
plan more cheaply without changing that plan's answers. It walks backward from the plan's outputs and drops
operators that cannot affect them, merges duplicate subgraphs, and substitutes cached results. Its view is
the plan. \system{}'s view is a standing program wider than any one request, together with the request in
front of it. The demanded set is rooted at that request, not at everything the program can produce, and it
is recomputed as results arrive, so a finished branch stops owing its remaining steps. In that view,
Helium's prune corresponds to the demanded set of a request that asks for every output of the plan, and its
subgraph sharing and cache substitution correspond to fusion and reuse, two of the physical operations the
demanded set governs. A step that is ready, valid, and serves only an output nobody asked
for today is outside Helium's question; it is the case the permission layer exists for. Agent runs also change
their plans --- a check fails and kills a branch before its setup has started, the planner adds steps
mid-request, one candidate fixes an issue so its siblings are no longer needed --- and there only a set
recomputed as the run unfolds stays exact (Section~\ref{sec:dynamic}).

LLM-as-Scheduler is the closest dynamic controller~\cite{xiang2026las}. After each step a cheap gate and
then an LLM decide whether to continue, verify, repair, or reroute, and MaAS and AgentPrune learn which agents
and messages a query uses~\cite{zhang2025maas,zhang2025agentprune}. These systems predict what a request needs
and may trade quality for cost. \system{} derives it exactly, at each ready step, by graph traversal with no
model call: whether the step lies on any path to the active request, and it ties wherever the request needs
everything. LLM-as-Scheduler's early exit corresponds to goal stopping, our \code{stop\_eager} baseline; until it
fires, ready steps for outputs outside the request still run. Repair and rerouting
remain available as policies over the steps permission admits. Tool-call necessity estimates whether one
call is warranted~\cite{tocall2026}, and LOTUS chooses a physical implementation of one semantic
operator~\cite{patel2025lotus}; both decisions also sit inside a step the demanded set has admitted. The
predicate and the set coincide when every ready step is demanded (Proposition~\ref{prop:equiv}) and
separate when the program is wider than the request (Proposition~\ref{prop:sep}). To our knowledge, \system{}
is the first to put these pieces in one formulation: a demanded set derived exactly from the declared graph for
each request, kept current as the run unfolds, read by every physical operation, and backed by proved properties
that say where it must tie.

EnumGRPO provides component precedents for our physical-operation study by searching projection width,
operator placement, and relation-level selectivity under a joint quality-cost
objective~\cite{nie2026agenticquery}. However, it asks which variant wins; ours asks what the live goal
authorizes and subtracts the identical eager mechanism's effect from every lever. The signed remainder
attributes value to demand rather than merely finding a cheaper setting.

Other work reduces inference cost without graph-demand semantics. FrugalGPT adapts prompts, approximates
models, or conditionally adds calls in a cheap-to-expensive cascade~\cite{chen2023frugalgpt}; RouteLLM sends
each query to one selected model before generation~\cite{ong2025routellm}. These methods choose how many or
which model calls answer a query; permission decides which nodes of an already described workflow the
current goal reaches. The levels compose, which is also why our comparisons hold the model fixed.

\paragraph{Operations and their sources.}
Adaptive reordering under revised estimates is classical query
processing~\cite{avnur2000eddies,graefe1994volcano}. Reuse and cache validity are the subject of incremental
computation: dynamic dependence graphs, change propagation, and memoization identify stale and unaffected
subcomputations~\cite{acar2005selfadjusting,hammer2014adapton}; differential dataflow
extends this to partially ordered versions and nested
iteration~\cite{mcsherry2013differential,abadi2015differentialfoundations}. That literature asks what is
\emph{stale} after an input changes; permission asks what is \emph{wanted} after the request changes. A node
can be current and still not owed.

Serving-layer reuse shares attention state through paged KV memory or automatic prefix
matching~\cite{kwon2023pagedattention,zheng2024sglang}. It makes a repeated call
cheaper; artifact reuse can eliminate the call, so the mechanisms compose. Meta-Dataflows shares
exploratory work across jobs~\cite{castrofernandez2018metadataflows}; Helium applies pruning and
common-subexpression elimination to agent plans~\cite{wadlom2026helium}; speculation exchanges redundant
work for latency as speculative decoding does within a call~\cite{leviathan2023speculative}; and our
static/dynamic dependency boundary follows selective applicative
functors~\cite{mokhov2019selective}. These mechanisms are established individually; \system{} contributes
their common request-relative semantics and a matched-control discipline that exposes each operation's
signed, possibly negative, demand-specific remainder.

\section{Arenas: public benchmarks and composed programs}
\label{sec:arenas}

Six public arenas are wired to one engine through adapters that preserve their native cost columns
(Table~\ref{tab:arenas}). We classify an arena's \emph{shape} by the relationship between its active goal
and the outputs its graph could produce, not by application domain. A \emph{pool} asks for one acceptable
answer among alternatives: until one succeeds, every candidate may still be needed, so the demand closure
contains the whole pool. A \emph{multi-product graph} produces several independently useful outputs while
the current request names only some of them; branches serving the other products can therefore fall outside
the goal's closure. A \emph{build graph} asks for one native target inside a larger package, and its
dependencies form that target's closure while unrelated package targets remain outside it. Three arenas are
pools, WfCommons and DataSciBench are multi-product, and Bazel supplies the exploratory build-graph domain;
it is not an agentic workload, and every claim resting on it is labeled accordingly.

These arenas are a controlled map of the mechanism's boundary, not an exhaustive account of where
demand-driven execution may help. Public corpora supply natural graphs, costs, outcomes, and the important
cases where we predict no saving. The composed portfolios and authored gyms deliberately vary the structure
those corpora hold nearly constant; we label every such construction and retain native costs wherever they
exist. They demonstrate causal conditions, not how often those conditions occur in deployment. Nor do we
claim that \system{} beats every eager executor with every heuristic: the primary comparison is the
strongest goal-stopping eager baseline, and each physical operation receives a matched eager control. An
eager executor given the same correct goal closure would make the same permission decisions --- it would,
in effect, have implemented the mechanism studied here.

\begin{table}[t]
\centering
\caption{Arenas. Cost columns are arena-native and never rescaled; ``live''
means the materializer ran real code or a real model during our measurement.}
\label{tab:arenas}
\footnotesize
\setlength{\tabcolsep}{4pt}
\begin{tabular}{p{0.19\linewidth}p{0.08\linewidth}p{0.11\linewidth}p{0.11\linewidth}p{0.39\linewidth}}
\toprule
\textbf{Arena} & \textbf{Mode} & \textbf{Cost unit} & \textbf{Shape} & \textbf{What it is}\\
\midrule
WfCommons / WfInstances~\cite{coleman2022wfcommons,wfinstances} & replay &
measured CPU\,s & multi-product & 154 readable production scientific-workflow
traces, 86 with several output products; 303 seeded product requests.\\
DataSciBench~\cite{zhang2026datascibench} & live & real tokens & multi-product
& The benchmark supplies executable dependency-structured plans; our adaptation
retains 95 goals over 28 tasks, 24 of them with several deliverables.\\
SWE-bench Verified~\cite{jimenez2024swebench,swebenchverified} & replay; live gate & recorded;
container seconds & pool & 500 human-validated tasks; also the substrate for a live
release gate with real agent patches run in Docker.\\
Defects4J v1.1~\cite{just2014defects4j,sobreira2018defects4j} & replay & wall seconds & pool & 395 bugs;
185 settle under our replay criterion; also composed into portfolios.\\
SimulCost~\cite{simulcost2026} & replay; live & analytic cost & pool & 2{,}304
multi-round problems from the pinned arXiv v3 release; also composed into portfolios.\\
Bazel~\cite{bazel} & replay & action CPU\,s & build graph & 419 native targets
on two production packages; exploratory second structural domain.\\
\bottomrule
\end{tabular}
\end{table}

\paragraph{Three preconditions, declared before the results.}
The semantics is sharp enough to state its own payoff conditions in advance, and we fix them here before any
arena is scored. The mechanism pays exactly where three conditions hold together: the program must contain
work that is \emph{ready but not demanded}, since permission refuses only what lies outside
$\demanded(R,\sigma)$; that work must be \emph{physically separable}, so a step can do part and skip the
rest; and the skipped part must be \emph{expensive enough} that removing it survives the accounting. Where
all three hold we predict a positive remainder. Where any one fails we predict exactly zero --- not a small
effect to be argued over, but an identity. Because the list is fixed in advance, the zeros below are
commitments the theory makes and then meets, rather than losses reported after the fact;
Section~\ref{sec:discussion} reads the two groups against each other.

\paragraph{What current corpora hold constant.}
Permission's opportunity has an exact measure --- the share of the graph a request does not reach --- so the
evidence base itself becomes something we can quantify rather than assume. Three features of current
evaluation hold that share near zero.
First, pool tasks keep every candidate relevant until one succeeds; depth does not help, because every rung
remains a possible witness. Second, single-request evaluation asks whether an agent solves one task rather
than how one standing program serves a sequence of narrower requests. Third, the planning prompt enumerates
the current output files and asks for a plan of that task, so the object being authored is anchored to the
request by construction. The resulting near-complete coverage is visible in the data: pool closures contain
every candidate, while across the 28 DataSciBench plans, each task one request, 137 of 139 nodes
(\textbf{98.6\%}) are demanded by some deliverable and \textbf{27 of 28} contain no work outside those requests. A controlled prompt
intervention below changes the stated billing semantics without changing that coverage
(Section~\ref{sec:planning-incentive}), so we do not attribute this narrowness to billing.

These corpora do exactly what they were built to do: establish whether an agent can produce and execute a
correct plan. What they hold nearly constant is the independent variable of this paper --- the requested
share of a standing program. Proposition~\ref{prop:equiv} therefore predicts exact ties on their pool
strata. To test separation rather than rediscover that zero, we compose native tasks into multi-product
programs and build authored gyms that switch one precondition at a time; every result identifies which
topology, request grouping, price, or outcome came from us.

\paragraph{Portfolios from native costs.}
A SimulCost program bundles several design studies from one physics family; a Defects4J program bundles
several repair targets from one project. Each product is an \textbf{AnyAcceptable} pool and the program goal
is \textbf{EachOf} over products, so one program contains a disjunctive equivalence region \emph{inside} a
conjunctive separation region. The topology is authored; the costs are native and unmodified. At width
$N=3,4,5,6$ this yields 893, 667, 535, and 443 programs. Every product position is requested in turn, every
zero row is retained, and an all-products request is carried as a control that must tie exactly.

\paragraph{A live release gate.}
Over SWE-bench we build the release-gate program of Section~\ref{sec:motivation-agent}: three issues in one
repository, a real
agent attempting candidate patches, each candidate run in a Docker container that installs the repository,
applies the patch, and executes the test suite. The cost is measured container time --- a cloud bill, not a
transcribed trace. Repositories are drawn by a deterministic stratified sampler after an early run selected
only Astropy by taking the head of an alphabetical list.

\paragraph{Authored gyms.}
Two gyms isolate the physical operations. Prices are declared \emph{before} anything runs, so ordering
claims are ex-ante rather than transcribed from recorded runtime, and their magnitudes are reported as
controlled factorial shapes, never as measurements of a real system. The \emph{engineering gym} contains:
\emph{screening}, three demanded alternatives ready after a cheap checkout --- a static check (2 CPU units),
an external review (25), and the full test suite (200) --- under \textbf{AnyAcceptable}; \emph{profile}, a
shared node producing a cheap pass/fail summary and an expensive execution trace, with a second consumer (a
coverage report) named in the program text; a prune family, a fidelity ladder, a fusion pair, a speculation
family, and a family of $K$ contingent irreversible actions behind commit barriers. The \emph{lifecycle gym}
supplies what a single-ask benchmark cannot: a standing graph plus an ordered list of tickets, with
\emph{session coverage} (the share of the program's products the whole sequence ever demands) and
\emph{invalidation} (the share of tickets whose upstream input changed) as explicit axes. Every arm is reset
identically between tickets, artifacts survive only when reuse is enabled, and invalidation voids the
changed node and its whole downstream cone for every arm.

\section{Implementation and methodology}
\label{sec:impl}
\label{sec:method}

\paragraph{Engine.}
One execution engine runs every condition. It receives the graph, goal, materializers, cost models, guards,
commit barriers, event log, and selector; only $\pi$ changes between arms. In sequential mode the main loop
ranks the dependency-ready set, offers its first node to the policy, incorporates any resulting observation,
and then forms the next offer. Demand-aware permission therefore sees a closure current at each offer: the
engine recomputes it when the state changes and reuses it otherwise.
A node is offered at most once per progress epoch and re-offered when progress occurs, which is how a
deferred node can later become demanded as guards resolve. The default order is alphabetical by node
identifier. This is deterministic and neutral, but it means \code{stop\_eager}'s bill depends on node
naming --- Corollary~\ref{cor:order} made visible. A
controlled renaming demonstrates the effect: sorting the undemanded branch early costs
\code{stop\_eager} 415, sorting it late costs 15, and \code{lazy} costs 15 either way. Parallel mode
snapshots a ready wave and refreshes state and demand between waves. Whether every member launches before
any completes depends on the materializer yielding before it reports; the replay materializers do so only in
the wave-parallel replay of Section~\ref{sec:sep}. Concurrency weakens goal stopping, so we report the
sequential comparison, which that replay confirms is the harder case for us on average, though not on every
request.

\paragraph{Adoption in LangGraph.}
Permission needs no runtime of its own. We compile a declared program into a stock LangGraph graph with static
edges and wrap each node's callable, which asks the policy, before the node does any work, whether the current
request demands it; LangGraph's scheduling and state are unchanged~\cite{langgraph2026}. A declined node returns
without doing work and records itself as declined, and its successors, which lie outside the demanded set too,
decline in turn; the wrapper costs at most 0.8~ms per node call, an upper bound over 112{,}262 calls that includes
the replays and compilation. In a pre-registered cost-only replay, the wrapped graphs run exactly the steps
the engine's \code{lazy} arm pays for, at identical bills, on all 303 WfCommons requests and all 95
DataSciBench goals, while stock LangGraph runs every step: the wrapper declines 32{,}515 of 55{,}541 WfCommons
steps (76.1\% of the pooled CPU) and 236 of 590 DataSciBench steps (45.1\% of the tokens). This is adoption
evidence for programs without guards that do not grow during a run; conditional edges, \code{Send} fan-out,
projection, and the built-in retry are not yet supported in the wrapper.

\paragraph{A live LangGraph application.}
The same check also wraps an application's own LangGraph \code{StateGraph} as its author wrote it, with
fan-out and joins, and returns a copy with the same schema, nodes, and edges; conditional edges, \code{Send},
and cycles are refused. On 60 seeded graphs it runs the same nodes in the same supersteps as the engine's
parallel mode under all three permission rules. We ran it live on a release-readiness application in stock
LangGraph idiom, fixed before any request was drawn: 12 nodes, 11 of them GPT-5.6-terra calls, three parallel
digests over a release candidate's fixes feeding seven deliverables, each release candidate 8 Django fixes from
SWE-bench Verified. Forty requests were drawn by rule (every single deliverable, every pair, 7 seeded triples,
and 5 requests for everything), written in natural language by the model in a seeded persona's words, and
mapped back to deliverables by the model: exactly on 38, with one deliverable added on 2, and none omitted.
Every arm ran live in LangGraph's own runtime, with a run-unique prompt prefix so no arm was served from
another's prompt cache. \code{lazy} made the predicted decision on all 122 runs. Against \code{stop\_eager}
it cost 24.0\% fewer dollars (95\% CI [19.8, 28.3]; 34 wins, 1 tie, no losses at a pre-registered $\pm5\%$
band) and 20.7\% fewer tokens [16.9, 24.7], tracking the share of the application each request left out (mean
29.7\%, Spearman $+0.84$); the five requests for everything tied on tokens. Total time fell 12.5\% [7.1, 18.0]
because \code{stop\_eager} waits for the rest of its last superstep, while time to the requested output did not
change (2.9\% sooner, CI [$-2.0$, 8.0]). Every arm delivered every requested deliverable, and blind 1--5 scores
from the same model matched on 39 of 40 requests, as often as two configurations running identical steps match
each other (\code{guard\_eager} against \code{stop\_eager}, 38 of 40). The two over-mapped requests cost
\code{lazy} part of its saving (24.9\% instead of 34.3\%, and 32.6\% instead of 48.1\%). The mean saving is
smaller than on WfCommons because the shared digests dominate this application's bill; it is a property of the
application, not of the rule. The study cost \$11.05 of model calls in all, and its code is a standalone
package with 491 tests, run with the model mocked.

\paragraph{The \code{react} arm.}
Where a \code{react} row appears (Section~\ref{sec:live}, Appendix~\ref{sec:prune}) it is not a policy over
the shared program and is not run through the matched-control harness, because not having a program is the
definition of ReAct. It is a separate ReAct-style executor instrumented into the same event log: it holds no
plan, reasons one step, takes one action, observes, and repeats, using the same materializers so that an
action costs what the same action costs under every policy. We make it deliberately strong --- an
oracle-competent agent that notices at once when an observation makes a downstream node pointless, as a
guard-respecting DAG would --- and we charge it what a planned executor never pays: a deliberation turn
before every action that re-reads the growing trajectory, priced as a declared token formula that grows with
the step count and kept in the token resource so it never enters a CPU or wall-clock comparison. Because it
plans only what it reaches, its materialization ratio has no shared denominator and is never reported; it is
compared on absolute cost per resource. On single-ask verification work it lands beside \code{stop\_eager},
since stopping when the goal is met is what both do; what it cannot do is hold a plan and refuse part of it,
so under an open goal it materializes everything it reaches and ties the eager arms.

\paragraph{Pairing and statistics.}
Every comparison is paired at the task level: the same program under two policies. We report the mean
per-task saving with a bootstrap 95\% confidence interval, the median, win/loss/tie counts, a Wilcoxon
signed-rank $p$, and rank-biserial effect size. Goal rows drawn from one task are clustered within it, and
product requests from one portfolio are averaged within the program before resampling across programs. When
every paired delta is zero we report ``identical on all $n$'' rather than a $p$-value. A pooled ratio of
totals weights expensive tasks and is usually larger than the paired mean (62.2\% versus 42.0\% on the
workflow arena); we quote paired means and label pooled figures as such.

\paragraph{Ex-ante cost and quality.}
Replay arenas record what work actually cost. That number is a legitimate outcome measure but never enters a
selector, because ranking by measured duration is hindsight; the authored gyms write cost before execution
for the same reason. Permission is cost-blind, so this does not touch the permission claim. Quality is
judged identically for every arm (an LLM judge on the live agentic arena), and a cost is never reported
without matched quality. Costs are vectors; every comparison reduces one to a scalar, and
Appendix~\ref{sec:prune} shows an operation whose verdict depends on which component survives. Elapsed time
is among the reported resources; adding it changed no archived number, because every arena except the
lead-time family declares CPU only.

\paragraph{The ordering grid.}
An ordering number measured only against \code{stop\_eager} is ambiguous, because a stopping rule is already
a weak permission proxy. We therefore cross three termination rules --- none (\code{guard\_eager}), a
stopping rule (\code{stop\_eager}), and demand (\code{lazy}) --- with three orderings (fixed, cheapest node,
cheapest branch) on the five priced arenas --- the two authored gyms plus SimulCost, WfCommons and
DataSciBench, which is a different set from Table~\ref{tab:arenas}'s six, because an ordering grid needs a
cost column a selector can rank and the pool replays do not supply one per node. The identity check on the
``none'' row is per task, not pooled, because a sum can hide two tasks moving in opposite directions.

\section{Results: permission}
\label{sec:results}

Table~\ref{tab:glance} is the permission argument in one page; every row is reported against
\code{stop\_eager}, the strongest baseline. Two regions organize the results. We lead with
\emph{separation}: when a program is wider than the current request, demand excludes off-path work that an
eager rule may bill. We then give the \emph{equivalence} boundary: when the request's closure covers the
whole graph, demand has nothing to exclude and the policies must tie exactly. The remaining subsections move
the request and program shape between those regions.

\begin{table}[t]
\centering
\caption{Evidence at a glance. One row per study, reported against
\code{stop\_eager}. The last column is the control that had to pass for the row to mean anything; every one
of them did.}
\label{tab:glance}
\footnotesize
\begin{tabular}{p{0.15\linewidth}p{0.055\linewidth}p{0.10\linewidth}p{0.28\linewidth}p{0.24\linewidth}}
\toprule
\textbf{Study} & \textbf{Mode} & \textbf{$n$} &
\textbf{Headline vs \code{stop\_eager}} & \textbf{Control}\\
\midrule
WfCommons permission & replay & 303 requests &
\textbf{42.0\%} [37.5, 46.7]; 211 / 0 / 92 &
68 single-product instances give exactly 0.0\%\\
\addlinespace
WfCommons shape law & replay & 235 observations &
50.3\% (one product), 21.4\% (half), \textbf{0.0\%} (all) &
86 / 86 exact ties when the goal names everything\\
\addlinespace
Pool arenas (SWE-bench, SimulCost, Defects4J) & replay &
3{,}199 asks & \textbf{0.0\%}, bit-identical bills &
Predicted by Prop.~\ref{prop:equiv}; a saving here would be a bug\\
\addlinespace
DataSciBench, agent-authored & live & 95 goals, 28 tasks &
\textbf{12.8\%} clustered [7.7, 18.2]; 15 / 0 / 13 &
Quality matched on every task\\
\addlinespace
Portfolio position sweep & replay & 3{,}572 rows &
\textbf{0.0\%} asking first $\to$ \textbf{85.5\%} asking last &
893 / 893 all-products ties; 0 losses\\
\addlinespace
Program width, $N{=}3..6$ & replay & 10{,}680 rows &
capture \textbf{100\%} (\code{lazy}) vs \textbf{0\%} open / $\approx${}\textbf{50\%} settleable &
2{,}538 / 2{,}538 controls tie; deltas 100\% off-path\\
\addlinespace
Live release gate & live & 24 rows, 4 repos + Django &
saving \textbf{51.7\%} / \textbf{52.1\%} of container time; capture
\textbf{100\%} (\code{lazy}) vs \textbf{0\%} open / \textbf{46.5\%},
\textbf{41.5\%} settleable &
8 / 8 controls tie; quality matched; 0 losses\\
\addlinespace
Bazel native targets (exploratory) & replay & 419 targets, 2 packages &
\textbf{96.0\%} [95.0, 96.8]; 418 / 1 / 0; kind skew 99.1\% (test) $\to$
28.2\% (package) &
capture 100\% vs $\approx${}40\% for \code{stop\_eager}\\
\bottomrule
\end{tabular}
\end{table}

\subsection{The separation region}
\label{sec:sep}

Production scientific workflows contain ready work a goal does not require: an instance often has several
output products and a consumer wants some of them. We replay 154 WfInstances traces, 86 with more than one
output product, and issue 303 seeded product requests using measured CPU seconds as the cost. A workflow's
products are its sinks, the tasks no other task in the trace consumes, read off the DAG by one rule in every
family; the 154 are the readable traces with at most 400 tasks, and an instance contributes every product if it
has at most three, otherwise three drawn on a per-instance seed, never on what a request would save.

Paired against \code{stop\_eager}, \code{lazy} saves a mean of \textbf{42.0\%} (95\% CI [37.5\%, 46.7\%]),
median 37.0\%, with \textbf{211 wins, 0 losses, 92 ties}, $p < 0.001$, rank-biserial $+1.00$. The wasted
materialization ratio is Proposition~\ref{prop:sep} as a measurement: \code{stop\_eager} spends 62.2\% of
what it materializes on work that never fed the answer, and \code{lazy} spends none. The negative control is
the single-product stratum: on the 68 requests whose instance has one output product, the saving is exactly
0.0\% (cheaper on 0/68), because there the requested product's closure is the whole workflow; on the 235
multi-product requests it is 54.2\% (cheaper on 211/235). Figure~\ref{fig:schematic}'s two Montage requests
are rows of this population: the band-2 image saves 74.8\% and the color mosaic 0.1\%.

\paragraph{Aggregate bills for this study.}
\begin{center}
\small
\begin{tabular}{lrrr}
\toprule
\textbf{Policy} & \textbf{Materialized CPU\,(s)} & \textbf{MR} & \textbf{Wasted MR}\\
\midrule
\code{guard\_eager} & 29{,}199{,}411 & 1.000 & 0.761\\
\code{stop\_eager} & 18{,}490{,}856 & 0.633 & 0.622\\
\code{lazy} & 6{,}980{,}658 & 0.239 & 0.000\\
\bottomrule
\end{tabular}
\end{center}

The 92 ties partition exactly as Corollary~\ref{cor:order} predicts. Against \code{guard\_eager} all 235
multi-product requests are cheaper, so every one of them contains off-path work. Sixty-eight ties are
therefore \emph{forced} --- the single-product stratum, where there was nothing to decline --- and the
remaining 24 are \emph{order-contingent}: off-path work existed and the alphabetical visit order happened to
close the goal before \code{stop\_eager} was offered it. None of the 92 is evidence that a stopping rule
suffices.

\paragraph{Order and concurrency.}
Shuffling the visit order breaks every order-contingent tie: under at least one of eight seeded orders,
\code{stop\_eager} pays off-path work on each of the 24, and its bill moves on 232 of the 303 requests
across seeds. Run wave-parallel instead --- each materializer yields before reporting, so every member of a
ready wave launches before any completes, and launched work is never cancelled --- \system{} saves
\textbf{44.5\%} (95\% CI [39.8\%, 49.3\%], median 40.3\%) against the same rule, with 235 wins, 0 losses,
and 68 ties, all of them the single-product controls; \code{stop\_eager}'s wasted materialization ratio rises
from 62.2\% to 66.5\%. The paired difference from the sequential comparison is $+2.5$ points (95\% CI
[$+1.6$, $+3.5$]): higher on 164 requests, equal on 109, and lower on 30, where a wave closes the goal at the
requested product's depth before deeper off-path work becomes ready, work sequential \code{stop\_eager}
had already paid for. Under waves \code{stop\_eager}'s bill does not move under any of the eight seeds, so
the margin is not ordering luck. The engine's concurrency is wave-synchronous with unbounded width, not
dataflow with recorded task durations, so this compares the two rules, not a deployment's latency. In every
arm the engine bills each trace task recorded at zero seconds as one unit (1{,}076 tasks, about 0.004\% of
planned cost).

\paragraph{Time to the requested output.}
Cost is not the only gain when workers are scarce. A pre-registered discrete-event simulation on the recorded
task runtimes of the same 303 requests, which at one worker reproduces the engine's sequential bills exactly,
gives \system{} the requested output 42.0\% sooner than \code{stop\_eager} with one worker (95\% CI [37.4\%,
46.6\%]), 34.2\% with 4 ([29.8\%, 38.7\%]), and 18.6\% with 16 ([15.3\%, 21.8\%]), and never later on any
request. With unbounded workers both follow the requested product's critical path and tie, while
\code{stop\_eager} pays 3.2$\times$ the CPU, because work in flight when the goal closes runs to completion. The
time comes from knowing the demanded set rather than from refusing work: a demand-first scheduler that runs
demanded steps first and fills idle workers with off-path work keeps nearly all of the gain (33.4\% and
17.6\% at 4 and 16 workers) while prepaying 38.7\% and 74.1\% of eager's off-path work for later requests; it
is slower than \system{} on 42 and 30 requests, never faster, because a worker holding off-path work is not
preempted. Knowing the set buys the time; refusing off-path work buys the CPU.

\paragraph{Undeclared dependencies.}
\system{} is only as correct as the program's declared edges. In a pre-registered fault-injection replay, we
removed declared data edges at random from these 303 requests and from the 95 DataSciBench goals (10 seeds per
request and level; 7{,}960 cases), so that a step fails detectably when its now-hidden producer has not run.
Plain \system{} is more fragile than eager: with one edge removed it fails 34.0\% of WfCommons requests against
15.9\% for \code{stop\_eager} (74.1\% against 55.7\% with about 5\% of edges removed), and 40.9\% against
23.4\% of DataSciBench goals. Every lazy failure is a hidden producer outside the request's closure, or inside
it but reached too late. Eager's extra successes are mostly luck of visit order on DataSciBench (94\% would
fail under another order) and mostly structural on WfCommons (23\%). \system{} has built-in recovery: a step
that finds an undeclared input missing raises an error naming the producing step, and the engine adds that step
as a dependency, recomputes the demanded set, and retries the step within the same run, keeping what was paid
(at most three retries per step). It restores every case, at +0.3--3.6\% over clean \system{} on WfCommons and
about +10\% on DataSciBench. Against eager with the same retry, \system{} keeps a mean saving of 41.8\% and
41.3\% on WfCommons and 10.0\% per DataSciBench goal (13.7\% without faults), though it is the dearer of the two
on 13--33\% of cases. Silent wrong output, which no retry can catch, is outside this test.

\subsection{The equivalence region}
\label{sec:equiv}
\label{sec:boundary}

Proposition~\ref{prop:equiv} predicts exact ties on flat disjunctive pools, and all three software and
simulation arenas have that native shape.

\begin{center}
\small
\begin{tabular}{lrrr}
\toprule
\textbf{Arena} & \textbf{$n$} & \textbf{vs \code{stop\_eager}} & \textbf{vs \code{guard\_eager}}\\
\midrule
SWE-bench Verified & 500 tasks & 0.0\%, identical on all 500 & 78.8\%, cheaper on 442/500\\
SimulCost & 2{,}304 problems & 0.0\%, identical on all 2{,}304 & 73.0\% mean, cheaper on 2{,}208\\
Defects4J & 395 bugs & 0.0\%, identical on all 395 & 29.9\% pooled; 70.3\% on 185 repairable\\
\bottomrule
\end{tabular}
\end{center}

On SWE-bench Verified the two policies materialize the identical integer cost, 574{,}739, on every task;
quality is matched in all three arenas (445 solved, matched, and 185 solved respectively). Set equality
predicts the identical integer; statistical similarity would not. Had we reported a saving here it would
have been a defect: the only way to skip a pool candidate before attempting it is to consult an outcome the
policy is not entitled to know. The \code{guard\_eager} column --- 78.8\%, 73.0\%, 29.9\% --- shows what a
paper reporting only a naive-eager comparison would be measuring: the absence of a cheap predicate rather
than the presence of a new mechanism.

The coincidence has classical precedent. Pingali and Arvind showed that data-driven evaluation of a
suitably transformed program ``performs exactly the same computation as a demand-driven evaluation of the
original program''~\cite{pingali1985demand}, and Weiser's slice is ``a minimal form which still produces
that behavior''~\cite{weiser1984slicing}. Those results explain our equivalence region; \system{} adds the
agent-runtime semantics and experimentally maps the boundary where equivalence ends. A pool arena is a
program equal to its own slice, and the zeros follow.

\subsection{The shape law}
\label{sec:shape}
Setting the goal on the 86 multi-product instances to one product, half the products, or all products --- a
controlled intervention --- gives 235 observations:

\begin{center}
\small
\begin{tabular}{lrrr}
\toprule
\textbf{Goal scope} & \textbf{$n$} & \textbf{Mean off-path CPU} & \textbf{Mean saving vs \code{stop\_eager}}\\
\midrule
One product & 86 & 55.3\% & 50.3\%\\
Half the products & 63 & 24.9\% & 21.4\%\\
All products & 86 & 0.0\% & 0.0\% (86/86 exact zeros)\\
\bottomrule
\end{tabular}
\end{center}

The all-products row is the falsification test and it passes to the bit. Within one-product scopes the
Spearman correlation between off-path CPU fraction and saving is $+0.931$: instances under 25\% off-path
save 3.3\% ($n{=}28$), instances above 75\% save 85.1\% ($n{=}40$). Sweeping eight seeded ready-wave orders
isolates Proposition~\ref{prop:orderinv} and Corollary~\ref{cor:order}. The sweep is over a fixed slice ---
the first 60 one-product requests of the population above, every arm on the same seed --- and its bills are
counted in \emph{materialized tasks} rather than CPU seconds, because the question is whether the executed
\emph{set} moves, and a count answers it without weighting by task duration. \code{lazy}'s pooled count is
bit-identical across all eight seeds (4{,}082 tasks), \code{guard\_eager}'s is too (10{,}394), and
\code{stop\_eager} moves from 7{,}652 to 8{,}132 --- a 6.0\% spread that is scheduling luck. \code{lazy}'s
margin over \code{stop\_eager} clears that spread on every seed (46.7--49.8\%, median 48.6\%), so the
separation is not an ordering artifact.

\paragraph{The shape on every replayed population.}
The law also holds request by request outside the controlled study. For all 6{,}300 requests behind the
headline results, we set the share of \code{guard\_eager}'s bill that lies off the requested path (the
opportunity $O$ as a fraction of the full bill) beside the saving against \code{stop\_eager}. Wherever
\code{lazy} pays exactly its closure at \code{guard\_eager}'s prices, the saving cannot exceed that share,
and it does not: no point lies above it among the 303 WfCommons requests, the 95 DataSciBench goals, or the
12 Django gate requests, and all 287 open portfolio requests sit exactly on it, because an open product never
lets goal stopping fire. The only points above it are 1{,}367 settleable portfolio requests, where \code{lazy}
stops inside the requested pool for less than its closure, and 2 live gate requests whose arms metered the
same steps differently. All 3{,}199 pool requests sit at the origin. Within each population the saving rises
with the share: Spearman $+0.943$ on the 303 WfCommons requests, $+0.746$ and $+0.778$ on the two live gates,
$+0.281$ on DataSciBench, and $+0.270$ and $+0.306$ on the settleable SimulCost and Defects4J portfolios. We
report these per population, because pooling them mixes relationships of different shape.

\subsection{Agent-authored graphs}
\label{sec:datasci}

DataSciBench carries the mechanism into live, agent-authored graphs, and is unusually well suited to it: its
Data Interpreter emits an explicit dependency-structured plan, and the benchmark evaluates generated
data-science code under task-specific criteria~\cite{zhang2026datascibench}. That supplies what a replayed
workflow cannot --- an agent-authored graph and executable work rather than a final answer alone. In our
adaptation, an LLM plans once, real code executes each step, and individual deliverables become separate
goals; an LLM judges their quality.

The measured population is reached in two steps: of 40 candidate tasks, 3 are skipped before any arm runs
because their benchmark inputs are absent or unreadable, and of the 37 attempted, 9 yield no measurement
because the planner's own generated program fails at the task level. That leaves \textbf{28 tasks and 95
goals}.

\begin{center}
\small
\begin{tabular}{p{0.46\linewidth}p{0.48\linewidth}}
\toprule
\textbf{Quantity} & \textbf{Value}\\
\midrule
Wins / losses / ties vs \code{stop\_eager} & 33 / 0 / 62\\
Task-clustered mean saving over all 95 goals & 12.8\% (95\% CI [7.7\%, 18.2\%]),
15/0/13, $p<0.001$, rb $+1.00$\\
Mean over 33 goals with nonzero observed savings & 39.4\%\\
Best single goal & 86.7\%\\
vs \code{guard\_eager} & 45.1\% pooled, 41.6\% mean per task\\
Quality matched & every task\\
\bottomrule
\end{tabular}
\end{center}

The two statistics answer two questions. Among the \textbf{33 of 95 goals with nonzero observed savings},
permission saves a mean of \textbf{39.4\%} of real tokens, and up to 86.7\% on one goal. This is the
positive-difference subset, not every goal whose plan contains off-path work. Averaged over all goals, the
task-clustered mean is 12.8\%. The first figure describes the magnitude of observed wins; the second
describes the population average. Both are reported, and Section~\ref{sec:discussion} returns to why public
corpora sit close to the fully-demanded corner. The tie structure explains the difference. Of the 62 ties,
\textbf{13 are forced}:
the goal's closure covered the whole plan. The other \textbf{49 are order-contingent}: 133 planned nodes
were left unrun by both arms under the recorded visit order, so \code{stop\_eager} was never offered work it
had no principled reason to refuse. Those 62 are not evidence the mechanism is unnecessary; they are
evidence that a predicate-based executor is one renaming away from paying.

\subsection{Planning incentives and graph breadth}
\label{sec:planning-incentive}

Here off-request work being free means \emph{free to describe}, not free to execute. Under demand billing, a
branch outside today's closure stays in the program but remains abstract and incurs no current bill; every
node inside the closure still runs and is paid. This creates an economic reason to favor recall --- include
every plausibly useful branch --- over precision --- include only today's branch. Under eager execution the
same surplus runs and becomes expensive.

We tested whether stating that incentive was enough. In a pre-registered, planning-only intervention on the
same 28 tasks, both prompts asked \code{gpt-5.6-terra} for a reusable program; one billed every ready step
and the other only the requested closure. Three paired repeats produced \textbf{168 of 168} valid plans, but
no broadening: readiness billing put 4 of 482 nodes off path (0.8\%) and demand billing 0 of 483 (0.0\%),
with mean plan size unchanged at 5.74 against 5.75. The paired difference in each plan's own off-path
fraction is $-0.9$ percentage points (95\% CI [$-2.4$, $0.0$], $p=0.501$).

The null concerns the planner's response, not the execution economics: Section~\ref{sec:width} measures the
price of one additional product as $+22.5\%$ under \code{stop\_eager} and $+0.0\%$ under \code{lazy}.
However, these single-request tasks anchor every named deliverable, so extra breadth can appear only as
unnamed internal work. A decisive test needs benchmarks with broad standing programs and sequences of
narrower asks, where optional branches are concrete future deliverables --- the regime formed when authored
agent graphs are reused across scheduled or event-driven requests
(Sections~\ref{sec:intro} and~\ref{sec:future}).

\subsection{Both regions on one program}
\label{sec:portfolio}

The composed portfolios (Section~\ref{sec:arenas}) put a disjunctive pool inside a conjunctive program, so
both regions appear under one cost column. At $N=3$ there are 893 programs (764 SimulCost, 129 Defects4J)
and 3{,}572 rows. Controls first: \textbf{893 of 893} all-products rows tie exactly, there are \textbf{0}
quality mismatches, and \textbf{0} losses anywhere. Products within a program are ordered, and we sweep
which one the goal requests. The table below is the \emph{settleable} stratum only --- products some
candidate can close, which is the stratum where \code{stop\_eager} has a stopping event available and the
position effect is therefore visible at all. The open stratum, where no candidate settles, is reported
separately below and shows no position dependence, because a goal that never closes gives goal stopping
nothing to do at any position.

\begin{center}
\small
\begin{tabular}{llrrrr}
\toprule
\textbf{Arena} & \textbf{Requested position} & \textbf{$n$} & \textbf{Mean} & \textbf{Median} & \textbf{Exact ties}\\
\midrule
Defects4J & 0 (first) & 65 & 0.0\% & 0.0\% & 65 / 65\\
Defects4J & 1 & 57 & 74.3\% & 78.0\% & 0\\
Defects4J & 2 (last) & 60 & 85.5\% & 87.3\% & 0\\
SimulCost & 0 (first) & 743 & 7.5\% & 0.0\% & 674 / 743\\
SimulCost & 1 & 735 & 67.3\% & 79.8\% & 68\\
SimulCost & 2 (last) & 732 & 83.5\% & 91.4\% & 2\\
\bottomrule
\end{tabular}
\end{center}

Program, policy, cost column, and graph are fixed; the \emph{only} thing that changes down the table is
which deliverable was asked for (Figure~\ref{fig:position}). Asking for the product that sorts first is the
equivalence region --- an exact tie on 65 of 65 Defects4J programs and 674 of 743 SimulCost programs,
because \code{stop\_eager} usually closes the goal before it is offered anything else. Asking for a
later product gives 74--86\%, because \code{stop\_eager} has already paid for the earlier ones. Over the
order-fair sweep the settleable stratum saves 52.5\% (SimulCost, $n{=}2{,}210$) and 51.5\% (Defects4J,
$n{=}182$). The open stratum --- products no candidate ever settles, declared in the freeze before the
numbers were seen --- saves 70.1\% (SimulCost, $n{=}82$) and 65.5\% (Defects4J, $n{=}205$), with no ties and
no losses at any position, because a goal that never closes gives goal stopping nothing to stop on wherever
the product sits. The two strata are reported apart and never blended. Restricting instead to the
lexicographically first product --- one row per program, the most \code{stop\_eager}-favorable ask in the
suite --- gives 12.8\%, and we give both that and the sweep mean, since quoting either alone is the same
error in opposite directions. Varying the derived shared-setup share across $0\times$, $1\times$, and
$2\times$ the program median moves the SimulCost settleable mean only from 53.2\% to 51.9\%.

\begin{figure}[t]
\centering
\begin{tikzpicture}
\begin{axis}[
  ybar, width=0.66\linewidth, height=5.3cm,
  bar width=13pt,
  xmin=-0.35, xmax=2.35,
  xtick={0,1,2},
  xticklabels={first,middle,last},
  ymin=-7, ymax=112, ytick={0,25,50,75,100},
  ylabel={mean saving vs \code{stop\_eager} (\%)},
  xlabel={which of the three products the goal requested},
  nodes near coords={\pgfmathprintnumber[fixed,precision=1]{\pgfplotspointmeta}},
  nodes near coords style={font=\scriptsize},
  tick label style={font=\scriptsize}, label style={font=\scriptsize},
  legend style={font=\scriptsize, draw=none, fill=none,
                at={(0.02,0.99)}, anchor=north west},
  legend cell align=left,
  axis lines=left,
]
\addplot[fill=black!15, draw=black, point meta=explicit,
         nodes near coords style={anchor=south east, font=\scriptsize}]
  coordinates {(0,1.0) [0.0] (1,74.3) [74.3] (2,85.5) [85.5]};
\addlegendentry{Defects4J}
\addplot[fill=black!50, draw=black, point meta=explicit,
         nodes near coords style={anchor=south west, font=\scriptsize}]
  coordinates {(0,7.5) [7.5] (1,67.3) [67.3] (2,83.5) [83.5]};
\addlegendentry{SimulCost}
\end{axis}
\end{tikzpicture}
\caption{Same program, same policies, same cost column --- only the ask
changes. Requesting the product that sorts first is the equivalence region; requesting a later product is
the separation region. A Defects4J product is a repair target; a SimulCost product is a design study. The
zero-valued first Defects4J bar is drawn as a 1-point stub solely for visibility; its label and reported
value remain 0.0\%.}
\label{fig:position}
\end{figure}

\subsection{Program width: opportunity versus capture}
\label{sec:width}

Program width is the dial that separates what an experiment manufactures from what a policy achieves. Asking
for one product of three places roughly two thirds of the program outside the goal, and one of six places
five sixths, so a raw percentage must rise with $N$ even where the policy has not improved. Rerunning the
construction at $N=3,4,5,6$ (2{,}679, 2{,}668, 2{,}675, 2{,}658 one-product rows; all \textbf{2{,}538}
all-products controls tie exactly) and applying the event-level accounting of Section~\ref{sec:theory}
separates four facts (Figure~\ref{fig:width}):

\begin{enumerate}[topsep=2pt,itemsep=1pt]
  \item \textbf{Experiment-created opportunity} $O$ rises with $N$, tracking
  $(N-1)/N$ almost exactly --- 65.1\% at $N=3$ to 82.7\% at $N=6$ in the open Defects4J stratum, against
  66.7\% and 83.3\%. This is arithmetic we chose.
  \item \textbf{\system{} captures 100\%} of executable off-path cost in every
  arena, stratum, and width.
  \item \textbf{\code{stop\_eager} captures 0\%} under open goals and
  approximately \textbf{50\%} under settleable goals.
  \item \textbf{100\%} of every nonzero cost difference between the policies
  is off-path work; shared setup and requested-product costs do not move.
\end{enumerate}

\begin{center}
\small
\begin{tabular}{rrrrr}
\toprule
$N$ & \textbf{\code{lazy}, all strata} & \textbf{stopping, open} &
\textbf{stopping, D4J repairable} & \textbf{stopping, SC settleable}\\
\midrule
3 & \textbf{100.0\%} & 0.0\% & 52.4\% & 50.3\%\\
4 & \textbf{100.0\%} & 0.0\% & 46.5\% & 50.5\%\\
5 & \textbf{100.0\%} & 0.0\% & 47.9\% & 50.5\%\\
6 & \textbf{100.0\%} & 0.0\% & 49.7\% & 50.7\%\\
\bottomrule
\end{tabular}
\end{center}

\begin{figure}[t]
\centering
\begin{tikzpicture}
\begin{axis}[resultaxis,
  title={(a) Opportunity: set by the experiment},
  xlabel={products per program $N$},
  ylabel={off-path \% of guard-eager bill},
  xmin=2.75, xmax=6.25, ymin=58, ymax=90,
  xtick={3,4,5,6}, ytick={60,70,80,90},
  legend pos=north west, legend cell align=left,
]
\addplot[black!50, dashed, thick, mark=none]
  coordinates {(3,66.7)(4,75.0)(5,80.0)(6,83.3)};
\addlegendentry{$(N{-}1)/N$ reference}
\addplot[black, thick, mark=*, mark size=1.7pt]
  coordinates {(3,65.1)(4,74.1)(5,79.4)(6,82.7)};
\addlegendentry{Defects4J, open}
\addplot[black, thick, mark=square*, mark size=1.7pt, densely dotted]
  coordinates {(3,66.0)(4,74.4)(5,79.6)(6,83.0)};
\addlegendentry{SimulCost, settleable}
\end{axis}
\end{tikzpicture}\hfill
\begin{tikzpicture}
\begin{axis}[resultaxis,
  title={(b) Capture: set by the policy},
  xlabel={products per program $N$},
  ylabel={\% of opportunity avoided},
  xmin=2.75, xmax=6.25, ymin=-8, ymax=152,
  xtick={3,4,5,6}, ytick={0,25,50,75,100},
  legend style={font=\scriptsize, draw=none, fill=none,
                at={(0.5,1.0)}, anchor=north, column sep=6pt},
  legend columns=2, legend cell align=left,
]
\addplot[black, very thick, mark=*, mark size=1.7pt]
  coordinates {(3,100)(4,100)(5,100)(6,100)};
\addlegendentry{\code{lazy}, every stratum}
\addplot[black!60, thick, mark=square*, mark size=1.7pt, densely dotted]
  coordinates {(3,50.3)(4,50.5)(5,50.5)(6,50.7)};
\addlegendentry{\code{stop\_eager}, settleable}
\addplot[black!60, thick, mark=triangle*, mark size=2.2pt, dashed]
  coordinates {(3,52.4)(4,46.5)(5,47.9)(6,49.7)};
\addlegendentry{\code{stop\_eager}, repairable}
\addplot[black, thick, mark=o, mark size=1.7pt]
  coordinates {(3,0)(4,0)(5,0)(6,0)};
\addlegendentry{\code{stop\_eager}, open}
\end{axis}
\end{tikzpicture}
\caption{Separating what the experiment manufactures from what the policy
achieves. Panel (a), the share of the guard-eager bill outside the demanded product, tracks $(N{-}1)/N$.
Panel (b), the fraction of that opportunity each policy avoids, is flat in $N$: \code{lazy} captures all of
it everywhere; goal stopping captures none when the demanded product cannot settle and about half when it
can.}
\label{fig:width}
\end{figure}

The 100\%/0\% split under an open goal follows directly from the two rules: when no candidate can settle the
demanded product the goal never closes, so \code{stop\_eager} has nothing to stop on. The near-50\% split
under settleable goals is a robust pattern under this exhaustive order-fair construction, not a constant; a
deployment whose requests name early products will favor \code{stop\_eager}, and one naming late products
will favor \code{lazy}. Program-clustered inference gives no losses and $p<0.001$ in every stratum.

\paragraph{The sweep in absolute units.}
Both panels are ratios. Re-reading the same frozen rows in the unit they were measured in answers the
question a planner faces: if one more product is described, who pays? On Defects4J measured wall seconds,
program-clustered:

\begin{center}
\small
\begin{tabular}{rrrrr}
\toprule
$N$ & programs & \code{guard\_eager} & \code{stop\_eager} & \code{lazy}\\
\midrule
3 & 129 & 41{,}002 & 30{,}991 & 10{,}119\\
4 &  96 & 54{,}525 & 42{,}175 & 10{,}125\\
5 &  77 & 67{,}917 & 52{,}901 & 10{,}128\\
6 &  63 & 81{,}173 & 62{,}786 & 10{,}121\\
\midrule
\multicolumn{2}{l}{slope per product} & $+13{,}391$ & $+10{,}611$ & $\mathbf{+0.98}$\\
\multicolumn{2}{l}{as \% of own mean} & $+21.9\%$ & $+22.5\%$ & $\mathbf{+0.0\%}$\\
\bottomrule
\end{tabular}
\end{center}

\code{lazy} varies by a factor of 1.0009 across four widths, because it pays shared setup plus the one
requested product --- a quantity with no $N$ in it. Every eager rule carries $N$ as a factor, and
\textbf{the strongest eager baseline is linear too}: \code{stop\_eager} grows at $+22.5\%$ per described
product, essentially \code{guard\_eager}'s slope. Goal stopping lowers the intercept, not the exponent,
because a stopping predicate refers to what has been computed and never to what was requested. SimulCost is
concordant in direction (eager $+21.6\%$, \code{lazy} $+0.2\%$ per product) and much noisier; we quote
Defects4J as the witness. Under demand-driven semantics, program breadth stops being a budget decision.

\subsection{A live money witness across four repositories}
\label{sec:live}

The live release gate (Section~\ref{sec:arenas}) puts real spending on the clock: each candidate patch runs
in a container and we meter its execution seconds. Two of the three issues are off-path, so the question is
concrete --- does a policy pay for real containers running real test suites for issues nobody asked about?

Both strata appear here, on the same cost column. A requested issue is \emph{settleable} on a row if some
candidate patch passes the project's test suite during our measurement, and \emph{open} if none does. The
distinction is not cosmetic: goal stopping can decline work only after the goal closes, so on an open row
\code{stop\_eager} coincides with \code{guard\_eager} node for node, while on a settleable row its predicate
fires and it declines whatever the visit order had not yet reached. Two runs are reported --- four gates
across four repositories, and four further gates on Django alone with a different patcher --- and the
strata are never pooled.

\begin{center}
\small
\begin{tabular}{lrrr}
\toprule
\textbf{Repository} & \textbf{$n$} & \textbf{Mean saving} & \textbf{Range}\\
\midrule
\multicolumn{4}{l}{\emph{Four repositories, one gate each}}\\
astropy & 3 & 49.4\% & 0.0--98.7\%\\
django & 3 & 65.8\% & 15.7--93.2\%\\
matplotlib & 3 & 66.5\% & 9.6--95.2\%\\
requests & 3 & 24.9\% & 0.0--57.7\%\\
All four & 12 & \textbf{51.7\%} & 0.0--98.7\%\\
\midrule
\multicolumn{4}{l}{\emph{Django, four further gates}}\\
django & 12 & \textbf{52.1\%} & 0.0--93.2\%\\
\bottomrule
\end{tabular}
\end{center}

Saving is against \code{stop\_eager}, the strongest eager arm, and the two runs give \textbf{10 wins, 2
ties, 0 losses} each. Every tie is a row where the requested issue settled early enough that goal stopping
declined the same work permission did; none is a row where \system{} paid more. Quality is matched on every
row, all eight all-products controls tie at exactly 0.0000\%, and \code{react} (Section~\ref{sec:method})
ties \code{stop\_eager} on all twenty-four rows, since it too avoids undemanded work only by reaching the
goal first.

The capture decomposition separates what the request made unnecessary from who actually avoided it, and it
is where the two strata diverge.

\begin{center}
\small
\begin{tabular}{llrrr}
\toprule
\textbf{Run} & \textbf{Stratum} & \textbf{$n$} & \textbf{\code{stop\_eager}} & \textbf{\system{}}\\
\midrule
Four repositories & settleable & 6 & 46.5\% & 100\%\\
                  & open       & 6 & \textbf{0.0\%} & 100\%\\
\midrule
Django            & settleable & 6 & 41.5\% & 100\%\\
                  & open       & 6 & \textbf{0.0\%} & 100\%\\
\bottomrule
\end{tabular}
\end{center}

\system{} captures every unit of off-path container time on all twenty-four rows. Goal stopping captures
a little under half of it where the goal closes, nothing at all where it does not, and the open-stratum
zero is an identity rather than a small effect --- the predicate cannot refer to what was requested, only to
what has been computed, so with the goal still open the strongest eager arm and the naive one are literally
the same arm. This is the live counterpart of the replayed strata in Sections~\ref{sec:portfolio}
and~\ref{sec:width}, and it reproduces their 40--50\% on metered container seconds, at 46.5\% and 41.5\%.

Three qualifications. The settleable means are averages over a split distribution rather than typical rows:
the six four-repository rows read $0, 0, 0, 79.2, 100, 100$ and the six Django rows read $0, 0, 0,
49.2, 100, 100$. Goal stopping recovers everything when the settle precedes the off-path work and nothing
when it follows, so the mean describes the arena's visit orders, not a rate any single gate should be
expected to show. This remains a case study rather than a population claim: four program clusters per run
bound the smallest exact Wilcoxon $p$ at 0.125, and the raw percentages carry the $N=3$ width dependence of
Section~\ref{sec:width}. Finally the two runs draw on disjoint issue pools, so they are two measurements of
the same structure rather than a repetition, and we report them apart and never pool them. Both were
materialized by the same patcher as every other live number in this paper.

\subsection{A second structural domain: native build targets}
\label{sec:bazel}

Bazel~\cite{bazel} is a build system that already computes the dependency closure of a requested target. We
use it as a second structural domain to check that the mechanism concerns graph shape rather than agents:
what we replay is the union graph of a whole package, asking the engine for one native target at a time so
the arena's off-path work is the rest of the package. Two production packages, \code{bazel-skylib} (284
targets) and \code{buildtools} (135 targets), are profiled once under a full-package build with measured
action CPU seconds, and every one of their 419 native targets is requested in turn as an
\textbf{AllOutputs} goal. The permission spine --- \code{guard\_eager}, \code{stop\_eager}, \code{lazy} ---
is run on each; unprofiled administrative parent actions are admitted at zero cost rather than priced by us.

\begin{center}
\small
\begin{tabular}{lrrr}
\toprule
\textbf{Slice} & \textbf{$n$} & \textbf{Mean saving vs \code{stop\_eager}} & \textbf{+ / = / $-$}\\
\midrule
Both packages & 419 & \textbf{96.0\%} [95.0, 96.8] & 418 / 1 / 0\\
\quad \code{bazel-skylib} & 284 & 97.9\% & \\
\quad \code{buildtools} & 135 & 91.9\% & \\
\midrule
\multicolumn{4}{@{}l}{\emph{By product kind}}\\
\quad test & 249 & 99.1\% & 249 / 0 / 0\\
\quad docs & 58 & 96.4\% & 58 / 0 / 0\\
\quad library or binary & 108 & 91.0\% & 108 / 0 / 0\\
\quad distribution package & 4 & 28.2\% & 3 / 1 / 0\\
\bottomrule
\end{tabular}
\end{center}

\paragraph{Kind skew.}
The saving runs from 99.1\% for a test target to 28.2\% for a distribution package, and the absolute bills
on \code{bazel-skylib} say why. A test target bills 0.18 CPU-s under \code{lazy} against 29.1 under
\code{stop\_eager} and 45.2 under \code{guard\_eager}: its closure is a few source files and one runner, so
nearly the whole union is off-path and stopping, which must build in visit order until the target lands,
pays for most of it. A docs target bills 0.41 against 15.8. A distribution package bills 0.78 against 1.19,
because its closure \emph{is} most of the package, so there is little outside it to refuse, and the one tie
in the table is a package whose closure covers the union. This is Section~\ref{sec:shape}'s law --- saving
tracks the share of the graph outside the requested closure --- read off product kinds rather than off a
controlled scope sweep, and it is what makes the arena informative: the same union graph yields four
different answers depending on what kind of thing was asked for.

\paragraph{Opportunity and capture.}
Capture is the quantity that transfers, and it carries over intact: \code{lazy} takes 100\% of the
executable opportunity in both packages, \code{stop\_eager} 42.3\% and 38.7\% --- the 40--50\% pattern of
the settleable stratum (Section~\ref{sec:width}) again, here because every build target does settle and
stopping fires somewhere in the alphabetical order. The 96.0\% headline is the larger number but the
smaller claim: a one-target ask against a whole-package union manufactures a very large opportunity
(99.5\% of full eager cost on \code{bazel-skylib}, 93.8\% on \code{buildtools}), so that percentage is
mostly a fact about request shape and should not be read against WfCommons's 42.0\%. Two production
packages are not the Bazel ecosystem, the targets within a package are clustered, and none of this is
agentic; the row is labeled exploratory throughout. What it establishes is that the two regions and the
capture split appear unchanged in a domain whose graphs, costs, and request language were none of them
authored by us.

\subsection{Contingent irreversible actions}
\label{sec:contingent}

The commit barrier required by invariant I2 (Section~\ref{sec:invariants}) is independent of permission: it
blocks an unauthorized irreversible act under every policy, but does not avoid the cost of preparing that
act. On an authored family of $K = 1 \ldots 4$ contingent commit paths, every irreversible node has its own
barrier, released only by evidence the run itself produces. Each path is an investigation costing 20 feeding
an irreversible order costing 5; the goal names the one response the incident invoked, and we sweep which
one that is.

\begin{center}
\small
\begin{tabular}{lrrrrr}
\toprule
& $K{=}1$ & $K{=}2$ & $K{=}3$ & $K{=}4$ & slope\\
\midrule
Irreversible acts (every arm) & 1 & 1 & 1 & 1 & ---\\
\code{guard\_eager} cost & 26 & 46 & 66 & 86 & $+20.0$\\
\code{stop\_eager} cost & 26 & 46 & 66 & 86 & $+20.0$\\
\code{lazy} cost & 26 & 26 & 26 & 26 & $\mathbf{0.00}$\\
\bottomrule
\end{tabular}
\end{center}

Exactly one irreversible action occurs on every arm at every $K$: safety belongs to the barrier and is never
credited to laziness. The preparation cost is the claim. Describing one more contingency costs 20 on every
eager arm --- \code{stop\_eager} does not even lower the intercept, billing 86 at $K=4$ like the naive rule,
because every investigation is offered before any order can close the goal --- and nothing under permission.
Two controls hold: with authorization refused after the investigations are paid, no arm acts and the eager
arms have bought all four cones for nothing (81 versus 21); with the barrier removed, every arm is refused
identically. Prices are authored, so only the two slopes transfer.

\subsection{When the run changes its own plan}
\label{sec:dynamic}

On a fixed plan the demanded set is exactly the slice a build system computes when a request arrives, and on
all 6{,}300 replayed requests above the set computed once at the start (the \emph{start-time set}) bills exactly
as \code{lazy}, as it should. Agent runs do not keep their plans fixed, and there the set must be recomputed to stay
exact. Two pre-registered studies, one live and one on real candidate outcomes, measure what that buys.

\paragraph{Live re-planning.}
We let the DataSciBench planner (GPT-5.6-terra) re-plan instead of repair: when a step's code fails for real,
the planner sees the plan, the failed step, and its error, and adds steps wherever it chooses. Each growth is
recorded once, live, and replayed to every arm that reaches the same failure, so all arms face the same new
steps with provider-measured tokens. Across 28 tasks and 95 goals, 19 of 175 steps failed, 10 tasks re-planned
15 times (5 times after a replacement also failed), and growth reached the path of 24 goals. \system{} solves
all 88 solvable goals (the planner produced no working replacement for the other 7) while running only the new
steps each request needs: \textbf{11.7\%} fewer tokens than \code{stop\_eager} per goal (task-clustered 95\% CI
[5.8, 18.9]; 32 wins, 63 ties, no losses), and \textbf{6.1\%} fewer in total than admitting every new step (6
wins, 89 ties, no losses), which on 6 goals also ran replacements for sibling deliverables nobody requested. A
set fixed at the start cannot follow the growth: it never runs the new steps, and 17 goals go unsolved. Every pre-registered prediction held on every goal: 17 of 17 rescues, 95 of 95 bills at or below
admitting every new step, and 71 of 71 exact ties on goals the planner never grew. The study, pilot included,
cost \$2.97 of model calls.

\paragraph{Retiring settled issues.}
A release gate asks for several issues at once, and about 20 published candidate patches compete for each.
Once a candidate fixes an issue, its siblings are no longer needed, and \system{} retires them
(\code{lazy\_retire}, reported separately from the base rule). We replay 162 SWE-bench Verified release gates
with real candidate outcomes. The strongest eager baseline here is \emph{per-pool stopping}, which stops each
issue at its first fix. Asked for two of three issues, retiring saves \textbf{22.7\%} against it (95\% CI [21.3,
24.1]; no losses), nearly independent of candidate order, because per-pool stopping still tests the unrequested
issue; replayed on the recorded container times of eight live gates, the saving is 29.5\% and 30.6\%. Against
\code{stop\_eager} it saves 81.2\% on those requests and \textbf{75.5\%} when all three issues are requested
(95\% CI [73.0, 77.9]; 161 wins, 1 tie, no losses), where per-pool stopping matches it. Against the start-time
set, which keeps paying for settled issues until all close, the saving is 72.3\% on two-issue requests and
41.7--75.5\% across nine candidate orders on three-issue ones, and the median three-issue request closes 87.5\%
sooner with one container slot and 40.1\% sooner with eight. Every pre-registered identity held in all
128{,}520 engine runs, and every arm solves exactly the same requests. With more than one slot, retiring can
cost a few percent more than the base rule on 4--8\% of cells, because freed slots start other candidates
earlier.

\section{Results: physical operations}
\label{sec:physops}

Section~\ref{sec:results} evaluated permission --- the constitutive component of the policy tuple in
Definition~\ref{def:policy} --- which decides whether a ready node may run. This section holds permission
fixed and evaluates order and projection, the tuple's other two components, alongside prune, reuse,
fidelity, fusion, and speculation; Table~\ref{tab:opsglance} summarizes their matched-control results, while
commit was measured with permission in Section~\ref{sec:contingent}.

Projection and fidelity lead because their positive remainders add to permission by narrowing how much of a
demanded node becomes concrete. Reordering follows as the strongest substitute for permission and as an
exact boundary on when scheduling can change a bill. Prune, reuse, fusion and speculation are developed in
Appendix~\ref{app:ops}: pruning is a demand-independent terminal action, reuse and fusion mostly overlap with
work permission already removed, and speculation moves in the opposite direction by buying unasked work
to reduce latency. Section~\ref{sec:physops-summary} reads all seven together.

\begin{table}[tp]
\centering
\caption{Physical operations at a glance. Remainder is demand-specific: lever
worth to \code{lazy} minus lever worth to the matched eager arm. A negative remainder means the eager
control captured more of the lever's absolute effect than \code{lazy} did, not that the lever failed. The
final block is not a remainder: speculation's rows report a cost/latency pair against the no-speculation
lazy arm, where positive latency is faster and negative cost is more expensive. Speculation's latency is
measured in \emph{rounds containing real work}, never in seconds; Appendix~\ref{sec:speculate} says why a
clock cannot be used here.}
\label{tab:opsglance}
\scriptsize
\setlength{\tabcolsep}{3.5pt}
\begin{tabular}{L{0.10\linewidth}L{0.21\linewidth}L{0.085\linewidth}L{0.26\linewidth}L{0.23\linewidth}}
\toprule
\textbf{Operation} & \textbf{Arena} & \textbf{Costs} & \textbf{Result} & \textbf{Read it as}\\
\midrule
Project & Live DataSciBench, 95 goals & live & \textbf{$+3.2\%$}; narrow asks bill $-14.3\%$; no significant blind-judge preference ($p=0.29$) & Fewer tokens; no detected preference\\
Project & Authored isolation pair & authored & \textbf{$+84.3\%$} (named consumer) vs \textbf{$+0.0\%$} (no consumer) & Pays iff the dropped field has another consumer\\
Project & CI verify job, 2 third-party repos & \textbf{measured cpu} & \textbf{$+89.1\%$}, \textbf{$+80.1\%$} (named, off-path); \textbf{$+0.0\%$} with no consumer and when both demanded & Same shape, prices from a clock; controls zero independently\\
Project & Authored sessions, coverage sweep & authored & $+34.3\%$, $+35.6\%$; $+35.4\% \to +72.9\%$ & Grows as session coverage falls\\
\addlinespace
Fidelity & Authored ladder, named consumer / no consumer / needed & authored & $+96.1\%$ / $+0.0\%$ / $+0.0\%$ & Projection one granularity down; sign pattern, not magnitude\\
\addlinespace
Reorder & No rule may stop the run; 5 arenas, every task & ex-ante, replay, live & \textbf{0.0\%} bill change, incl.\ 8 random orders & Exact identity under C1--C3\\
Reorder & Authored screening, 3 programs & true ex-ante & \textbf{$+0.0\%$} (both arms take 57\%) & Disjunctive pool: permission \emph{is} the stopping rule\\
Reorder & Live DataSciBench, 95 goals & live & \textbf{$-4.2\%$} node, $-5.6\%$ branch; recovers 21.4\% of the gap & A fifth of what permission refuses\\
Reorder & Authored lifecycle, inversion pair & authored & \textbf{$-180.7\%$}, $-231.9\%$ to \code{stop\_eager}; cheapest-first demand $\equiv$ \code{lazy} & High-variance substitute; permission immune\\
Reorder & Per task, 178 conjunctive tasks, 3 arenas & replay, live & Eager inverted on 11.7\% (CPU) / \textbf{15.8\%} (tokens), worst $-70.3\%$; \textbf{lazy inverted 0} & Inversion is real at real prices; lazy never\\
Reorder & Live materializer, own calls per leg & live & Executed set identical 8/8; bill $0.6\%$ across orderings vs $2.0\%$ within one & Identity survives nondeterminism\\
Reorder & Unpriced live plan & live & Every selector picks one schedule, 8/8 & Cost-ranked order is \emph{undefined} on a first pass\\
\addlinespace
Prune & Authored screening, 3 programs & authored & \textbf{$+0.0\%$}; $28\to203$ when needed & New terminal; not a disguised stop\\
Prune & Authored lead time, deadline 5 days & authored cpu + wall & Remainder \textbf{$+0.0\%$}, cost $-625\%$, on time \textbf{1/1} vs \textbf{0/1} & Large absolute worth, zero remainder --- both true\\
\addlinespace
Reuse & Live-priced DataSciBench follow-ups, 28 tasks & archived live & \textbf{$-7.6\%$}; flat vs session length & Corpus sits at coverage $\approx 1.0$\\
Reuse & Authored coverage sweep & authored & \textbf{$-73.8\% \to -15.3\%$}, no zero crossing; arms bit-identical at 25\% & Partial substitute for permission\\
Reuse & Swept to the floor: 3 widths $\times$ 3 lengths & authored & Never positive; \textbf{exactly $+0.0\%$} at one touched product & Touched count, not coverage, is the axis\\
\addlinespace
Fusion & Authored pair, both / off-path & authored & $+0.0\%$ / $\mathbf{-32.0\%}$ & Work permission declined is duplicate work fusion cannot merge\\
\addlinespace
Speculate & Dial corners, both directions & authored & Budget $0$ $\equiv$ \code{lazy} on cost \emph{and} materialized-node set; budget $\infty$ reaches \code{eager}'s set & Eager is a setting of speculation\\
Speculate & Standing sessions, coverage $\ge 50\%$, nothing invalidated & authored & Cost \textbf{exactly $+0.0\%$}, latency \textbf{$+11$ to $+22\%$} of rounds & Free latency\\
Speculate & Same sweep, coverage $<50\%$ or any invalidation & authored & Cost $-8$ to $-25\%$, latency $+8$ to $+20\%$ & A priced trade\\
Speculate & Same sweep, coverage $12.5\%$ & authored & Cost $-30.1\%$, latency \textbf{$+0.0\%$} & Strictly dominated\\
Speculate & Single-ask family, 3 programs & authored & Hit rate \textbf{0\%} by construction & Every single-ask benchmark reports speculation as dominated\\
\bottomrule
\end{tabular}
\end{table}

\subsection{Projection}
\label{sec:project}

\paragraph{Live LLM steps.}
Across \textbf{126 of the 590 measured DataSciBench nodes} --- every one whose consumers wanted less than it
produced --- whole asks consumed 290{,}471 tokens and narrow asks \textbf{248{,}968}, a \textbf{14.3\%}
saving. The two figures answer two questions, as with permission in
Section~\ref{sec:datasci}: 14.3\% is what projection is worth where the precondition holds, and the
\textbf{$+3.2\%$} remainder is that saving diluted across a corpus where it holds on roughly a fifth of
nodes. Projection is worth 2.2\% to \code{stop\_eager} and 5.5\% to \code{lazy}, and the reason the
remainder is positive is structural: what a node may
omit depends on who is asking, so the narrowing an eager runtime can justify from the program text alone is
strictly smaller than the narrowing a goal licenses.

That saving is admissible only if the narrowed field is as good as the same field asked alongside its
siblings, which no token count can check. Judged blind over all \textbf{151} field pairs --- provenance
hidden, and which answer appears first assigned by a hash of the pair rather than by which one is cheap ---
the verdicts are 94 equivalent, 33 favoring the \emph{narrow} answer and 24 the whole one; an exact
two-sided sign test on the 57 decisive pairs gives \textbf{$p=0.29$}. The blind judge therefore detects no
significant preference between the arms; this is not an equivalence or non-inferiority claim. We report the
14.3\% token reduction alongside that result. The unblinded figure is not measuring the answers: an earlier
instrument labeled the candidates and always showed the projected one first, and
re-scored blind, identical content is judged at least as good 89.5\% of the time in the first slot against
78.7\% in the second --- an eleven-point slot effect inside a single model. Any LLM-judged quality column
that fixes the order or names the arm is reporting some of that.

\paragraph{The precondition is one edge.}
The profile program's shared node emits a cheap pass/fail summary and an expensive execution trace, and the
program text names a coverage report reading the trace. The static eager projection control reads the
program without today's goal and must keep the trace because a real consumer edge names it; \code{lazy}'s
goal-scoped demand need not, since the coverage report is outside today's closure. Authored programs isolate
this to the single edge
(Appendix~\ref{app:ops}, Table~\ref{tab:projauthored}): programs differing only in whether the undemanded
half has a downstream consumer give $+84.3\%$ when it does and exactly $+0.0\%$ when it does not, and the
remainder grows from $+35.4\%$ to $+72.9\%$ as a standing session's coverage falls from 100\% to 25\%.

\emph{A narrower ask pays only when the undemanded field has another consumer in the program text.} Then no
static pass may drop it, an eager arm keeps paying for it every ticket, and only a live goal-scoped demanded
set declines it. With no consumer it is dead output, an ordinary dead-field pass removes it, and the eager
arm takes the identical saving. Both conditions must hold --- the materializer must physically do less, and
the dropped field must be named for somebody else. DataSciBench supplies a smaller, planner-bound witness:
narrow asks used 14.3\% fewer tokens and produced a $+3.2\%$ remainder on one planner. The controlled
Jinja/Flask pair isolates both conditions directly. The structure is common wherever one expensive step
serves two audiences on different schedules: a test run emitting a pass/fail a gate turns on and a coverage
report a weekly dashboard reads; a workflow stage emitting a summary for one product and a full intermediate
dataset for another.

\paragraph{Measured prices.}
The magnitudes above are arithmetic on constants we chose, so we built the same shape over code we did not
write and let a clock set the prices. We instantiate a continuous-integration verify job in two pinned
third-party repositories: Jinja, Pallets' Python template engine, and Flask, its web application
framework~\cite{pallets2026jinja,pallets2026flask}. They form a controlled pair: the same
organization, language, and pytest toolchain, but distinct codebases and suite sizes (911 and 494 passing
tests). Both have modular tests, so the same job can target one affected module or the full suite --- the
physical decomposition projection requires.

The job emits three deliverables: a \emph{verdict} for the module under change, a \emph{coverage} report,
and a \emph{regression} result for the whole suite. Its rule is mechanical: whole suite if regression was
asked for and one module otherwise, instrumentation on only if coverage was asked for. Nothing is computed
and then discarded. The four distinct command lines were each run seven times and the median kept, before
any policy object existed, and the seven non-empty field subsets collapse onto those four, so every ask is
billed from its own measurement rather than interpolated. The fields interact --- coverage over the whole
suite costs more than over one module --- so the node publishes no per-field decomposition, which would
overstate what dropping two of them saves.

\begin{center}
\small
\begin{tabular}{llrrr}
\toprule
\textbf{Repository} & \textbf{Cell} & \textbf{Worth to \code{stop\_eager}} & \textbf{Worth to \code{lazy}} & \textbf{Remainder}\\
\midrule
Jinja & named consumer, off-path & $0.0\%$ & $+89.1\%$ & \textbf{$+89.1\%$}\\
Jinja & no-consumer control & $+89.1\%$ & $+89.1\%$ & \textbf{$+0.0\%$}\\
Jinja & both demanded & $0.0\%$ & $+0.0\%$ & \textbf{$+0.0\%$}\\
\addlinespace
Flask & named consumer, off-path & $0.0\%$ & $+80.1\%$ & \textbf{$+80.1\%$}\\
Flask & no-consumer control & $+80.1\%$ & $+80.1\%$ & \textbf{$+0.0\%$}\\
Flask & both demanded & $0.0\%$ & $+0.0\%$ & \textbf{$+0.0\%$}\\
\bottomrule
\end{tabular}
\end{center}

Three cells rather than two, because each of the claim's two preconditions must be switchable on its own:
the no-consumer control removes the second consumer, while the all-demanded control keeps that consumer and
puts it inside today's request. Both remainders are exactly zero, and the no-consumer control reaches zero
while projection still saves $+89.1\%$ for \emph{both} arms. That zero means the static pass took the whole
saving, not that narrowing did nothing.

Two qualifications belong with the number. The zeros are structural rather than empirical: in the
no-consumer control both arms resolve to the same demanded field set and bill the same measured entry, so
the remainder is zero by construction, and the control's force is that it excludes an accounting leak
rather than that it surveyed a population. And the magnitude is a property of the repository, the ratio of
one test module to a whole suite plus coverage overhead, moving between $76\%$ and $89\%$ across the two
repositories and across repetition counts --- so the sign and the control pattern are the result, and the
magnitude is not a constant.

Finally, the boundary this control draws. The static eager projection control reads the program text without
today's goal, and a stronger static pass specialized to a known root would trace backward from the gate
check, find the
coverage and regression deliverables unreachable, and take the same saving --- against \emph{that} baseline
the remainder collapses to zero, and we do not claim otherwise. The separation is between a program compiled
once and run under a sequence of asks and one recompiled per ask, which is the distinction permission itself
draws: the demanded set is a function of the goal, and a static pass matches it only by being handed the
goal and re-run. Where recompiling per ticket is free, projection's remainder belongs to whoever ships the
compiler. Where the program is fixed and the ask varies --- a standing CI pipeline, a release gate, a
service --- it does not.
\subsection{Fidelity}
\label{sec:fidelity}

Fidelity is projection at precision granularity: both compare a static reading as the eager control with a
live, goal-relative reading as the lazy treatment. We test it in an authored controlled scenario, not a
SimulCost replay. A standing workflow serves two consumers: today's gate needs a coarse yes/no answer, while
a later published metric needs the precise result. Its materializer ladder is a cached estimate (cost 2,
accuracy 0.25), a sampled run (30, 0.80), and a full run (200, 0.98), plus shared checkout and output costs
of 1 and 5. The static eager control sees the precise consumer in the program and pays 206; goal-relative
fidelity sees only today's gate and pays 8. Thus the predicted remainder is
$96.1\%=(206-8)/206=(200-2)/206$. The table switches off each precondition in turn.

\begin{center}
\small
\begin{tabular}{L{0.34\linewidth}L{0.14\linewidth}rrc}
\toprule
program & today's goal & worth to baseline & remainder & sign \\
\midrule
Precise consumer outside today's goal & gate check &  $0.0\%$ & $+96.1\%$ & $\mathbf{+}$ \\
No precise consumer                    & gate check & $96.1\%$ &  $+0.0\%$ & $\mathbf{0}$ \\
Precise result requested today         & report     &  $0.0\%$ &  $+0.0\%$ & $\mathbf{0}$ \\
\bottomrule
\end{tabular}
\end{center}

The magnitude is arithmetic on constants we chose; the claim is the sign pattern, which reproduces
projection's consumer-present and no-consumer controls. Only when the precise consumer exists and the goal
does not reach it do the arms separate, and permission alone moves nothing there. We cannot supply a corpus
row, and the reason is a finding: SimulCost, the one public multi-fidelity benchmark, records error as a
residual against an \emph{executed} refined run, so learning that the cheap answer sufficed requires
computing the expensive one --- verification is 66.0\% of spend over the 31{,}092 attempts that bill it
separately, and the benchmark excludes it from the cost its agents are scored on.

Appendix~\ref{app:ops} completes the factorial for the four operations held back from this section,
including the two whose worth is real but arrives outside the resource a remainder is computed in: pruning
under a deadline, where the deciding resource is calendar time rather than spend, and speculation's dial,
whose corners are verified as identities against \code{lazy} and \code{eager}.

\subsection{Reordering}
\label{sec:reorder}

Reordering has an exact structural boundary, not merely a negative average. Under three conditions and with
no rule able to terminate the run early, every ready, guard-valid node eventually runs; a selector changes
\emph{when} work runs, never \emph{whether}, so the executed set is invariant; with fixed additive node
charges, the bill is too. The conditions are \emph{(C1) deterministic materializer payloads},
\emph{(C2) payload-only guards} that do not read the clock or execution history, and \emph{(C3) a fixed
graph during the comparison}. Under C1--C3, with demanded guards that hold, the executed set is an
order-independent fixpoint. Five priced
arenas verify the bill identity at \textbf{0.0\% on every task}, under cheapest-node, cheapest-branch, and
eight random orders; every replay arena satisfies C1--C3 exactly because its payloads and graph are fixed.

However, outside that boundary reordering can be valuable whenever visit order changes the termination
event: which candidate first satisfies a disjunctive goal, or whether a stopping rule reaches the goal
before off-path work. In the disjunctive screening family it saves 57\% for both permission arms; with
predicate stopping it recovers part of permission's gap. For conjunctive goals under \code{lazy}, its value
remains exactly zero for the opposite reason: permission removes off-path work under every order, leaving
no executed-set difference for ordering to create.

\paragraph{How much a price heuristic recovers.}
Define overlap as \emph{recovered}/\emph{gap}: of the difference between the stopping rule and demand, the
share a scheduler picks up with no notion of what was asked for. Across the four priced populations that
have a gap to recover --- SimulCost's is zero, so no ratio is defined there --- it runs
from 21.4\% on DataSciBench in real tokens to 70.7\% on an authored lifecycle program
(Appendix~\ref{app:ops}, Table~\ref{tab:overlap}), authored programs flattering ordering because their
prices correlate with off-path-ness by construction. On the live population cheapest-node ordering was worth
4.2\% of \code{stop\_eager}'s pooled bill and 0.0\% of \code{lazy}'s, cheapest-branch 5.6\% and 0.0\%, so
permission does roughly five times what a price heuristic reaches; Section~\ref{sec:planner} reports what
that ratio does under a change of planner.

That fifth is concentrated rather than typical: cheapest-node ordering changes nothing on 63 of the 95
goals, helps on 17, backfires on 15, and three goals supply 93\% of every token it saves. It is not a rate a
scheduler earns broadly but a few large wins net of comparable small losses, divided by a gap permission
closes on every goal at once.

\paragraph{Price is not relevance.}
Hand both arms the same reasonable heuristic, cheapest ready node first. On authored lifecycle programs the
eager arm pays roughly three times what arbitrary order cost it ($-180.7\%$ and $-231.9\%$) while
the cheapest-first lazy arm is bit-identical to \code{lazy}: an eager arm's only brake is goal
satisfaction, which arrives from the on-path node, and cheapest-first walks it through every cheap off-path
branch while deferring the expensive on-path node that would have closed the goal. Permission asks whether
the node is required for what was asked, and is structurally immune.

The inversion is not an authored artifact. It appears on 7 of 60 WfCommons tasks in recorded CPU (worst
$-47.6\%$) and 15 of 95 DataSciBench goals in real tokens (worst $-70.3\%$, $6{,}768 \to 11{,}525$), against
12 of 23 authored lifecycle tasks (Appendix~\ref{app:ops}, Table~\ref{tab:invert}). Across 178 conjunctive
tasks in three arenas and two real cost units, the lazy bill moved \textbf{zero times}; any such inversion
would violate the order-invariance prediction.

\paragraph{The live first pass.}
With each arm making its own model calls and a repeat of each ordering as control, no selector changed the
executed set on 8/8 tasks --- but no task produced competing schedules, because a freshly planned live graph
carries no prices: a cost-ranked selector sees one price everywhere, ties, and breaks ties by identifier,
which is what arbitrary order already does. Permission needs only the goal, known before anything runs. On a
priced authored shape materialized by real model calls the selectors did diverge, and the bill moved
less across orderings ($0.6\%$) than between two runs of the same ordering ($2.0\%$) --- the identity exact
for the work performed and, for tokens, inside the noise floor of asking one model the same question twice.

\subsection{Which rows survive a change of planner}
\label{sec:planner}

Every live number above was produced by a language model writing the program being measured. That raises a
question the rest of the paper cannot answer from inside a single run: which results are properties of the
mechanism, and which are properties of the model. We ran the live agentic arena under three planners ---
\code{gemini-3.1-pro-preview}, \code{gpt-5.6-terra} and \code{claude-sonnet-5} --- with the benchmark, the
harness, the arms and the temperature held fixed.

Each planner completes tasks the others cannot, so the three runs measure overlapping rather than identical
populations. A DataSciBench task yields a measurement only if the planner's own generated code runs, and the
three fail on \emph{disjoint} sets: the only tasks failing under all three are one missing a library from
the image and one no planner handled. We therefore compare on the 77 goals all three measured, since
comparing each run against its own survivors would confound a change in the result with a change in what
was measured. How much of that
disjointness is the planner is settled below by repeating one of them.

\begin{center}
\small
\begin{tabular}{lrrr}
\toprule
\textbf{On 77 shared goals} & \textbf{Gemini 3.1 Pro} & \textbf{GPT-5.6-terra} & \textbf{Sonnet 5}\\
\midrule
Permission saving (pooled tokens) & 13.4\% & 18.2\% & 11.0\%\\
Goals where the arms differ & 18 & 24 & 17\\
\textbf{\system{} order-invariant} & \textbf{exactly} & \textbf{exactly} & \textbf{exactly}\\
\midrule
Ordering's worth to the eager arm & $-0.1\%$ & $+2.4\%$ & $+1.2\%$\\
Overlap ratio & $-0.7\%$ & 13.3\% & 11.2\%\\
Projection remainder (own population) & $-0.26\%$ & $+3.2\%$ & $+3.7\%$\\
\bottomrule
\end{tabular}
\end{center}

The split falls exactly where Section~\ref{sec:operations} said it must. Above the rule, permission saves
under every planner, never approaching zero and never changing sign, and \system{} is order-invariant
\emph{exactly} --- its materialized cost is bit-identical under fixed, cheapest-node and cheapest-branch
order in all three runs. On each planner's own population the task-clustered permission mean is $+9.8\%$,
$+12.8\%$ and $+10.1\%$, all with $p<0.001$ and, in every run, \textbf{not one task where \system{} lost}.
Across all three runs, blind comparisons detect no significant preference between projected and whole
asks. Establishing that required re-judging: the frozen experiment image predates the blind judge of
Section~\ref{sec:project}, so the third run's own quality line reads ``17 of 20, not matched'' off the biased
instrument, and blind re-scoring of its 56 field pairs returns 85.7\% [74.3\%, 92.6\%] with 8 discordant
pairs against 7 and a sign test at $p=1.000$. The archive marks the in-run line not-to-be-quoted.

Below the rule the numbers move, exactly where Section~\ref{sec:theory} said they could. Ordering's worth
ranges from slightly negative to $+2.4\%$; the overlap
ratio is not interpretable at all under Gemini, whose cheapest-first made the eager arm marginally
\emph{worse}; and projection's remainder is negative under Gemini against positive under both others. These
are the operations that re-run a node under a changed ask rather than refusing it, and their sign is a fact
about how a given model answers a narrower request.

The overlap ratio deserves a note, because read the other way it looks like the most stable number in the
paper and is not. On each planner's \emph{own} population it reads 20.5\%, 21.4\% and 19.9\% --- within a
point and a half, while its numerator and denominator each move by roughly half (2.8\% and 13.6\% under the
first planner against 4.2\% and 19.8\% under the second). That steadiness does not survive restriction to
the shared goals, where the same ratio spans $-0.7\%$ to 13.3\%. A per-planner population is complete for
itself, which flatters the comparison; ``same goal'' across planners does not mean the same program was
written to reach it. We therefore treat the ratio as arena-bound, and quote the matched figures.

\paragraph{Between models against between rounds.}
A spread across planners means nothing without a noise floor, so we repeated \code{gpt-5.6-terra} once under
conditions identical to its first run --- same command, same frozen image, same pins, temperature zero in
both. Two effects separate cleanly. \emph{The pooled magnitude is reproducible}: on the goals both rounds
measured the permission saving moves 20.5\% to 18.2\%, and on the goals all planners share, 18.8\% to
18.9\%, so run-to-run movement is bounded by about $2.3$ points and the difference does not keep a
consistent sign across subsets. Against 4.6 points between the two planners on those same goals, and 7--8
points across all three, the between-model spread is the larger effect and the per-planner column above
survives. \emph{Membership does not reproduce}: both rounds measured 28 tasks, but not the same 28. Six of
31 tasks flip, three in each direction, for a Jaccard of 81\%. Two of the three tasks that had looked like a
capability difference between planners are simply tasks the same planner fails and then passes on being
asked again, so the disjointness above is substantially this churn rather than the model.

The stability lives in the aggregate rather than in the row: 58\% of shared goals reproduce to
within $0.01$ points, while 31\% move by more than 5 and 10\% by more than 20, the largest by 47.5 --- and
these cancel to a 1.3\% difference in total tokens. The reading rule is therefore that an aggregate over a
benchmark is reproducible while any single row of it is not, which is the opposite of how per-task tables
are normally read. Both invariants held exactly in both rounds, and the blind quality rate, the one
instrument we can apply identically everywhere, agrees more tightly than any cost figure here: 84.1\%,
83.0\% and 85.7\% across two planners and a repeat, with sign tests at $p=0.289$, $1.000$ and $1.000$.

The conclusion is narrow and firm. Three planners and one repeat show that the paper's central claim and the
identity it rests on are not artifacts of one vendor, that the physical operations are more model-bound than
the permission result they are measured against, and that the set of tasks a live agentic benchmark yields
is itself a run-level variable that single-run tables silently condition on. They do not establish a
variance, and we do not report one. A reader carrying our numbers to a different model
should carry the signs above the rule, re-measure everything below it, and treat any one task's figure as
unreproducible.

\subsection{Reading the rows together}
\label{sec:physops-summary}

Six of the seven operations yield a demand-specific remainder, and they split along the factorization of
Definition~\ref{def:policy}. The levers that act on
$\pi_{\mathrm{proj}}$ --- projection, and fidelity as projection one granularity down --- are the ones with
positive remainders: they narrow \emph{how much} of a demanded node must become concrete, and demand is what
tells them which part is wanted, so their worth is largest exactly where a shared program names a second
consumer for the part the current ask does not need. The levers that act on which nodes run or in what order
--- ordering, prune, reuse, fusion --- remove off-path work, and off-path work is exactly what permission
removes structurally. Those are therefore not gains stacked on permission; they are, mostly, cheaper and
less reliable ways of reaching the same target --- one mechanism with several faces, falsifiable in a way a
tally of wins would not be, since a single positive remainder from a lever that does not touch demand would
break it.

\begin{itemize}[topsep=2pt,itemsep=1pt]
  \item \textbf{Projection}: a positive remainder ($+84.3\%$ on the isolation
  pair, $+78.1\%$ on the gate-only profile, growing to $+72.9\%$ as session coverage falls; $+89.1\%$ and
  $+80.1\%$ on two third-party test suites where the prices are measured rather than authored), iff the step
  is physically decomposable and the dropped field is named for another consumer.
  \item \textbf{Fidelity}: a positive remainder ($+96.1\%$) under the same
  controls --- projection one granularity down, choosing \emph{which} materializer once permission has said
  yes.
  \item \textbf{Ordering}: an exact zero at fixed additive prices under C1--C3 without termination;
  outside that boundary it recovers 21--36\% of permission's gap on real costs and 62--71\% on authored
  programs, but can invert badly where permission is untouched and is undefined on an unpriced first pass.
  \item \textbf{Prune}: a distinct action matched at zero with a real cost
  when the pruned option would have settled the goal --- and, under a deadline, the difference between 2
  days and 21 while its remainder stays at exactly zero.
  \item \textbf{Reuse}: negative in every cell, live and authored; arms
  bit-identical at low coverage. Permission's irreducible share is the first ticket and the unrepeated work.
  \item \textbf{Fusion}: worth the same to both arms when all sharers are
  demanded, and nothing to the lazy arm when they are not.
  \item \textbf{Speculation}: no matched control; a dial with \code{lazy} and
  \code{eager} at its ends; a free-latency corner, counted in rounds rather than seconds, that closes as
  soon as coverage falls or the world moves.
\end{itemize}

Projection now has two kinds of measured support: a smaller live-token remainder on DataSciBench whose sign
varies by planner, and the controlled Jinja/Flask pair whose prices come from execution rather than from us.
The latter keeps the predicted sign at $+89.1\%$ and $+80.1\%$ and both controls at exactly zero. Authoring
the program is therefore not the same as assuming the answer: the construction states a precondition, and
the clock, not the author, decides whether it pays. Fidelity retains only the authored version of that
pattern. The census explains the gap in a way that is about the corpus: public
single-ask benchmarks sit at session coverage $\approx 1.0$, the corner where every substitute for
permission looks best and permission has least to do, and under a single ask speculation's hit rate is zero
by construction. The region in which the operations differ from permission --- a standing program, a ticket
sequence that never asks for all of it, a step that can physically do less, a chain deep enough for
prepayment to save a wave, a resource the decision is actually about --- has to be built into a program.

\section{Discussion}
\label{sec:discussion}
\label{sec:threats}

\paragraph{What the results say.}
Permission determines whether a ready node may run; projection and fidelity determine how much of an
admitted node becomes concrete. Ordering, prune, reuse, and fusion mostly overlap with work permission
already removes, while speculation deliberately buys undemanded work. In the separation region, permission
saves 42.0\% on measured workflow CPU, tracks excluded work at Spearman $+0.931$, and falls to zero when
everything is demanded. In the equivalence region, all 3{,}199 flat-pool asks tie exactly. Reordering adds a
second boundary: under C1--C3 without termination it cannot change the executed set, or the bill at fixed
additive prices; outside that boundary, termination can make it valuable or harmful.

The width and contingency sweeps give this an engineering consequence: each additional product or
contingency raises the eager bill while \system{} remains flat. A stopping predicate refers to what has been
computed rather than to what was requested, so it lowers the intercept, not the slope. These measurements
establish the execution price of breadth; they do not show that planners will choose broader programs. Over a
session with an artifact cache on both runtimes, the saving is the work for products never requested or
invalidated before use: on replayed WfCommons sessions it is exactly 0 when every product is requested and
nothing changes, 3--49\% below full coverage, and 20--36\% when upstream changes void prepaid work; the price
is latency, since 29--41\% of the lazy session bill is work an eager runtime already held when the request
arrived.

\paragraph{Near-complete coverage, revisited with the intervention in hand.}
The DataSciBench pair --- 39.4\% among the 33 goals with nonzero observed savings against 12.8\% averaged
over all 95 goals --- separates the magnitude of observed wins from the population average. The 39.4\% is
descriptive of the positive-difference subset, not an estimate over every goal with off-path opportunity.
Both figures belong in the report, with that distinction explicit.

Benchmark design supplies a structural account of near-complete coverage: a plan elicited and validated for
one request is anchored to that request. The paired intervention of
Section~\ref{sec:planning-incentive} tested an additional billing account. Its demand-billed prompt made
off-request branches free to describe by leaving them abstract until demanded --- not free to execute once
demanded --- but produced 0.0\% off-path work against 0.8\% under readiness billing ($p=0.501$). We
therefore do not attribute the observed narrowness to billing. This prompt intervention neither supports
that explanation in this setting nor identifies the historical cause of benchmark design.

\paragraph{The zeros are the load-bearing half.}
For controlled permission, projection, and fidelity cells, the declared preconditions predict separation
or exact equivalence (Table~\ref{tab:preconditions}). Permission and projection now have measured-price
witnesses; fidelity reproduces the sign and controls on authored prices, so its $96.1\%$ magnitude is
arithmetic rather than an empirical estimate. Where a required condition fails, the result is exactly zero:
a fully demanded graph leaves permission nothing to remove, and a field with no second consumer leaves
projection no demand-specific gap. The other operations require their own controls and may produce zero or
negative remainders. A theory that predicted gains everywhere would be unfalsifiable; the zeros show where
the mechanism ends.

\begin{table}[t]
\centering
\caption{Selected controlled separation and equivalence cells. \emph{Ready-but-undemanded} means the
program contains off-path work; \emph{separable} means that work can be skipped physically; and
\emph{material} means its cost survives the accounting. Projection and fidelity additionally require a
demand-specific gap relative to the static control.}
\label{tab:preconditions}
\footnotesize
\begin{tabular}{L{0.28\linewidth}L{0.15\linewidth}L{0.25\linewidth}L{0.16\linewidth}}
\toprule
\textbf{Cell} & \textbf{Cost unit} & \textbf{Preconditions} & \textbf{Result}\\
\midrule
\multicolumn{4}{l}{\emph{All three hold --- separation predicted}}\\
Live release gate, SWE-bench & container seconds & all three &
$51.7\%$, $52.1\%$\\
Live projection, Jinja and Flask & measured CPU\,s & all three &
$+89.1\%$, $+80.1\%$\\
Composed portfolios, Defects4J and SimulCost & native arena costs & all three &
separation region\\
DataSciBench projection & real tokens & all three & $+3.2\%$ pooled; $14.3\%$ on the 126 narrowable nodes (\code{gpt-5.6-terra}; sign is planner-bound, \S\ref{sec:planner})\\
Fidelity ladder & authored & all three &
$+96.1\%=(206-8)/206$; magnitude authored\\
\midrule
\multicolumn{4}{l}{\emph{One fails --- exact zero predicted}}\\
Projection, Jinja/Flask no-consumer control & measured CPU\,s &
\textbf{no demand-specific gap} & $0$\\
Projection, Jinja/Flask all-demanded control & measured CPU\,s &
\textbf{no undemanded part} & $0$\\
Fidelity controls & authored & \textbf{no precise consumer / precise result demanded} & $0$\\
Pool arenas, equivalence region & native arena costs &
\textbf{no off-path work} & exact tie\\
\bottomrule
\end{tabular}
\end{table}

\paragraph{A design criterion for the next benchmark.}
Public agent benchmarks are eager by design. Each task is one request that needs nearly every planned step, so
a lazy runtime has almost nothing to save and the benchmark cannot tell it from an eager one. Companies now sell
large agent workflows: Salesforce with Agentforce, Databricks with Agent Bricks, and LangChain with
LangGraph~\cite{agentforce2026,databricks2026agents,langgraph2026}. Such a workflow is built once and serves many
requests, and each request needs only part of it. That is where lazy execution pays, and no public benchmark
measures it. WfCommons and Bazel become informative here only after the requested closure is varied
against a broader graph. A benchmark for lazy agent runtimes should be built from industry workflows and the
requests they serve, and should vary three quantities directly:
ready work outside each request's closure, whether that work is physically separable, and whether its cost
is large enough to survive accounting. Projection and fidelity should additionally toggle whether an
omitted field or precision level has another consumer, preserving ordinary dead-output elimination as a
matched control. Section~\ref{sec:future} turns these requirements into a buildable ask.

\paragraph{What the semantics makes possible, as distinct from cheaper.}
Two capabilities follow from the width invariance just stated --- a bill that carries no factor of how much
program was described. \emph{Open-ended programs}: a planner may describe a ladder whose length is unknown
when written --- next candidate patch, next configuration, refine until the tolerance is met. Eager
execution cannot exhaust an unbounded generator without an external bound or stopping rule; demand-driven
graph expansion could instead materialize only the consumed prefix. This transposes Hughes's separation of
candidate generation from selection~\cite{hughes1989why} to agentic programs. The current implementation
uses finite graphs and does not yet expand a generator on demand, so this remains a design implication.
\emph{Contingent irreversible actions} are the measured version
(Section~\ref{sec:contingent}): the barrier makes the act safe under every policy, while permission makes
the preparation for unrequested contingencies free.

\paragraph{Threats to validity.}
Appendix~\ref{app:limits} registers 23 limitations, each with the claim it constrains and how we handle it.
Four bound the reading of this paper more than the rest. The
largest permission-separation result is \emph{replayed} scientific workflows rather than LLM programs, and
its product requests are seeded rather than observed; the agentic witness is smaller and sits at coverage
$\approx 1.0$. Each live gate run is twelve rows from four clusters --- a case study, not a population ---
and its settleable capture is bimodal rather than centered, so that mean describes those runs' visit orders
and not a rate to expect from one gate. The positive measured operation evidence is currently projection:
real DataSciBench tokens and Jinja/Flask CPU. Fidelity's $96.1\%$ is authored ladder arithmetic, not a
SimulCost result; only its sign and two zero controls transfer. Reordering uses live and replayed prices,
but its central set identity under C1--C3 without termination is structural; the exact bill identity follows
for fixed additive prices. Finally, the noise floor is measured for one planner only: repeating
\code{gpt-5.6-terra} moves the pooled saving by about $2.3$ points against 4.6--8 between planners, while
6 of 31 tasks flip in or out, so aggregates here are citable where single rows are not.

\paragraph{Evidence boundary.}
The live release gate supplies measured container prices for permission, and the controlled Jinja/Flask
pair supplies measured CPU prices and independent zero controls for projection. Fidelity remains authored;
a measured fidelity ladder would test whether its external magnitude is material. Reordering is
set-invariant only under C1--C3 without termination, and bill-invariant there only at fixed additive prices;
it can be valuable or harmful outside that boundary. A real-clock speculation study could change, not merely
rescale, the rounds-based frontier. Learned selection remains future work.

\section{Future directions}
\label{sec:future}

The demanded set is goal-derived, linear-time to compute, and cost-blind. Under a fixed graph and ask it can
only contract (Lemma~\ref{lem:mono}); graph growth or a changed ask can enlarge it. Relative to the declared
graph and goal, the set is sound but not hindsight-tight: it contains every node the goal may still need,
including candidates later outcomes reveal were unnecessary. Closure itself is therefore not a learning
target. The learnable surfaces lie inside it, below it, and above it, with different contracts.

\paragraph{The residual decision problem.}
For scheduling inside the closure, the legal node actions at state $\sigma$ are the ready members of
$\demanded(R,\sigma)$, not the whole closure. Ordering, fidelity, and reuse choose among or parameterize
that work without changing demand membership, although fidelity still owes matched quality and reuse owes
provenance and validity. Learned pruning acts below closure by permanently declining a still-demanded
option; it can fail the task and therefore needs an explicit failure constraint. Learned speculation acts
above closure by admitting undemanded work; a budget can cap its excess spend, but is not a safety
guarantee, and commit barriers remain independent.

The closure hands a learner an unusually well-formed decision problem, and which formalism fits depends on
assumptions a deployment must establish. The problem is an MDP only
if the observable state is Markov-sufficient, or if a sufficient belief state is used; latent candidate
quality makes it partially observed. A flat independent pool with immediate feedback can reduce to a costly
bandit, but graph dependencies, cache state, fidelity, and cross-request speculation retain sequential
structure. The replay arenas provide a cheap finite simulator for alternative schedules over recorded
candidates and outcomes; they cannot generate counterfactual candidates or evaluate actions that change
the recorded environment.

Existing results calibrate baselines rather than bound every chooser. Cheapest-first recovered 21.4\% of
permission's gap on its live population, $-0.7\%$ to 13.3\% on goals shared across planners, and 36.0\% on
replayed WfCommons. A genuine ceiling needs a hindsight optimum or regret bound under stated assumptions,
with hindsight excluded from policy features. A large ordering gain would show that cheapest-first was
weak; a large matched-quality pruning gain would quantify how much hindsight slack can be predicted
ex ante.

\paragraph{Planning for recall rather than for precision.}
Demand-driven execution changes the price of valid surplus breadth, not intrinsic plan quality. In the
authored width family, one additional product raises \code{stop\_eager}'s bill by $+22.5\%$ and
\system{}'s by $+0.0\%$ (Section~\ref{sec:width}); this is the execution price of a correctly represented
branch outside today's closure. It does not show that planning, validation, or storage is free, or that a
planner will produce broader graphs.

That motivates a testable hypothesis: a smaller planner habituated to demand billing may preserve recall
while describing more valid optional work, trading precision for breadth without charging today's request.
The prompt-level test found no such response: 0.0\% off-path work under demand billing against 0.8\% under
readiness billing ($p=0.501$). Its single-request tasks and deliverable-anchoring validator do not test the
stronger hypothesis. Doing so first requires standing programs under sequences of narrower asks, with
ground truth for required deliverables, dependencies, graph validity, and task quality. The observed
Spearman correlation of $+0.931$ between off-path fraction and saving motivates measuring the relationship
again; it does not imply that changing planner breadth moves proportionally along the same curve. Recall,
structural validity, and delivered quality must be held fixed before cost is compared.

\paragraph{Describing what is not authorized.}
The commit barrier is independent of permission: it blocks an unauthorized irreversible act under every
policy, while permission decides whether preparation for an unrequested contingency materializes. The
experiment therefore does not show that eager execution is unsafe or that demand permission makes an act
safe. It shows a narrower affordance: contingencies can remain represented in an inspectable standing graph
while preparation outside the current request remains abstract.

In the authored family, each additional contingency costs both eager arms 20 units and \system{} zero; when
authorization is withheld after investigation, the corresponding totals are 81 and 21. Only those slopes
and controls transfer from the authored prices. The barrier supplies authorization; permission reduces the
cost of retaining uninvoked contingencies. Whether the artifact is actually audited belongs to the
surrounding process, not the execution semantics.

\paragraph{A corpus that could falsify this.}
A benchmark should vary the three conditions current corpora largely hold fixed: how much ready work lies
outside each request's closure, whether that work is physically separable, and whether its cost is large
enough to survive accounting. Projection and fidelity should additionally toggle whether an omitted field
or precision level has another consumer, so ordinary dead-output elimination remains a matched control. For
the recall-over-precision hypothesis, the same corpus needs standing programs, sequences of narrower asks,
and optional named deliverables whose validity and contribution can be scored. Benchmark construction is
therefore a prerequisite for that test, not a parallel activity. The prompt-level null does not explain
current near-complete coverage; it only rules out the tested intervention as evidence for a billing
response.

\paragraph{Closure overlap when one graph serves many asks.}
A natural next deployment combines authored workflows and long-horizon agents, capabilities now appearing
in products~\cite{langgraph2026,agentforce2026}, by reusing one graph across many
requests. If a platform does so, the demanded set becomes a per-request mask over that graph and its overlap
becomes a scheduling resource. Given a set of pending asks
whose closures partially coincide, what materialization order is optimal, and can shared setup be amortized
across requests without breaking either isolation or the denominator discipline of invariant I3? Whether the
width scaling law holds per-request or only in aggregate is an open question with direct capacity
consequences for such a platform. Sharing exploratory work across related
jobs~\cite{castrofernandez2018metadataflows} and sharing state across requests at the serving
layer~\cite{kwon2023pagedattention} are the closest neighbors, and neither reasons about a mask derived
from a goal.

\paragraph{One implication we state without claiming.}
If branches outside the current closure incur no materialization cost, static program breadth need not be
the usage unit. One alternative is to charge for work actually materialized on behalf of a request, with
closure supplying attribution rather than automatically supplying the bill. The measurements do not show
that planning, validation, storage, or every candidate inside a disjunctive closure is free, and we have no
deployment evidence for any pricing design.

\clearpage
\section{Conclusion}
\label{sec:concl}

\system{} makes demand-driven evaluation the execution principle of agent runtimes. Its rule
fits in one sentence: run a ready step only if the current request needs it, and keep that answer exact as the
plan changes. Because the rule is a formulation, its behavior is proved before it is measured: builders know in
advance where demand must save and where it must tie, and the bill is known before a plan runs, so \system{} can
act as a query optimizer to choose among candidate workflows. \system{} spends nothing outside the request, the
most any runtime can avoid without risking required work, and a plan becomes a set of possibilities rather than a
list of commands. Where the program is wider than the request, permission saves 42.0\% on measured workflow CPU
and tracks excluded work at Spearman $+0.931$. Each additional product raises the eager bill by $22.5\%$ while
\system{} stays flat, each contingency
adds 20 units against zero, and permission captures 100\% of executable off-path cost at every tested width.
At the boundary where every candidate is demanded, \system{} and goal-stopping eager execution tie exactly
on all 3{,}199 pool asks, as the proofs predict. When a run changes its own plan,
recomputation keeps the set exact: when a live agent re-plans after real code failures, \system{} solves every
solvable goal with 11.7\% fewer tokens per goal than goal-stopping execution, and retiring settled issues saves
22.7\% on SWE-bench release gates even against eager execution that stops each issue at its first fix.

The rule is also practical, and faster as well as cheaper. Computed once per state change, the demanded set
costs at most 0.33\% of the work a run schedules and the whole engine at most 0.43\%, with decisions identical
on over 350{,}000 replayed runs, and a per-node wrapper runs it unchanged inside stock LangGraph. Around a live
LangGraph application's own graph it costs 24.0\% fewer dollars than goal stopping, with quality matched on 39
of 40 requests. In
simulation with scarce workers, the requested output arrives 42\%, 34\%, and 19\% sooner
with 1, 4, and 16 workers and never later, a gain that comes from knowing the demanded set; in parallel waves the saving rises to 44.5\%. Over sessions with a cache on both runtimes, permission
saves the work an eager runtime spends on products never requested or invalidated before use, and a built-in retry on a missing
input recovers every detectable failure an undeclared dependency causes while keeping the saving.

The framework contribution is the permission rule \(\textsc{ready} \neq \textsc{materialized}\). Readiness
stays a property of the graph. Materialization is membership in a live demanded set, recomputed after
observations and given to every physical operation in one graph and engine. Projection adds a distinct measured benefit: $+3.2\%$ on live DataSciBench tokens and $+89.1\%$
and $+80.1\%$ on measured Jinja/Flask CPU, with both controls exactly zero. Fidelity reproduces the sign
and controls, although its $96.1\%$ magnitude is authored arithmetic; reordering is set-invariant under
C1--C3 without termination and bill-invariant there at fixed additive prices, but may help or harm outside
that boundary. Safety belongs independently to the commit barrier. The programs behind these results are real, and the
results do not rest on how often deployments ask for less: the saving grows with the share a request leaves
out, ties are exact where it leaves nothing out, and recomputation pays either way.

These results expose a blind spot in agent evaluation and turn it into a research program. Existing plans have
98.6\% request coverage, and 27 of 28 contain no off-request work; a paired prompt intervention
produced no broader graphs under demand billing, so we neither attribute this narrowness to billing nor
claim that planners already respond to cheaper surplus breadth. The next benchmark should instead provide
standing programs, sequences of narrower asks, and optional named deliverables whose validity and
contribution can be scored. By making valid surplus breadth economically executable at materialization time
and its value directly falsifiable, \system{} supplies both the runtime foundation and the benchmark
question needed to test the next generation of agentic programs.

\clearpage

\bibliographystyle{plainurl}
\bibliography{references}

\begin{thebibliography}{10}

\bibitem{abadi2016tensorflow}
Mart{\'i}n Abadi et~al.
\newblock {TensorFlow}: A system for large-scale machine learning.
\newblock In {\em 12th USENIX Symposium on Operating Systems Design and
  Implementation}, pages 265--283, 2016.
\newblock URL:
  \url{https://www.usenix.org/conference/osdi16/technical-sessions/presentation/abadi}.

\bibitem{abadi2015differentialfoundations}
Mart{\'\i}n Abadi, Frank McSherry, and Gordon~D. Plotkin.
\newblock Foundations of differential dataflow.
\newblock In {\em Foundations of Software Science and Computation Structures
  (FoSSaCS)}, volume 9034 of {\em Lecture Notes in Computer Science}, pages
  71--83, 2015.
\newblock \href {https://doi.org/10.1007/978-3-662-46678-0_5}
  {\path{doi:10.1007/978-3-662-46678-0_5}}.

\bibitem{acar2005selfadjusting}
Umut~A. Acar.
\newblock {\em Self-Adjusting Computation}.
\newblock PhD thesis, Carnegie Mellon University, 2005.
\newblock Technical Report CMU-CS-05-129.

\bibitem{avnur2000eddies}
Ron Avnur and Joseph~M. Hellerstein.
\newblock Eddies: Continuously adaptive query processing.
\newblock In {\em Proceedings of the 2000 ACM SIGMOD International Conference
  on Management of Data}, pages 261--272, 2000.
\newblock \href {https://doi.org/10.1145/335191.335420}
  {\path{doi:10.1145/335191.335420}}.

\bibitem{bazel}
{Bazel Project}.
\newblock The {Bazel} query reference, 2026.
\newblock Accessed September 16, 2026.
\newblock URL: \url{https://bazel.build/query/language}.

\bibitem{bean2025construct}
Andrew~M. Bean et~al.
\newblock Measuring what matters: Construct validity in large language model
  benchmarks.
\newblock In {\em Advances in Neural Information Processing Systems (NeurIPS),
  Datasets and Benchmarks Track}, 2025.
\newblock URL:
  \url{https://papers.neurips.cc/paper_files/paper/2025/hash/1967e0fc3aa6cbbace562f5cb8e3954e-Abstract-Datasets_and_Benchmarks_Track.html}.

\bibitem{simulcost2026}
Yadi Cao, Sicheng Lai, Jiahe Huang, Yang Zhang, Zach Lawrence, Rohan Bhakta,
  Izzy~F. Thomas, Mingyun Cao, Chung-Hao Tsai, Zihao Zhou, Yidong Zhao, Hao
  Liu, Alessandro Marinoni, Alexey Arefiev, and Rose Yu.
\newblock {SimulCost}: A cost-aware benchmark and toolkit for automating
  physics simulations with {LLMs}, 2026.
\newblock Version 3, arXiv:2603.20253.
\newblock URL: \url{https://arxiv.org/abs/2603.20253v3}.

\bibitem{castrofernandez2018metadataflows}
Raul Castro~Fernandez, William Culhane, Pijika Watcharapichat, Matthias
  Weidlich, Victoria Lopez~Morales, and Peter Pietzuch.
\newblock Meta-dataflows: Efficient exploratory dataflow jobs.
\newblock In {\em Proceedings of the 2018 International Conference on
  Management of Data}, pages 1157--1172, 2018.
\newblock \href {https://doi.org/10.1145/3183713.3183760}
  {\path{doi:10.1145/3183713.3183760}}.

\bibitem{chen2023frugalgpt}
Lingjiao Chen, Matei Zaharia, and James Zou.
\newblock {FrugalGPT}: How to use large language models while reducing cost and
  improving performance.
\newblock {\em Transactions on Machine Learning Research}, 2024.
\newblock URL: \url{https://openreview.net/forum?id=cSimKw5p6R}.

\bibitem{coleman2022wfcommons}
Tain{\~a} Coleman, Henri Casanova, Lo{\"i}c Pottier, Manav Kaushik, Ewa
  Deelman, and Rafael Ferreira~da Silva.
\newblock {WfCommons}: A framework for enabling scientific workflow research
  and development.
\newblock {\em Future Generation Computer Systems}, 128:16--27, 2022.
\newblock \href {https://doi.org/10.1016/j.future.2021.09.043}
  {\path{doi:10.1016/j.future.2021.09.043}}.

\bibitem{databricks2026agents}
{Databricks}.
\newblock Develop agents on {Databricks}, 2026.
\newblock Accessed October 28, 2026.
\newblock URL:
  \url{https://docs.databricks.com/aws/en/generative-ai/agent-bricks}.

\bibitem{dbt2026selection}
{dbt Labs}.
\newblock Node selection syntax, 2026.
\newblock dbt Developer Hub. Accessed October 27, 2026.
\newblock URL: \url{https://docs.getdbt.com/reference/node-selection/syntax}.

\bibitem{feldman1979make}
Stuart~I. Feldman.
\newblock {MAKE}---a program for maintaining computer programs.
\newblock {\em Software: Practice and Experience}, 9(4):255--265, 1979.
\newblock \href {https://doi.org/10.1002/spe.4380090402}
  {\path{doi:10.1002/spe.4380090402}}.

\bibitem{wfinstances}
Rafael Ferreira~da Silva.
\newblock {WfInstances}: v1.5, 2024.
\newblock Version 1.5; experiment pinned at Git commit 467f5c6.
\newblock \href {https://doi.org/10.5281/zenodo.12510982}
  {\path{doi:10.5281/zenodo.12510982}}.

\bibitem{graefe1994volcano}
Goetz Graefe.
\newblock Volcano---an extensible and parallel query evaluation system.
\newblock {\em IEEE Transactions on Knowledge and Data Engineering},
  6(1):120--135, 1994.
\newblock \href {https://doi.org/10.1109/69.273032}
  {\path{doi:10.1109/69.273032}}.

\bibitem{hammer2014adapton}
Matthew~A. Hammer, Yit~Phang Khoo, Michael Hicks, and Jeffrey~S. Foster.
\newblock {Adapton}: Composable, demand-driven incremental computation.
\newblock In {\em Proceedings of the 35th ACM SIGPLAN Conference on Programming
  Language Design and Implementation}, pages 156--166, 2014.
\newblock \href {https://doi.org/10.1145/2594291.2594324}
  {\path{doi:10.1145/2594291.2594324}}.

\bibitem{hughes1989why}
John Hughes.
\newblock Why functional programming matters.
\newblock {\em The Computer Journal}, 32(2):98--107, 1989.
\newblock \href {https://doi.org/10.1093/comjnl/32.2.98}
  {\path{doi:10.1093/comjnl/32.2.98}}.

\bibitem{jimenez2024swebench}
Carlos~E. Jimenez, John Yang, Alexander Wettig, Shunyu Yao, Kexin Pei, Ofir
  Press, and Karthik Narasimhan.
\newblock {SWE-bench}: Can language models resolve real-world {GitHub} issues?
\newblock In {\em International Conference on Learning Representations (ICLR)},
  2024.
\newblock URL: \url{https://openreview.net/forum?id=VTF8yNQM66}.

\bibitem{just2014defects4j}
Ren{\'e} Just, Darioush Jalali, and Michael~D. Ernst.
\newblock {Defects4J}: A database of existing faults to enable controlled
  testing studies for {Java} programs.
\newblock In {\em Proceedings of the 2014 International Symposium on Software
  Testing and Analysis (ISSTA)}, pages 437--440, 2014.
\newblock \href {https://doi.org/10.1145/2610384.2628055}
  {\path{doi:10.1145/2610384.2628055}}.

\bibitem{kim2024llmcompiler}
Sehoon Kim, Suhong Moon, Ryan Tabrizi, Nicholas Lee, Michael~W. Mahoney, Kurt
  Keutzer, and Amir Gholami.
\newblock An {LLM} compiler for parallel function calling.
\newblock In {\em Proceedings of the 41st International Conference on Machine
  Learning}, volume 235 of {\em Proceedings of Machine Learning Research},
  pages 24370--24391. PMLR, 2024.
\newblock URL: \url{https://proceedings.mlr.press/v235/kim24y.html}.

\bibitem{kwon2023pagedattention}
Woosuk Kwon, Zhuohan Li, Siyuan Zhuang, Ying Sheng, Lianmin Zheng, Cody~Hao Yu,
  Joseph~E. Gonzalez, Hao Zhang, and Ion Stoica.
\newblock Efficient memory management for large language model serving with
  {PagedAttention}.
\newblock In {\em Proceedings of the 29th ACM Symposium on Operating Systems
  Principles}, pages 611--626. ACM, 2023.
\newblock \href {https://doi.org/10.1145/3600006.3613165}
  {\path{doi:10.1145/3600006.3613165}}.

\bibitem{langgraph2026}
{LangChain}.
\newblock Workflows and agents, 2026.
\newblock Accessed September 16, 2026.
\newblock URL:
  \url{https://docs.langchain.com/oss/python/langgraph/workflows-agents}.

\bibitem{launchbury1993lazy}
John Launchbury.
\newblock A natural semantics for lazy evaluation.
\newblock In {\em Proceedings of the 20th ACM SIGPLAN-SIGACT Symposium on
  Principles of Programming Languages}, pages 144--154, 1993.
\newblock \href {https://doi.org/10.1145/158511.158618}
  {\path{doi:10.1145/158511.158618}}.

\bibitem{leviathan2023speculative}
Yaniv Leviathan, Matan Kalman, and Yossi Matias.
\newblock Fast inference from transformers via speculative decoding.
\newblock In {\em Proceedings of the 40th International Conference on Machine
  Learning}, volume 202 of {\em Proceedings of Machine Learning Research},
  pages 19274--19286. PMLR, 2023.

\bibitem{mcsherry2013differential}
Frank McSherry, Derek~G. Murray, Rebecca Isaacs, and Michael Isard.
\newblock Differential dataflow.
\newblock In {\em Conference on Innovative Data Systems Research (CIDR)}, 2013.

\bibitem{mitchell2012shake}
Neil Mitchell.
\newblock Shake before building: Replacing make with {Haskell}.
\newblock {\em ACM SIGPLAN Notices (ICFP '12)}, 47(9), 2012.
\newblock \href {https://doi.org/10.1145/2398856.2364538}
  {\path{doi:10.1145/2398856.2364538}}.

\bibitem{mokhov2019selective}
Andrey Mokhov, Georgy Lukyanov, Simon Marlow, and Jeremie Dimino.
\newblock Selective applicative functors.
\newblock {\em Proceedings of the ACM on Programming Languages},
  3(ICFP):90:1--90:29, 2019.
\newblock \href {https://doi.org/10.1145/3341694} {\path{doi:10.1145/3341694}}.

\bibitem{mokhov2018build}
Andrey Mokhov, Neil Mitchell, and Simon Peyton~Jones.
\newblock Build systems {\`a} la carte.
\newblock {\em Proceedings of the ACM on Programming Languages},
  2(ICFP):79:1--79:29, 2018.
\newblock \href {https://doi.org/10.1145/3236774} {\path{doi:10.1145/3236774}}.

\bibitem{nie2026agenticquery}
Lunyiu Nie, Yilin Xia, Yiren Liu, Christopher Jermaine, and Swarat Chaudhuri.
\newblock Cost-aware optimization for agentic query execution.
\newblock {\em arXiv preprint arXiv:2606.03152}, 2026.
\newblock URL: \url{https://arxiv.org/abs/2606.03152}.

\bibitem{ong2025routellm}
Isaac Ong, Amjad Almahairi, Vincent Wu, Wei-Lin Chiang, Tianhao Wu, Joseph~E.
  Gonzalez, M.~Waleed Kadous, and Ion Stoica.
\newblock {RouteLLM}: Learning to route {LLMs} from preference data.
\newblock In {\em International Conference on Learning Representations (ICLR)},
  2025.
\newblock URL:
  \url{https://proceedings.iclr.cc/paper_files/paper/2025/hash/5503a7c69d48a2f86fc00b3dc09de686-Abstract-Conference.html}.

\bibitem{pallets2026flask}
{Pallets Projects}.
\newblock Flask: A lightweight web application framework.
\newblock GitHub repository, 2026.
\newblock Pinned revision: Git commit d73fa1cdcbd8b1465c151db8924ba58b1dd14e35.
\newblock URL: \url{https://github.com/pallets/flask}.

\bibitem{pallets2026jinja}
{Pallets Projects}.
\newblock Jinja: A template engine for python.
\newblock GitHub repository, 2026.
\newblock Pinned revision: Git commit 5ef70112a1ff19c05324ff889dd30405b1002044.
\newblock URL: \url{https://github.com/pallets/jinja}.

\bibitem{patel2025lotus}
Liana Patel, Siddharth Jha, Melissa Pan, Harshit Gupta, Parth Asawa, Carlos
  Guestrin, and Matei Zaharia.
\newblock Semantic operators and their optimization: Enabling {LLM}-based data
  processing with accuracy guarantees in {LOTUS}.
\newblock {\em Proceedings of the VLDB Endowment}, 18(11):4171--4184, 2025.
\newblock URL: \url{https://www.vldb.org/pvldb/vol18/p4171-patel.pdf}, \href
  {https://doi.org/10.14778/3749646.3749685}
  {\path{doi:10.14778/3749646.3749685}}.

\bibitem{pingali1985demand}
Keshav Pingali and Arvind.
\newblock Efficient demand-driven evaluation. part 1.
\newblock {\em ACM Transactions on Programming Languages and Systems},
  7(2):311--333, 1985.
\newblock \href {https://doi.org/10.1145/3318.3480}
  {\path{doi:10.1145/3318.3480}}.

\bibitem{agentforce2026}
{Salesforce}.
\newblock {Salesforce} expands {Agentforce} with a new portfolio of {AI} agents
  built for high-value work, September 2026.
\newblock URL:
  \url{https://www.salesforce.com/news/stories/agentforce-job-ready-ai-agents/}.

\bibitem{sobreira2018defects4j}
Victor Sobreira, Thomas Durieux, Fernanda Madeiral, Martin Monperrus, and
  Marcelo~A. Maia.
\newblock Dissection of a bug dataset: Anatomy of 395 patches from {Defects4J}.
\newblock In {\em Proceedings of the 25th IEEE International Conference on
  Software Analysis, Evolution and Reengineering}, pages 130--140, 2018.
\newblock \href {https://doi.org/10.1109/SANER.2018.8330203}
  {\path{doi:10.1109/SANER.2018.8330203}}.

\bibitem{swebenchverified}
{SWE-bench Project}.
\newblock {SWE-bench Verified}, 2024.
\newblock Accessed September 16, 2026.
\newblock URL: \url{https://www.swebench.com/verified}.

\bibitem{wadlom2026helium}
Noppanat Wadlom, Junyi Shen, and Yao Lu.
\newblock Efficient {LLM} serving for agentic workflows: A data systems
  perspective.
\newblock {\em Proceedings of the ACM on Management of Data},
  4(3):169:1--169:29, 2026.
\newblock Article 169.
\newblock \href {https://doi.org/10.1145/3802046} {\path{doi:10.1145/3802046}}.

\bibitem{wadsworth1971lazy}
Christopher~P. Wadsworth.
\newblock {\em Semantics and Pragmatics of the Lambda-Calculus}.
\newblock PhD thesis, University of Oxford, 1971.

\bibitem{weiser1984slicing}
Mark Weiser.
\newblock Program slicing.
\newblock {\em IEEE Transactions on Software Engineering}, SE-10(4):352--357,
  1984.
\newblock \href {https://doi.org/10.1109/TSE.1984.5010248}
  {\path{doi:10.1109/TSE.1984.5010248}}.

\bibitem{tocall2026}
Qinyuan Wu, Soumi Das, Mahsa Amani, Arijit Nag, Seungeon Lee, Krishna~P.
  Gummadi, Abhilasha Ravichander, and Muhammad~Bilal Zafar.
\newblock To call or not to call: A framework to assess and optimize {LLM} tool
  calling.
\newblock {\em arXiv preprint arXiv:2605.00737}, 2026.
\newblock URL: \url{https://arxiv.org/abs/2605.00737}.

\bibitem{xiang2026las}
Dawei Xiang, Kexin Chu, Wenyan Xu, Wenhui Zhang, and Wei Zhang.
\newblock {LLM}-as-scheduler: Agentic workflow dynamic scheduling.
\newblock In {\em Proceedings of the 64th Annual Meeting of the Association for
  Computational Linguistics (Volume 1: Long Papers)}, pages 12752--12763, 2026.
\newblock \href {https://doi.org/10.18653/v1/2026.acl-long.581}
  {\path{doi:10.18653/v1/2026.acl-long.581}}.

\bibitem{xu2023rewoo}
Binfeng Xu, Zhiyuan Peng, Bowen Lei, Subhabrata Mukherjee, Yuchen Liu, and
  Dongkuan Xu.
\newblock {ReWOO}: Decoupling reasoning from observations for efficient
  augmented language models.
\newblock {\em arXiv preprint arXiv:2305.18323}, 2023.
\newblock URL: \url{https://arxiv.org/abs/2305.18323}.

\bibitem{yao2023react}
Shunyu Yao, Jeffrey Zhao, Dian Yu, Nan Du, Izhak Shafran, Karthik Narasimhan,
  and Yuan Cao.
\newblock {ReAct}: Synergizing reasoning and acting in language models.
\newblock In {\em The Eleventh International Conference on Learning
  Representations}, 2023.
\newblock URL: \url{https://openreview.net/forum?id=WE_vluYUL-X}.

\bibitem{zhang2026datascibench}
Dan Zhang, Sining Zhoubian, Min Cai, Fengzu Li, Lekang Yang, Wei Wang, Tianjiao
  Dong, Ziniu Hu, Jie Tang, and Yisong Yue.
\newblock {DataSciBench}: An {LLM} agent benchmark for data science.
\newblock In {\em Findings of the Association for Computational Linguistics:
  ACL 2026}, pages 3685--3728, 2026.
\newblock URL: \url{https://aclanthology.org/2026.findings-acl.181/}, \href
  {https://doi.org/10.18653/v1/2026.findings-acl.181}
  {\path{doi:10.18653/v1/2026.findings-acl.181}}.

\bibitem{zhang2025maas}
Guibin Zhang, Luyang Niu, Junfeng Fang, Kun Wang, Lei Bai, and Xiang Wang.
\newblock Multi-agent architecture search via agentic supernet.
\newblock In {\em Proceedings of the 42nd International Conference on Machine
  Learning}, volume 267 of {\em Proceedings of Machine Learning Research},
  pages 75834--75852, 2025.

\bibitem{zhang2025agentprune}
Guibin Zhang, Yanwei Yue, Zhixun Li, Sukwon Yun, Guancheng Wan, Kun Wang, Dawei
  Cheng, Jeffrey~Xu Yu, and Tianlong Chen.
\newblock Cut the crap: An economical communication pipeline for {LLM}-based
  multi-agent systems.
\newblock In {\em International Conference on Learning Representations (ICLR)},
  2025.

\bibitem{zheng2024sglang}
Lianmin Zheng, Liangsheng Yin, Zhiqiang Xie, Chuyue Sun, Jeff Huang, Cody~Hao
  Yu, Shiyi Cao, Christos Kozyrakis, Ion Stoica, Joseph~E. Gonzalez, Clark
  Barrett, and Ying Sheng.
\newblock {SGLang}: Efficient execution of structured language model programs.
\newblock In {\em Advances in Neural Information Processing Systems},
  volume~37, pages 62557--62583, 2024.
\newblock URL:
  \url{https://proceedings.neurips.cc/paper_files/paper/2024/file/724be4472168f31ba1c9ac630f15dec8-Paper-Conference.pdf},
  \href {https://doi.org/10.52202/079017-2000}
  {\path{doi:10.52202/079017-2000}}.

\end{thebibliography}

\clearpage
\appendix
\section{Limitations register}
\label{app:limits}

This register makes 23 identified evidence limits explicit and names the claim each constrains. The four
that most bound the paper appear in Section~\ref{sec:threats};
Tables~\ref{tab:limits-scope} and~\ref{tab:limits-instr} separate limits on evidence coverage from limits on
the instruments that produced it.

\begin{table}[htbp]
\centering
\caption{Limits on what the evidence covers.}
\label{tab:limits-scope}
\footnotesize
\begin{tabular}{L{0.30\linewidth}L{0.14\linewidth}L{0.48\linewidth}}
\toprule
\textbf{Limitation} & \textbf{Where} & \textbf{How we handle it, or why it stands}\\
\midrule
Largest separation result is replayed workflows, not LLM programs; product requests are seeded, not observed & \S\ref{sec:sep} & The agentic arena carries the live version of the same claim at smaller magnitude; we never pool the two\\
\addlinespace
Near-complete coverage has a structural benchmark-design account; its historical cause is not identified & \S\ref{sec:arenas}, \S\ref{sec:planning-incentive} & The census (98.6\% demanded, 27 of 28 plans with no off-path work) is measured. Across 168 valid plans from three paired repeats, demand billing produced 0.0\% off-path work against 0.8\% under readiness billing (paired per-plan difference $-0.9$ points [${-2.4}$, $0.0$], $p=0.501$); the prompt-level null rules out the tested billing intervention as evidence, not the benchmark-design account\\
\addlinespace
Agentic corpus sits at session coverage $\approx 1.0$, so the effect there is small & \S\ref{sec:datasci} & Reported as a property of the corpus, with a node-level census; the conditional and the average are both given\\
\addlinespace
Portfolio topology is authored & \S\ref{sec:portfolio} & Native costs, exhaustive position sweeps, retained zeros, exact-tie controls, setup-share sensitivity, normalization by opportunity\\
\addlinespace
The four widths regroup one task pool rather than forming four datasets & \S\ref{sec:width} & Stated; capture rates rather than raw savings are the quantity compared across widths\\
\addlinespace
Each live gate run is 12 rows from 4 clusters & \S\ref{sec:live} & Case study, never a population claim; smallest exact Wilcoxon $p$ at 4 clusters is 0.125 and no $p$ is quoted off it\\
\addlinespace
Settleable capture is bimodal, not centered & \S\ref{sec:live} & Its mean describes those runs' visit orders; the order-fair sweep is the population version\\
\addlinespace
Whether an issue settles is partly a property of the patcher, not the issue & \S\ref{sec:live} & On Django our writer closed 3 of 8 issues no ranked submission resolved, because the harness shows failing tests; \emph{settleable} therefore means settleable-when-shown-the-tests, and resolve rates are not comparable to published SWE-bench numbers\\
\addlinespace
Near-50\% settleable capture is a pattern, not a law & \S\ref{sec:width} & Will move when requests are position-biased or costs correlate with position\\
\addlinespace
Bazel rests on two packages with clustered targets, and its headline is mostly manufactured opportunity & \S\ref{sec:bazel} & Labeled exploratory; only the kind skew and the capture split are meant to carry weight\\
\bottomrule
\end{tabular}
\end{table}

\begin{table}[htbp]
\centering
\caption{Limits on the instruments.}
\label{tab:limits-instr}
\footnotesize
\begin{tabular}{L{0.30\linewidth}L{0.14\linewidth}L{0.48\linewidth}}
\toprule
\textbf{Limitation} & \textbf{Where} & \textbf{How we handle it, or why it stands}\\
\midrule
Fidelity's $+96.1\%$ magnitude is authored, not measured & \S\ref{sec:fidelity} & Projection has measured support: $+3.2\%$ on live DataSciBench tokens and $+89.1\%$ / $+80.1\%$ on measured Jinja/Flask CPU. Fidelity remains authored ladder arithmetic, so only its sign and two zero controls transfer\\
\addlinespace
Every remainder is computed in one resource & \S\ref{sec:physops} & Prune already shows the resource can decide the verdict while the factorial stays internally consistent; the other rows are not audited for it, and reordering --- a latency lever scored on spend --- is the obvious suspect\\
\addlinespace
Speculation counts rounds rather than seconds, on uniform priors, leaning on one deep program & App.~\ref{sec:speculate} & A clock in a replay would time the harness; the priors are optimistically ordered relative to an indifferent speculator, and both are stated\\
\addlinespace
Live leg of the ordering identity is one authored program at one width, one vendor & \S\ref{sec:reorder} & Reported as the identity on the executed set, with the bill allowed to move by the materializer's own run-to-run noise\\
\addlinespace
Robustness is tested only for failures a step can detect & \S\ref{sec:sep} & The built-in retry restores every case with an injected missing edge; a hidden dependency that silently corrupts an output would pass unnoticed under either runtime\\
\addlinespace
Engine overhead is measured on replayed Montage workflows only & \S\ref{sec:demand} & There, total \code{lazy} engine time is at most 0.43\% of the recorded work (median 0.032\%), most of it the demanded-set computation; other workloads and step costs are not measured\\
\addlinespace
Ties are classified order-contingent from one recorded visit order & \S\ref{sec:sep}, \S\ref{sec:datasci} & On WfCommons all 24 break under at least one of eight shuffled orders; DataSciBench's 49 still rest on one recorded order\\
\addlinespace
Measured population varies from run to run & \S\ref{sec:planner} & Repeating one planner flips 6 of 31 tasks in or out (Jaccard 81\%), so the disjoint sets across planners are substantially run churn and not a capability ranking; shared-goal comparisons use the 77-goal intersection, while own-population figures, including projection, are labeled as such\\
\addlinespace
One repeat, of one planner & \S\ref{sec:planner} & The noise floor is measured only for \code{gpt-5.6-terra}: pooled movement under repetition is bounded by $2.3$ points against 4.6--8 between planners, but two runs give a difference and not a variance, and whether the other planners churn alike is untested\\
\addlinespace
Per-goal figures do not reproduce & \S\ref{sec:planner} & Repetition leaves 58\% of shared goals within $0.01$ points but moves 10\% by more than 20, canceling to 1.3\% in the pool; aggregates are citable, single rows are not\\
\addlinespace
Overlap ratios are scope-bound, and their apparent cross-planner stability does not survive matching & \S\ref{sec:reorder} & The WfCommons cell is quoted where the gap is 5.0\%, not the archived 42.0\%; on matched goals the ratio ranges from $-0.7\%$ to 13.3\%, so we treat it as arena-bound\\
\addlinespace
LLM-judged quality is sensitive to slot and label & \S\ref{sec:project} & Judged blind, candidates unlabeled and order assigned by a hash of the pair; the same content scores 89.5\% in the first slot against 78.7\% in the second, so any instrument that fixes order or names the arm reports part of that\\
\addlinespace
The controlled Jinja/Flask projection pair concedes its baseline and is narrow in scope & \S\ref{sec:project} & The static eager projection control is not given today's goal, and a static pass that were given it would take the same saving, so the remainder is a claim about programs compiled once and asked many times; its two controls are structural rather than a survey, and each magnitude comes from one test module against one suite in one of two sibling pure-Python projects on one machine\\
\bottomrule
\end{tabular}
\end{table}

\clearpage
\section{Operations measured in full}
\label{app:ops}

This appendix gives the authored projection controls and detailed reordering populations supporting
Section~\ref{sec:physops}, followed by the four operations summarized but not developed there. Together they
complete the operation map: prune isolates demand-independent refusal, reuse and fusion quantify mechanisms
that substitute for permission, and speculation traces the cost/latency frontier beyond demand.
Table~\ref{tab:scenario-vocabulary} defines the controlled scenario terms.

\begin{table}[H]
\centering
\caption{Reader-facing vocabulary for the controlled scenarios in this appendix. These terms describe
structural conditions, not software identifiers.}
\label{tab:scenario-vocabulary}
\small
\begin{tabular}{L{0.25\linewidth}L{0.67\linewidth}}
\toprule
\textbf{Scenario term} & \textbf{Meaning}\\
\midrule
Named-consumer case & An output omitted from today's request is still required by another consumer recorded
in the standing program, so a static eager control cannot discard it.\\
No-consumer control & No node reads the omitted output, so ordinary dead-output elimination is available to
both eager and lazy arms.\\
All-demanded control & Today's request needs every output, so neither arm may project any of them away.\\
Standing session & Several requests arrive over time against one persistent graph; prior artifacts may be
reused when still valid.\\
Session coverage & Fraction of the standing program's products requested at least once during the session.\\
Invalidation & Fraction of requests whose changed input makes an earlier artifact and its downstream
artifacts stale.\\
Settleable / open & A requested goal does / does not obtain an acceptable result during the run.\\
Combined policy names & A row such as ``stopping + pruning'' means the base permission policy plus the named
operation; the corresponding lazy and eager rows receive that operation as a matched pair.\\
\bottomrule
\end{tabular}
\end{table}

\subsection{Projection: the authored magnitudes}
\label{app:projauthored}

\begin{table}[H]
\centering
\caption{Projection on authored programs, where we set the prices. The named-consumer/no-consumer pair is
the one-edge isolation: the programs differ only in whether the expensive undemanded output has a downstream
consumer. The gate-output/all-demanded pair separately varies whether today's goal requests that output.
The session rows carry the same shape into a standing program under a ticket sequence. Magnitudes are
authored; the sign and control pattern are the claim, and the measured version is in
Section~\ref{sec:project}.}
\label{tab:projauthored}
\footnotesize
\begin{tabular}{L{0.38\linewidth}C{0.18\linewidth}C{0.14\linewidth}C{0.16\linewidth}}
\toprule
\textbf{Scenario} & \textbf{Worth to \code{stop\_eager}} & \textbf{Worth to \code{lazy}} & \textbf{Remainder}\\
\midrule
Gate output only; trace has another consumer (bills 256 / 256 / 56) &
$0.0\%$ & $+78.1\%$ & \textbf{$+78.1\%$}\\
Gate and trace both requested (bills 261 / 261 / 261) &
$0.0\%$ & $+0.0\%$ & \textbf{$+0.0\%$}\\
\addlinespace
Named-consumer case & 0.0\% & $+84.3\%$ & \textbf{$+84.3\%$}\\
No-consumer control & $+84.3\%$ & $+84.3\%$ & \textbf{$+0.0\%$}\\
\addlinespace
Release-package standing session & 0.0\% & $+34.3\%$ & \textbf{$+34.3\%$}\\
Model-refresh standing session & 0.0\% & $+35.6\%$ & \textbf{$+35.6\%$}\\
Session-coverage sweep, $100\% \to 25\%$ & 0.0\% & --- & \textbf{$+35.4\% \to +72.9\%$}\\
\bottomrule
\end{tabular}
\end{table}

Before projection, \code{stop\_eager} and \code{lazy} have identical costs in both isolation programs.
Adding the one consumer edge leaves projection worth 0.0\% to \code{stop\_eager} but raises its worth to
\code{lazy} from 0.0\% to 84.3\%.

\subsection{Reordering details}
\label{app:reorder}

\begin{table}[H]
\centering
\caption{Overlap by arena --- of the cost difference between the stopping rule
and demand, the share a price heuristic recovers with no notion of what was asked for. All four rows come
from the one ordering grid, so the columns are commensurable.}
\label{tab:overlap}
\small
\begin{tabular}{llrrrrr}
\toprule
\textbf{Arena} & \textbf{Cost unit} & \textbf{$n$} & \textbf{Stop rule} & \textbf{Demand} & \textbf{Perm.\ gap} & \textbf{Overlap}\\
\midrule
Release-package session & authored CPU & 5 & 42.0\% & 0.0\% & 59.5\% & \textbf{70.7\%}\\
Model-refresh session & authored CPU & 5 & 25.8\% & 0.0\% & 41.7\% & \textbf{61.8\%}\\
WfCommons (60 req., 200-task cap) & recorded CPU & 60 & 1.8\% & 0.0\% & 5.0\% & \textbf{36.0\%}\\
DataSciBench & real tokens & 95 & 4.2\% & 0.0\% & 19.8\% & \textbf{21.4\%}\\
\bottomrule
\end{tabular}
\end{table}

\begin{table}[H]
\centering
\caption{Cheapest-first handed to both arms. Top: the authored pair, where the
eager arm pays about three times what arbitrary order cost it. Bottom: how often the inversion appears per
arena, and the lazy column that is zero everywhere.}
\label{tab:invert}
\small
\begin{tabular}{L{0.35\linewidth}rrr}
\toprule
\textbf{Scenario} & \textbf{Stopping} & \textbf{Stopping + cheapest-first} & \textbf{Worth}\\
\midrule
Named-consumer projection case & 332 & 932 & \textbf{$-180.7\%$}\\
No-consumer projection control & 332 & 1{,}102 & \textbf{$-231.9\%$}\\
\bottomrule
\end{tabular}

\vspace{6pt}
\begin{tabular}{llrrrr}
\toprule
\textbf{Arena} & \textbf{Cost unit} & \textbf{Tasks} & \textbf{Eager inverted} & \textbf{Worst} & \textbf{Lazy inverted}\\
\midrule
WfCommons & recorded CPU-s & 60 & 7 (11.7\%) & $-47.6\%$ & \textbf{0}\\
DataSciBench & \textbf{real tokens} & 95 & 15 (15.8\%) & $\mathbf{-70.3\%}$ & \textbf{0}\\
Authored lifecycle & authored cpu-s & 23 & 12 (52.2\%) & $-231.9\%$ & \textbf{0}\\
\bottomrule
\end{tabular}
\end{table}

\subsection{Explicit prune}
\label{sec:prune}

Pruning permanently declines a node that is demanded, ready, and passes every guard --- a live, legal option
--- for a reason the graph does not encode (a deadline, a policy on what may be touched, or a budget). It is
not early stopping, which fires once the goal is met, and not a guard, which encodes a condition in the
program text. This experiment uses the same three methods as the earlier screening control but visits the
static check first, so fixed order reaches external review before the suite and an unpruned run settles at
cost 28 rather than 201. The two families are therefore not compared bill-for-bill. The prune rule removes
external review whenever the static check is inconclusive, identically on both permission arms:

\begin{center}
\small
\begin{tabular}{L{0.28\linewidth}rrL{0.18\linewidth}L{0.24\linewidth}}
\toprule
\textbf{Scenario} & \textbf{Unpruned} & \textbf{Pruned} & \textbf{Pruned option} & \textbf{Quality}\\
\midrule
Static check already succeeds & 3 & 3 & none & matched\\
Static check is inconclusive; suite succeeds & 28 & 203 & external review & matched --- paid the suite instead\\
Only external review succeeds & 28 & 203 & external review & \textbf{goal not solved}\\
\bottomrule
\end{tabular}
\end{center}

Remainder \textbf{$+0.0\%$} on every matched program: the rule reads the calendar and the budget, not the
demanded set, so nothing demand-driven is left to measure. Burning the review when it would have settled the
goal for 28 costs 203 instead, and when the suite does not settle either, the pruned arms fail the goal
outright. That negative control is what distinguishes prune from a disguised early stop. Prune is a
capability the semantics must name, so that a risky refusal cannot later be smuggled in as a guard --- not a
saving to apply by default.

\paragraph{The resource decides the verdict.}
In authored CPU seconds, prune is a $7.25\times$ regression. But a pruner exists to decline an option for a
reason the CPU column does not contain, while calendar lead time was absent from the printed resource table.
A lead-time variant keeps the structure and prices both resources: static check 2 CPU / same day, review 25
CPU / 21 days, suite 200 CPU / 2 days, release in 5 days. Fixed order reaches review before the suite, so an
unpruned run settles on review --- cheaply, and three weeks late.

\begin{center}
\small
\begin{tabular}{lrrc}
\toprule
\textbf{Policy} & \textbf{CPU cost} & \textbf{Lead time (days)} & \textbf{On time}\\
\midrule
\code{eager}, \code{guard\_eager} & 228 & 23.0 & 0/1\\
\code{stop\_eager}, \code{lazy}, \code{react} & 28 & 21.0 & 0/1\\
Stopping + pruning, \code{lazy} + pruning & 203 & \textbf{2.0} & \textbf{1/1}\\
\bottomrule
\end{tabular}
\end{center}

CPU accounting reports prune as $-625.0\%$ on both arms with remainder $+0.0\%$; calendar lead time reports
2 days versus 21, and only the pruned arms deliver on time. Two controls: with the ship date at 30 days the
rule does not fire and every arm is on time at CPU cost 28; when the static check settles, prune never fires
and every arm bills 3. \code{stop\_eager} obtains the same 21 days as \code{lazy}, so no lead-time remainder
exists. A remainder measures demand-awareness, not absolute worth, and the resource can decide the verdict
while every arm is measured identically.

\subsection{Reuse}
\label{sec:reuse}

Reuse asks whether an artifact on hand can satisfy a later ask. On the 28 DataSciBench follow-up
tasks, one graph is shared across a task's deliverables; each step is priced at the tokens metered for the
goal that ran it, a cache hit is priced at zero, and artifact retention is toggled identically on both arms.

\begin{center}
\small
\begin{tabular}{lr}
\toprule
\textbf{Quantity} & \textbf{Value}\\
\midrule
Mean reuse worth to \code{stop\_eager} & \textbf{54.4\%}\\
Mean reuse worth to \code{lazy} & \textbf{46.8\%}\\
Mean demand-specific remainder & \textbf{$-7.6\%$}\\
Pooled tokens, stopping / stopping + reuse & 698{,}928 / 291{,}942\\
Pooled tokens, \code{lazy} / \code{lazy} + reuse & 594{,}328 / 294{,}933\\
\bottomrule
\end{tabular}
\end{center}

The remainder is negative for the same asymmetry as ordering's: the eager arm did the extra work that
populated the cache, so the cache has more to give it. The two pooled bills differ by 1.0\%, \system{}'s the higher;
per task, the lazy arm with reuse is cheaper on 9, equal on 15, and dearer on 4, because deferring a step can
land it on a different metered price. The standing-session experiment sweeps session coverage and
invalidation rate, and the remainder is negative in every cell:

\begin{center}
\small
\begin{tabular}{L{0.16\linewidth}C{0.22\linewidth}C{0.22\linewidth}C{0.22\linewidth}}
\toprule
\textbf{Coverage} & \textbf{No invalidation} & \textbf{25\% invalidated} & \textbf{50\% invalidated}\\
\midrule
100\% & $-73.8\%$ & $-59.7\%$ & $-44.9\%$\\
75\%  & $-73.8\%$ & $-59.7\%$ & $-44.9\%$\\
50\%  & $-35.8\%$ & $-58.1\%$ & $-45.4\%$\\
25\%  & $-15.3\%$ & $-30.7\%$ & $-46.0\%$\\
\bottomrule
\end{tabular}
\end{center}

At 25\% coverage the two reuse arms are \textbf{bit-identical} (191 vs.\ 191 CPU-s). In the no-invalidation
column, the remainder rises as coverage falls ($-73.8\% \to -15.3\%$); the invalidated columns are not
monotone. Swept to its floor --- one touched product of eight, three widths, three session lengths --- it
reaches exactly $+0.0\%$ and never passes it. The cells agree by \emph{touched count}, not coverage: 8 of 32
products bills what 8 of 8 bills. A cache is a partial substitute for permission, whose irreducible value is
the first ticket and work a session never repeats. Lengthening the same eight-task cohort from two through
five tickets yields remainders of $-1.7\%$, $-1.4\%$, $-3.0\%$, and $-3.0\%$, with no movement toward a
crossover.

\subsection{Fusion}
\label{sec:fusion}

Fusion pays shared physical work once. It is equally available under eager permission, so we expected no
demand-specific gain and measured that prediction. A single static rewrite is not a valid matched
intervention here: fusing a demanded node with an undemanded one would re-bill the undemanded work to
\code{lazy} and change the common planned denominator. Fusion is therefore represented by an execution-time
ledger --- both nodes stay planned at full cost, and the second member of an equivalence class is not
charged for work the first already did. A hand-merged control bills 17, exactly what the two-node graph
bills with the ledger.

\begin{center}
\small
\begin{tabular}{L{0.24\linewidth}L{0.16\linewidth}C{0.18\linewidth}C{0.13\linewidth}C{0.12\linewidth}}
\toprule
\textbf{Scenario} & \textbf{Today's goal} & \textbf{Worth to \code{stop\_eager}} & \textbf{Worth to \code{lazy}} & \textbf{Remainder} \\
\midrule
Both computations requested & both reports & $32.0\%$ & $32.0\%$ & $+0.0\%$ \\
Second computation outside request & metrics report only & $32.0\%$ & $0.0\%$ & $\mathbf{-32.0\%}$ \\
\bottomrule
\end{tabular}
\end{center}

The zero confirms fusion is not secretly a demand mechanism. The negative row is the sharpest confirmation
in the section: fusion is worth 32.0\% to a baseline that ran the addendum and has two solver setups to
merge, and nothing to an arm that never ran it, because work permission declined is duplicate work fusion
cannot merge. A positive number here would have meant permission had left duplicated unrequested
work on the table. In the archived DataSciBench plans, 21 of 32 contain a product read by two or more steps,
so the structure occurs. This is a separate pre-Terra census, not the final 28-task Terra population; what
those corpora cannot price is the shared half.

\subsection{Speculation}
\label{sec:speculate}

Speculation materializes an eligible node outside the current demanded set, betting that a later ticket will
demand it. It does not enlarge the demanded set; it admits work outside that set. Its matched eager control
is degenerate because an eager arm already materializes every ready node.

\paragraph{Eager is a setting of this operation, not a baseline for it.}
The speculator takes a budget (spend permitted outside demand) and a floor on the authored odds it will
accept, and the two interpolate between the permission rules this paper has spent its length separating:

\begin{center}
\small
\begin{tabular}{llll}
\toprule
budget & prior floor & behavior & role \\
\midrule
$0$      & ---   & bet on nothing      & \code{lazy}: lower corner \\
bounded  & $0.5$ & the row under test  & the interior \\
$\infty$ & $0$   & bet on every eligible node & \code{eager}: upper corner \\
\bottomrule
\end{tabular}
\end{center}

Both corners are verified as \emph{identities}: at budget $0$ the speculative arm is indistinguishable from
\code{lazy} on cost \emph{and on the set of paid node identifiers}; opened fully it reaches \code{eager}'s
materialized set over the nodes that declare odds. ``Eager versus lazy'' becomes a comparison of two
settings of one operation.

\paragraph{Speculation cannot reduce cost, so its case is latency.}
Reuse already stops a session paying twice, so a bet can only move work earlier, and a wrong bet adds work.
If $C_{\mathrm{spec}}$ and $C_0$ are session costs with and without speculation, then
$C_{\mathrm{spec}}\ge C_0$; equivalently, every signed cost saving in the table is at most zero. Latency is
measured as \emph{rounds containing real work} under concurrent execution,
and every latency figure in this subsection is a percentage of those rounds rather than of elapsed time. The
reason is not convenience: replayed materializers return instantly, so a wall clock here would time our
harness rather than the work, and a seconds figure would be an artifact of the replay. A round is the unit a
scheduler actually saves when a bet lands --- one concurrent wave that no longer has to wait --- so the
proxy is the right shape, but it is a proxy, and the conversion to seconds depends on a real materializer's
service time, which this instrument does not supply. The lifecycle materializer reports without yielding,
so each wave's members ran one after another; rerunning all 21 lifecycle programs with a materializer that
yields before reporting changes no ticket's rounds or cost. The frontier compares parallel lazy execution
with reuse, with and without speculation, over standing programs driven by a ticket sequence:

\begin{center}
\small
\begin{tabular}{L{0.30\linewidth}rrr}
\toprule
\textbf{Scenario} & \textbf{Coverage} & \textbf{Cost} & \textbf{Latency}\\
\midrule
No invalidation & 100\% & $+0.0\%$  & $\mathbf{+15.4\%}$ \\
No invalidation & 75\%  & $+0.0\%$  & $\mathbf{+15.4\%}$ \\
No invalidation & 50\%  & $+0.0\%$  & $\mathbf{+22.2\%}$ \\
No invalidation & 25\%  & $-15.7\%$ & $+20.0\%$ \\
No invalidation & 12.5\% & $-30.1\%$ & $+0.0\%$ \\
50\% invalidated & 100\% & $-25.4\%$ & $+13.3\%$ \\
50\% invalidated & 50\%  & $-17.3\%$ & $+20.0\%$ \\
50\% invalidated & 12.5\% & $-30.1\%$ & $+0.0\%$ \\
Model refresh; no upstream change & --- & $+0.0\%$ & $\mathbf{+11.1\%}$ \\
Model refresh; upstream changes mid-session & --- & $-7.9\%$ & $+8.3\%$ \\
\bottomrule
\end{tabular}
\end{center}

Three regimes, and the boundary is whether the bet is vindicated rather than invalidation alone.
\textbf{Free latency}: at coverage $\ge 50\%$ with nothing invalidated, cost ties \emph{exactly} and latency
improves 11--22\%; every bet is vindicated, so the session pays the same total sooner. \textbf{A priced
trade}: once coverage falls below 50\% \emph{or} anything is invalidated, 8--25\% more cost buys 8--20\%
less waiting; whoever prepaid the widest set loses most when the world moves, which is the sentence this
paper already wrote about the eager arm. The 25\%-coverage row enters this regime despite zero invalidation:
a bet can go unredeemed simply because the session never returns for it.
\textbf{Strictly dominated}: at coverage 12.5\%, 30\% more cost buys nothing, because the session never
revisits enough for a bet to land. Speculation only has somewhere to spend because permission left the node
\textsc{planned}; an eager runtime bought the same work at full price on ticket one with no way to decline
it.

In this fixed-graph, single-ask family, the demanded set may contract after observations but cannot grow. A
node speculated while outside that set therefore cannot become demanded later in the run, giving a 0\% hit
rate by construction. The speculative case pays 5 units above the lazy arm's checkout, profile, and gate
work, so every fixed-graph, single-ask benchmark of this form reports speculation as strictly dominated.
Three refusals bound the claim: effectful nodes are never bet on at any odds; a node with no declared prior
is declined even at unbounded budget; and skill is refused --- every bettable node declares the same prior
($0.5$), so the frontier is what a \textbf{skill-free} speculator achieves and the floor a learned one must
clear.

\section{Full proofs}
\label{app:proofs}

The proofs use the execution model of Section~\ref{sec:model}. At each decision, $G$ is finite, its demand roots and the producer of every edge belong to $N(G)$, and every
materialized node emits a payload; compared policies receive the same graph, goal, projection,
materializers, barriers, selector, and execution mode unless one is the variable under study. A node becomes
dependency-ready only after every node it has an edge from has materialized or been skipped, so its guards
are resolved when it is offered; an unresolved guard fails open only inside the demanded set. A node is
\emph{live} if it is present, nonterminal under $O$, and has no guard resolving to \emph{fails}; by
Definition~\ref{def:demanded}, $\demanded(R,\sigma)$ is the set of live nodes that reach a live root demand
through live nodes. ``Offered'' means selected from the dependency-ready set for policy evaluation; ``executable'' means
its guards have not failed and it is reversible or covered by a released barrier. Set comparisons couple
realized outputs only for nodes invoked in both runs. Where a result fixes $G$ and $R$, transformation,
graph growth, request change, and artifact invalidation are outside that run segment and trigger a fresh
application of the result.

\paragraph{Lemma~\ref{lem:mono} (monotonicity).}
\begin{proof}[Full proof]
Fix $G$ and $R$ for the run segment stated in the lemma. $\demanded(R,\sigma_t)$ is the set of live nodes
that reach a live root demand through live nodes. Within the segment no node appears, a terminal node stays
terminal, and a guard resolved false stays false, because no payload it reads changes. A guard resolving true
kills nothing, because an unresolved guard already failed open. So the live set can only shrink, and with it
the set of live nodes that reach a live root through live nodes. Therefore
$\demanded(R,\sigma_{t+1})\subseteq\demanded(R,\sigma_t)$. By induction, a node outside the demanded set at
time $t$ remains outside it at every later state of that segment.
\end{proof}

\paragraph{Proposition~\ref{prop:equiv} (flat-pool equivalence).}
\begin{proof}[Full proof]
Consider the next node selected from a shared ready wave. If the goal is already satisfied, both policies
stop. Otherwise both apply the same guard and barrier tests, and an executable node belongs to
$\demanded(R,\sigma)$ by assumption, so both invoke materialization on the same node or neither does.
Coupling its realized output gives both runs the same new observation and next ready wave. Induction over
selected nodes gives identical invocations until the goal is satisfied; both then refuse all remaining work.
Hence $M_{\code{lazy}}=M_{\code{stop\_eager}}$. On a flat pool, take an executable node offered on the chain
of $c_i$. The nodes after it on that chain have not run, and their guards read nodes that have not run, so
those nodes are live; the offered node therefore reaches the root $c_i$ through live nodes and belongs to
$\demanded(R,\sigma)$.
\end{proof}

\paragraph{Proposition~\ref{prop:sep} (structural separation).}
\begin{proof}[Full proof]
Fix $G$ and $R$ and let $n\notin\demanded(R,\sigma)$. By Lemma~\ref{lem:mono}, $n$ remains outside the
demanded set throughout the fixed-program segment. Demand membership is necessary for \code{lazy} to invoke
materialization, so $n\notin M_{\code{lazy}}(\sigma)$ under every selector. The rule for
\code{stop\_eager} contains no membership test. If $n$ is offered while the goal is open, its guards have
not failed, and it is legal under the shared barrier, so \code{stop\_eager} invokes it and
$n\in M_{\code{stop\_eager}}(\sigma)$. Graph growth or a new request begins a new segment with a recomputed
demanded set.
\end{proof}

\paragraph{Corollary~\ref{cor:order} (order contingency).}
\begin{proof}[Full proof]
Let $n$ be executable off-path work that becomes ready and is eventually offered.
Proposition~\ref{prop:sep} excludes $n$ from the lazy invocation set. If the offer occurs while the goal is
open, \code{stop\_eager} invokes $n$; if it occurs after closure, \code{stop\_eager} refuses it. Thus
inclusion in the stopping executor's set is contingent on position relative to the closing event, while the
lazy decision is not. The sets can coincide only if every such node is postponed past that event. When
\code{stop\_eager}'s set contains \code{lazy}'s, each such invocation adds a strictly positive cost to
\code{stop\_eager}'s bill alone, so the same condition is necessary for bill equality. On a conjunctive goal
under the conditions of Proposition~\ref{prop:orderinv}, the goal closes only after the closure has run, and
\code{stop\_eager} refuses nothing executable before closure, so it invokes every node \code{lazy} does. If
no closing event occurs, \code{stop\_eager} invokes every such node it is
offered.
\end{proof}

\paragraph{Proposition~\ref{prop:orderinv} (order invariance).}
\begin{proof}[Full proof]
By Lemma~\ref{lem:mono}, \code{lazy} invokes nothing outside $D_0$. Take $u\in D_0$ and a path of live
nodes $u=v_0\to v_1\to\cdots\to v_k$ to a live root. Each $v_{i+1}$ waits for $v_i$, so none becomes terminal
before $u$ does; under C1 and C2 every guard in $D_0$ resolves the same way in every order, so by hypothesis
none dies; and under C3 none disappears. So $u$ stays demanded until it runs: once $u$ is ready, fairness
offers it, the goal is still open because it cannot close before $D_0$ has run, and demand permission admits
$u$ exactly when it is executable. By induction in topological order, each node therefore ends in the same
state in every order: a member of $D_0$ runs exactly when its producers have all materialized or been
skipped and it is executable, a node outside $D_0$ never runs, and a guard-dead node is skipped once ready.
Hence $M_{\code{lazy}}$ is invariant.

For \textbf{EachOf} over disjunctive pools, the statement applies only to the across-product choice: an
unrequested product is outside the roots in every order. Within a requested pool, the first passing
candidate may change with order, so the stronger whole-set invariance does not apply, as stated after the
proposition in Section~\ref{sec:theory}.
\end{proof}

\paragraph{Pricing a plan.} Under the conditions of Proposition~\ref{prop:orderinv}, with fixed additive node
costs $c$, \code{lazy}'s bill $\sum_{n\in M_{\code{lazy}}}c(n)$ depends only on $G$, $R$, $\sigma_0$, and the
fixed barrier choices, because $M_{\code{lazy}}$ does. The proof shows that each member of $D_0$ runs once
exactly when it is executable and that nothing else runs, so when every member of $D_0$ is reversible the bill
is $\sum_{n\in D_0}c(n)$, computed in one backward traversal before any node runs. Candidate graphs
$G_1,\ldots,G_m$ for the same request can therefore be compared by price before execution, as a query optimizer
compares plans. \code{stop\_eager} has no such price in general: by Corollary~\ref{cor:order}, its bill varies
with the selector.

\end{document}